%% file: main.tex
\documentclass[sigconf]{acmart}

\usepackage{graphicx} % Required for inserting images
\usepackage{amsmath}
\usepackage{amsthm}
\usepackage{titlesec}
\usepackage{makecell}

\usepackage[ruled,linesnumbered]{algorithm2e}

\newtheorem{theorem}{Theorem}[section]

\newtheorem{assumption}{Assumption}[section]

\AtBeginDocument{%
  \providecommand\BibTeX{{%
    \normalfont B\kern-0.5em{\scshape i\kern-0.25em b}\kern-0.8em\TeX}}}

\setcopyright{acmlicensed}
\copyrightyear{2024}
\acmYear{2024}
\acmConference[KDD '24]{Proceedings of the 30th ACM SIGKDD Conference on Knowledge Discovery and Data Mining}{August 25--29, 2024}{Barcelona, Spain}
\acmBooktitle{Proceedings of the 30th ACM SIGKDD Conference on Knowledge Discovery and Data Mining (KDD '24), August 25--29, 2024, Barcelona, Spain}
\acmDOI{10.1145/3637528. 3671943}
\acmISBN{979-8-4007-0490-1/24/ 08}

\begin{document}

\title{EntropyStop: Unsupervised Deep Outlier Detection with Loss Entropy}

\author{Yihong Huang}
\affiliation{%
  \institution{East China Normal University}
 \city{Shanghai}
  \country{China}}
\email{hyh957947142@gmail.com}

\author{Yuang Zhang}
\affiliation{%
  \institution{East China Normal University}
 \city{Shanghai}
  \country{China}}
\email{51255902045@stu.ecnu.edu.cn}

\author{Liping Wang}
\authornote{Corresponding author.}
\affiliation{%
  \institution{East China Normal University}
 \city{Shanghai}
  \country{China}}
\email{lipingwang@sei.ecnu.edu.cn}

\author{Fan Zhang}
\affiliation{%
  \institution{Guangzhou University}
 \city{Guangzhou}
  \country{China}}
\email{fanzhang.cs@gmail.com}

\author{Xuemin Lin}
\affiliation{%
  \institution{Shanghai Jiao Tong University}
 \city{Shanghai}
  \country{China}}
\email{xuemin.lin@gmail.com}

\renewcommand{\shortauthors}{
Yihong Huang, Yuang Zhang, Liping Wang, Fan Zhang, \& Xuemin Lin}
\definecolor{newtext}{rgb}{0, 0, 0}
% \definecolor{newtext}{rgb}{0, 0, 0.99}
% \definecolor{newtext}{rgb}{0.89, 0.26, 0.2}
\newcommand\newtext[1]{\textcolor{newtext}{#1}}

\begin{abstract}
Unsupervised Outlier Detection (UOD) is an important data mining task. With the advance of deep learning, deep Outlier Detection (OD) has received broad interest. Most deep UOD models are trained exclusively on clean datasets to learn the distribution of the normal data, which requires huge manual efforts to clean the real-world data if possible. Instead of relying on clean datasets, some approaches directly train and detect on unlabeled contaminated datasets, leading to the need for methods that are robust to such challenging conditions. Ensemble methods emerged as a superior solution to enhance model robustness against contaminated training sets. However, the training time is greatly increased by the ensemble mechanism. 

In this study, we investigate the impact of outliers on training, aiming to halt training on unlabeled contaminated datasets before performance degradation. Initially, we noted that blending normal and anomalous data causes AUC fluctuations—a label-dependent measure of detection accuracy. To circumvent the need for labels, we propose a zero-label entropy metric named Loss Entropy for loss distribution, enabling us to infer optimal stopping points for training without labels. Meanwhile, a negative correlation between entropy metric and the label-based AUC score is demonstrated by theoretical proofs. Based on this, an automated early-stopping algorithm called EntropyStop is designed to halt training when loss entropy suggests the maximum model detection capability. 
% In this study, we investigate the impact of outliers on the training phase, aiming to halt training on unlabeled contaminated datasets before performance degradation. Initially, we noted that blending normal and anomalous data causes AUC fluctuations—a label-dependent measure of detection accuracy. To circumvent the need for labels, we propose a zero-label entropy metric named Loss Entropy for loss distribution, enabling us to infer optimal stopping points for training without labels. Meanwhile, we theoretically demonstrate negative correlation between entropy metric and the label-based AUC. Based on this, we develop an automated early-stopping algorithm, EntropyStop, which halts training when loss entropy suggests the maximum model detection capability. 
%First, we observed that indiscriminately feeding both normal and anomalous data to the model led to continuous fluctuations in its detection capability during training. Consequently, we propose a novel metric, the entropy of the loss distribution, to assess the changes in detection performance of model after each parameter update. We theoretically demonstrate a negative correlation between entropy metric and the label-based AUC score. 
We conduct extensive experiments on ADBench (including 47 real datasets), and the overall results indicate that AutoEncoder (AE) enhanced by our approach not only achieves better performance than ensemble AEs but also requires under \newtext{2\%} of training time. Lastly, loss entropy and EntropyStop are evaluated on other deep OD models, exhibiting their broad potential applicability.

\end{abstract}




% \received{20 February 2007}
% \received[revised]{12 March 2009}
% \received[accepted]{17 May 2024}

%%
%% This command processes the author and affiliation and title
%% information and builds the first part of the formatted document.
\maketitle

\section{Introduction}

Outlier Detection (OD) is a fundamental machine learning task, which aims to detect the instances that significantly deviate from the majority \cite{od-survey}. In some contexts, outliers are also named as anomalies, deviants, novelties, or exceptions \cite{od-survey}. Due to various applications of OD in high-impact domains (e.g. financial fraud \cite{financial-example}), numerous researchers are devoted to proposing algorithms to tackle OD \cite{ADbench,Ts-benchmark,god-benchmark}. According to the availability of labels, OD tasks and solutions can be categorized into Supervised OD, Semi-Supervised OD, and Unsupervised OD \cite{inlier-priority}. With the rapid development of deep learning, deep OD algorithms are proposed increasingly \cite{deep-od-survey-2021,deep-od-survey-2019,deep-od-survey-3}. Compared to traditional algorithms, deep ODs can handle various kinds of complex data and high-dimensional data more effectively.
%\vspace{-1mm}

\begin{figure}[!htbp]
  \centering
  \includegraphics[scale=0.47]{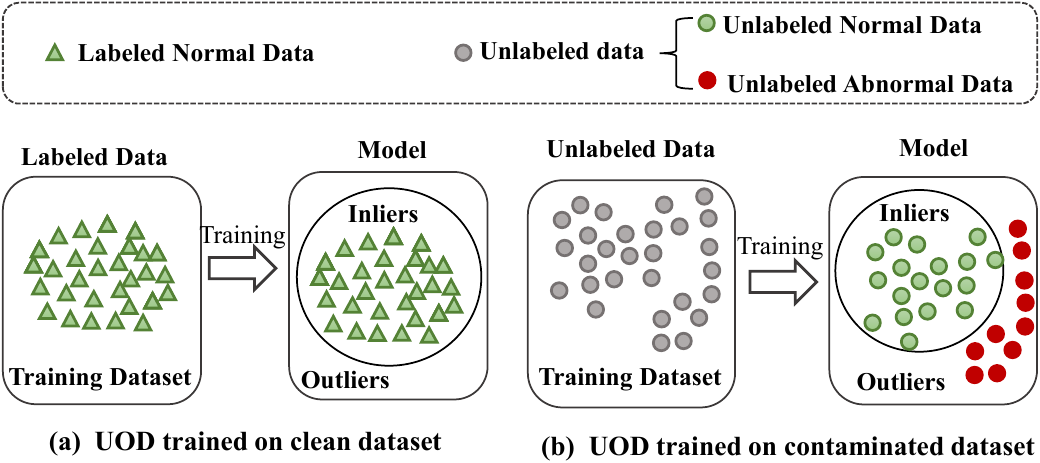}
  %\vspace{-4mm}
  \caption{Two paradigms of unsupervised OD}
  \label{Fig:two-uod}
\end{figure}

Unsupervised OD (UOD) aims to identify outliers in a contaminated dataset (i.e., a dataset consisting of both normal data and outliers) without the availability of labeled data \cite{inlier-priority}. 
While the study on deep UOD is extensive, it is crucial to distinguish between two fundamentally different paradigms within this domain. The first type, as shown in Fig. \ref{Fig:two-uod}(a), refers to the algorithms that are trained exclusively on clean datasets, e.g. DeepSVDD \cite{deep-svdd}, NeuTraL AD \cite{NTL}, ICL \cite{ICL}, AnoGAN \cite{AnoGAN}. These UOD algorithms operate on the premise that the training set is devoid of outliers, allowing the trained models to be applied to new test datasets containing potential anomalies. This approach necessitates the manual collection of large normal data, which imposes a burden before OD.

Conversely, the second type of UOD algorithms, shown in Fig. \ref{Fig:two-uod}(b), are designed to operate directly on the dataset that contains outliers, e.g., RandNet \cite{randnet}, ROBOD \cite{robod}, RDP \cite{RDP}, RDA \cite{RDA}, IsolationForest \cite{IsolationForest}, GAAL \cite{gan-ensemble}.  These models are capable of identifying outliers within the training set itself or, after being trained on a contaminated dataset, can be deployed to detect anomalies in new data—provided that the distribution of the new data aligns with that of the original training set.  
%In scenarios where the training dataset is inherently polluted, models typically exhibit a decrease in detecting effectiveness compared to those trained on clean data.  
The focus of our work is on the second paradigm where the UOD models are trained on contaminated datasets, which is more challenging. In this paper, we will discuss the purely unsupervised scenario where there is no available label for both training and validation. 
For the sake of convenience, the term Unsupervised OD mentioned in the remainder of this article, unless specifically stated otherwise, will refer to OD in the purely unsupervised setting.
\vspace{-2mm}
\begin{figure}[!htbp]
  \centering
  \includegraphics[width=0.23\textwidth]{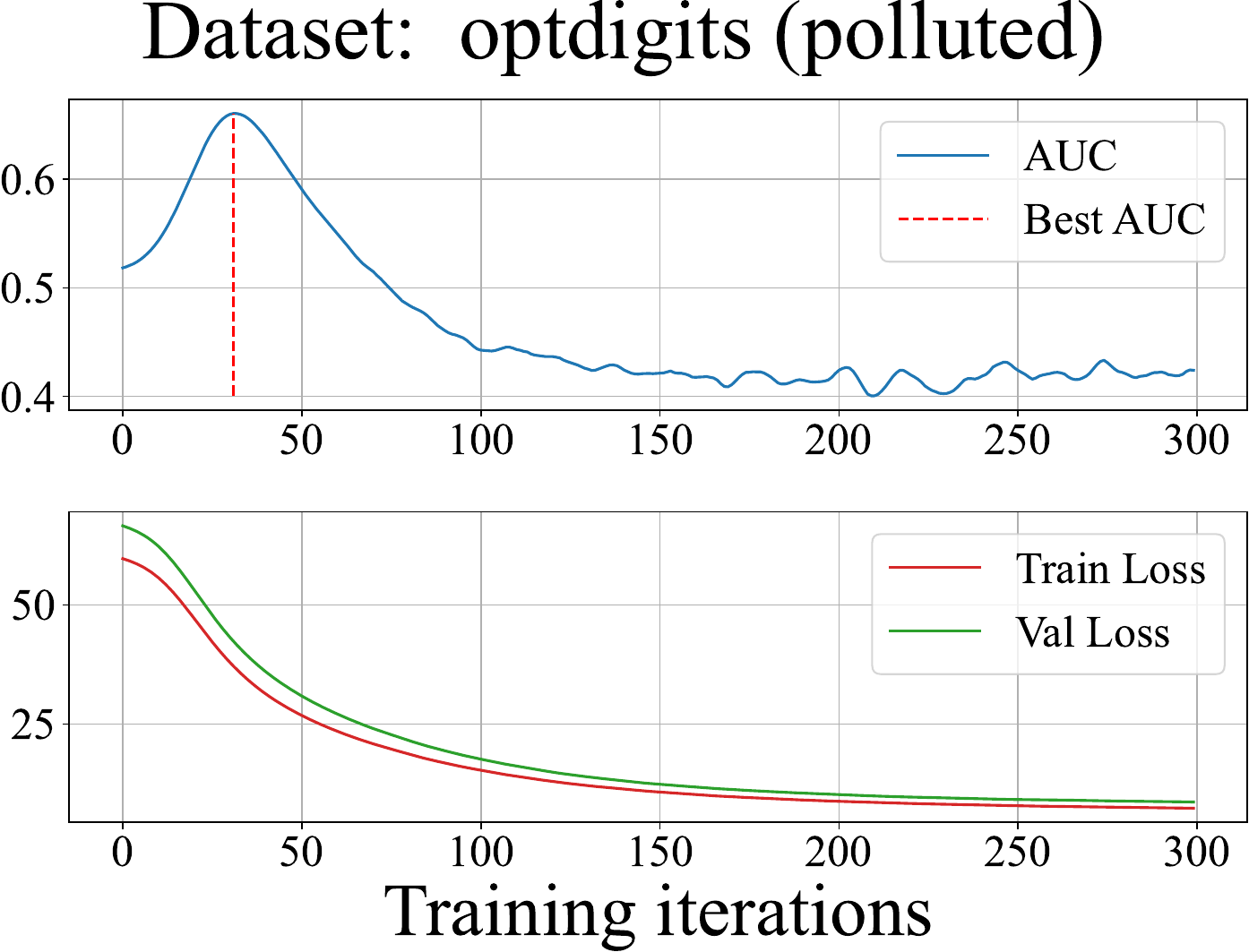}
  % \hspace{1in}
  \includegraphics[width=0.23\textwidth]{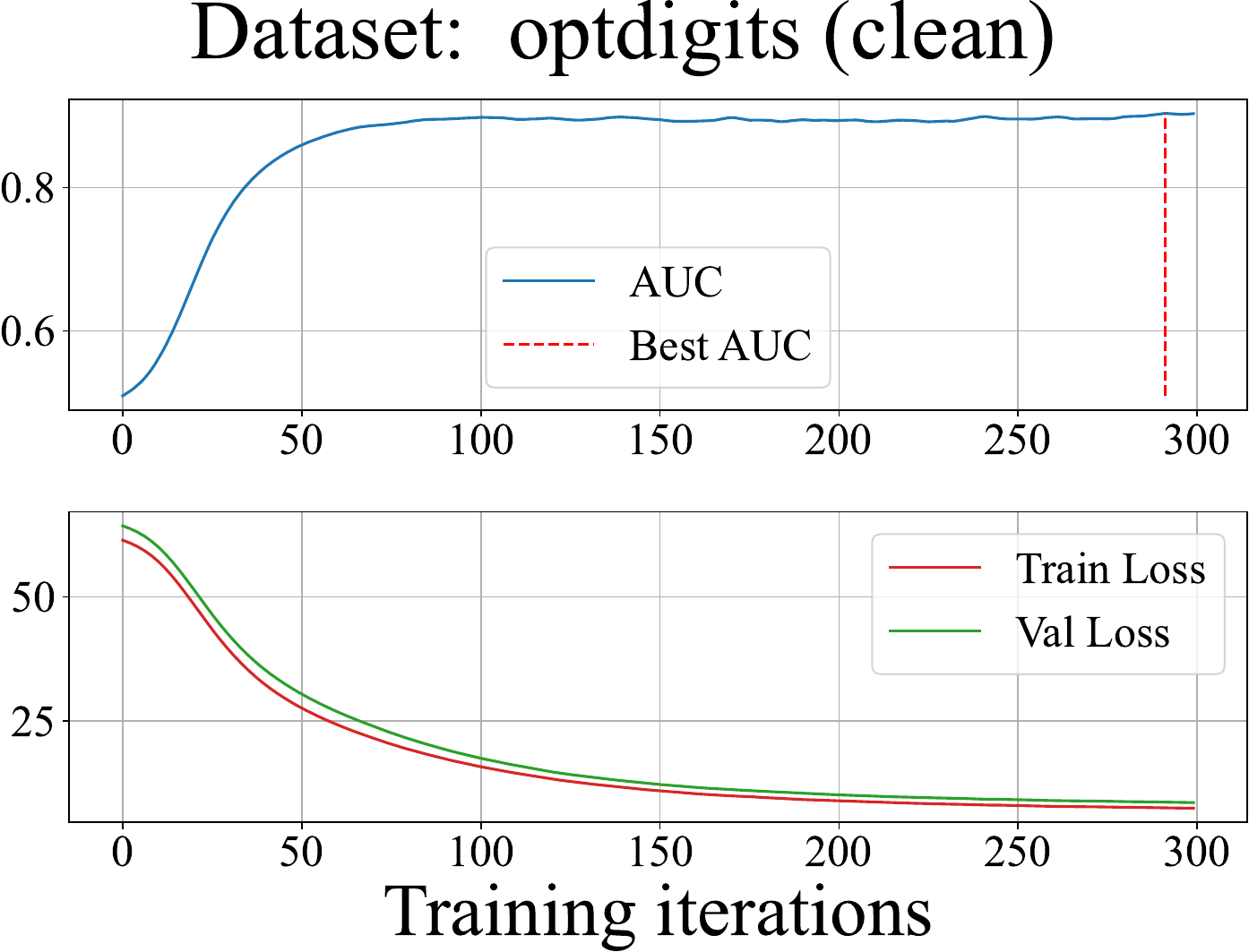}
  \includegraphics[width=0.23\textwidth]{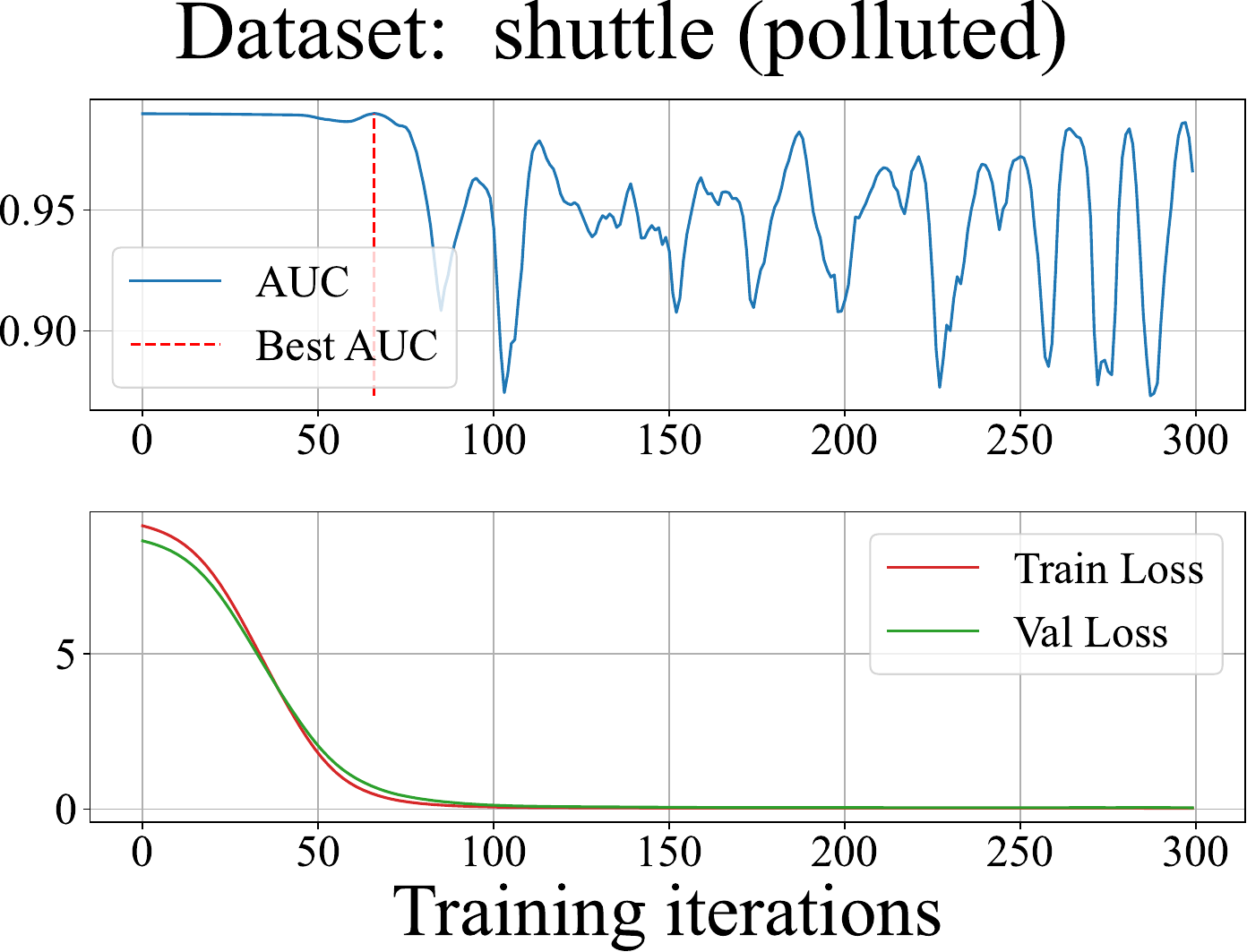}
  \includegraphics[width=0.23\textwidth]{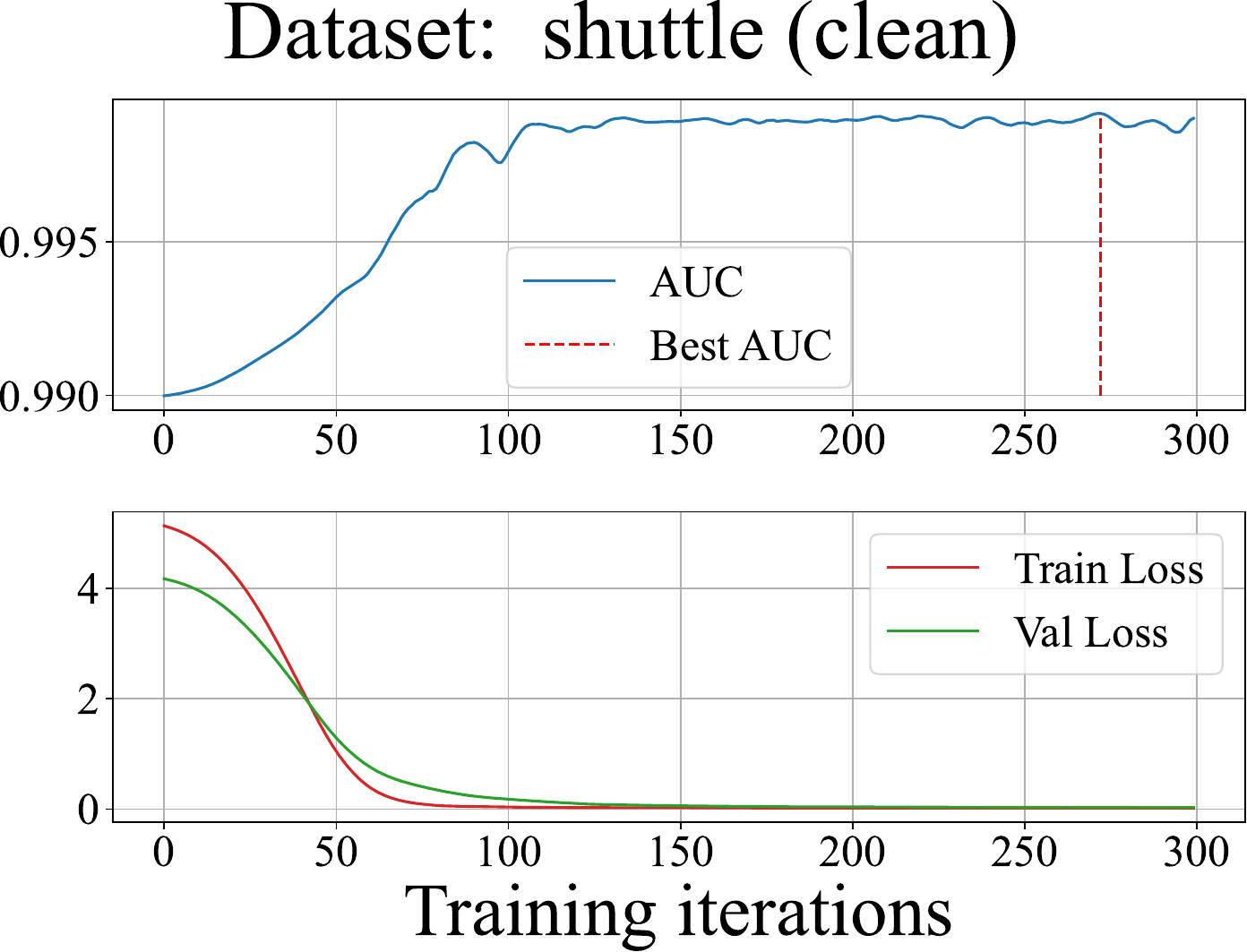}
  \caption{UOD training process of AutoEncoder on 2 datasets}
  \label{Fig:loss-auc}
\end{figure}
\vspace{-1mm}

\noindent \textbf{\textit{Challenge}}. \newtext{It is well-known that the model trained on a contaminated dataset will result in a much worse performance. Fig. \ref{Fig:loss-auc} shows the trend of Area Under Curve (AUC) \cite{auc} and loss throughout the unsupervised training process of an AutoEncoder (AE) across two datasets.
The dataset is divided into training and validation sets. "Polluted" indicates the presence of outliers in the datasets, whereas "clean" signifies that the datasets contain only normal samples. The AUC is calculated using the labels by Eq. \ref{eq: auc}. Note that we assume labels of the validation set in the "polluted" setting are not available here for evaluation, adhering to a purely unsupervised paradigm.
Fig. \ref{Fig:loss-auc} shows that when the AE model is trained on a clean dataset, the AUC increases steadily until convergence. However, on the contaminated dataset, the AUC exhibits significant fluctuations, and no such noteworthy features are observed in the loss curve. Such fluctuations in AUC can be attributed to the AE model's objective of minimizing loss across both normal and anomalous data, leading to a scenario where a reduction in total loss does not necessarily equate to enhanced detection capabilities. }The compulsory divergence between unsupervised training objective and application objective leads to the observed volatility in AUC.

To solve the above issue, the SOTA deep UOD models adopt an ensemble learning approach \cite{gan-ensemble,RDP,randnet,robod} to enhance the model's robustness to outliers. Their  strategy is to train multiple OD models (such as AEs) and use the results of voting to enhance the robustness and improve detection performance.
To generate diverse voting outcomes, models are intentionally overfitted to the dataset through varying configurations, such as different random seeds or hyperparameters (HPs) \cite{robod}, and extended training periods \cite{randnet}.
However, overtraining a large number of models imposes significant time and computation costs.

\vspace{1mm}
\noindent \textbf{\textit{Our Solution}}. The current dilemma: the presence of anomalies diminishes training effectiveness, while existing ensemble solutions improve performance at the sacrifice of efficiency. Different from current methods, we propose a novel approach through data distribution analysis. This work employs early stopping to mitigate the negative effects of outliers in the training sets, thus improving training efficiency and effectiveness. Specifically, this work delves into the impact of outliers on the model training process. We first identify the existence of a loss gap (i.e., the expected difference in training loss between outliers and inliers) and introduce a novel metric, the entropy of loss distribution in different training iterations, to reflect changes in the detection capability during training. We theoretically demonstrate that under certain assumptions, \newtext{an increase in AUC is likely to cause a decrease in loss entropy, with the converse also holding.} Notably, unlike AUC, the computation of entropy does not require labels. In this case, we can utilize the entropy curve to mirror changes in the AUC curve (examples are given in Fig. \ref{Fig:auc-entropy-corre}). Surprisingly, our experiments reveal a strong correlation between the two metrics across numerous real-world datasets. Leveraging this, we propose a label-free early stopping algorithm that uses entropy minimization as a cue for optimal training cessation. 

%This zero-label approach for deep UOD significantly reduces  training time and improves result quality. 

Our experiments across 47 real datasets \cite{ADbench} observed that AE models often achieve high AUC relatively early in training, and our entropy-based early stopping algorithm effectively identifies these moments to automatically halt training. The results demonstrate that our method significantly enhances the detection performance of AE, while significantly reducing training time compared to AE ensemble solutions.
Lastly, we discovered that the entropy-based early stopping algorithm can also be extended to other deep UOD models, exhibiting their broad potential applicability. 
%Lastly, it is observed that the entropy curve can be adaptive to the various hyperparameter configurations of deep UOD models, effectively mirroring the variations in detection capability.
The contributions of this paper are as follows:%\vspace{0mm}
\begin{itemize}
    \item We conduct an in-depth analysis of the impact of outliers (anomalies) during the training process of deep UOD models, based on the principle of imbalance between normal and anomalous instances.
    \item We propose a novel metric (loss entropy), i.e., the entropy of loss distribution, to reflect changes in modeling AUC with no labels. To the best of our knowledge, this is the first indicator that can predict changes in model performance without labels, validated across a multitude of real datasets.
    \item We develop an automated early-stopping algorithm that can automatically help UOD models avoid fitting on anomalous data and reduce training time.
    \item We conduct extensive experiments to demonstrate the efficacy of our metric and algorithm, validating the superior performance compared to ensemble solutions while requiring a minor fraction of time.
\end{itemize}
To foster future research, we
open source all codes at 
\newtext{\url{https://github.com/goldenNormal/EntropyStop-KDD2024}}.

%URL\footnote{https://www.dropbox.com/scl/fi/ce0aw9qdfjcvpgbb2hge1/EntropyStop-Code.zip?rlkey=4kbjwui8ww0tdqby2v8stcnvt\&dl=0}.

\section{Related Work}
\subsection{Unsupervised Outlier Detection}
Unsupervised outlier detection (UOD) is a vibrant research area, which aims at detecting outliers in datasets without any label during the training \cite{od-survey}. Solutions for unsupervised OD can be broadly categorized into shallow (traditional) \cite{IsolationForest, lof,knn} and deep (neural network) methods. Compared to traditional counterparts, deep methods handles large, high-dimensional and complex data better \cite{deep-od-survey-2019,deep-od-survey-2021,deep-od-survey-3}. 
Most deep UOD models \cite{ICL,NTL,deep-svdd,AnoGAN} are trained \textit{exclusively on clean datasets} to learn the distribution of the normal data. A fundamental premise of this methodology presupposes the availability of clean training data to instruct the model on the characteristics of "normal" instances. However, this assumption frequently encounters practical challenges, as datasets are often enormous and may inadvertently include anomalies that the model seeks to identify \cite{LOE}. In response to this dilemma, certain studies \cite{dagmm,RDA,RDP} venture into developing deep UOD algorithms that operate directly on contaminated datasets. Model ensemble approaches are proposed for their outstanding performance and robustness, coupled with a diminished sensitivity to hyperparameters (HPs) \cite{robod,randnet,gan-ensemble}. Additionally, efforts are made to adapt models originally trained on clean datasets to contaminated ones through outlier refinement processes \cite{LOE,outlier_refine_2015,outlier_refine_2021}. Nevertheless, to our best knowledge, the existing UOD studies do not capture the changes in model performance during the training to enable effective early stopping. 

%consider early-stop the training process by model

\subsection{Early Stopping Techniques}
Early stopping is an effective and broadly used technique in machine learning. Early stopping algorithms are designed to monitor and stop the training when it no longer benefits the final performance. A well-known application of early stopping is to use it as a regularization method to tackle overfitting problems with cross-validation, which can be traced back to the 1990s \cite{ES_1st}. Recently, with a deeper understanding of learning dynamics, early stopping is also found practical in noisy-labeled scenarios \cite{ES_noise_0, ES_noise_1, ES_noise_2, ES_noise_3}. According to these previous studies, overfitting to the noisy samples in the later stage of training decreases the model's performance, and can be mitigated by early stopping. Previous works show the outstanding ability of early stopping to deal with noisy learning environments. However, existing researches focus on supervised or semi-supervised settings, while early stopping in unsupervised \newtext{contaminated training set} is significantly more challenging.
To our best knowledge, we are the first to apply a label-free and distribution-based heuristic to explore the potential of early stopping in Unsupervised OD \newtext{on contaminated training sets.}

\section{Preliminary}
\noindent \textbf{Problem Formulation} (Unsupervised OD).
\textit{Considering a data space $\mathcal{X}$,
an unlabeled dataset $D = \{\textbf{x}_j\}_{j=1}^n$ consists of an inlier set $D_{in}$ and an outlier set $D_{out}$, which originate from two different underlying distributions $\mathcal{X}_{in}$ and $\mathcal{X}_{out}$, respectively \cite{uod-definition}. The goal is to learn an outlier score function $f(\cdot)$ to calculate the outlier score value $v_j = f(\textbf{x}_j)$ for each data point $\textbf{x}_j \in D$. Without loss of generality, a higher $f(\textbf{x}_j)$ indicates more likelihood of $\textbf{x}_j$ to be an outlier.
}

\noindent \textbf{Unsupervised Training  Formulation for OD.} Given a UOD model $M$, at each iteration, a batch of instances $D^{b} = \{x_0,x_1,...,x_{n}\}$ is sampled from the data space $\mathcal{X}$. The loss $\mathcal{L}$ for model $M$ is calculated over $D^{b}$ as follows:
$$
\mathcal{L}(M; D^{b}) = \frac{1}{|D^{b}|} \sum_{x \in D^{b}} \mathcal{J}_M(x) =  \frac{1}{|D^{b}|} \sum_{x \in D^{b}} f_M(x) = \frac{1}{|D^{b}|} \sum_i v_i  
$$
where $\mathcal{J}_M(\cdot)$ denotes the unsupervised loss function of $M$ while  $\mathcal{L}$ denotes the loss based on which the model $M$ updates its parameters by minimizing $\mathcal{L}$, with assumption that the learning rate $\eta$ is sufficiently small. In addition, we assume the unsupervised loss function $\mathcal{J}_M(\cdot)$ and outlier score function $f_M(\cdot)$ are exactly the same in our context. If this does not hold, at least the Assumption \ref{assum:align} should be 
 satisfied in our context.
Throughout the training process, no labels are available to provide direct training signals, nor are there validation labels to evaluate the model's performance. 
Since $f_M(x) > 0$ typically holds, we assume $f_M(x) >0$.

 \begin{assumption}[Alignment] 
$$
\forall \textbf{x}_i, \textbf{x}_j \sim X, \quad f_M(\textbf{x}_i) < f_M(\textbf{x}_j) \iff \mathcal{J}_M(\textbf{x}_i) < \mathcal{J}_M(\textbf{x}_j)
$$
\label{assum:align}
%\forall \textbf{x}_i,\textbf{x}_j \sim X$,  $f_M(\textbf{x}_i) < f_M(\textbf{x}_j)$, then $\mathcal{J}_M(\textbf{x}_i) < \mathcal{J}_M(\textbf{x}_j)$.
\end{assumption}

\noindent \textbf{Objective: }
The objective is to train the model $M$ such that it achieves the best detection performance on $\mathcal{X}$. Specifically, we aim to maximize the probability that an inlier from $\mathcal{X}_{in}$ has a lower outlier score than an outlier from $\mathcal{X}_{out}$, i.e., 
\begin{align}
        P(v^{-} < v^{+}) = P(f_M(x_{in}) < f_M(x_{out})| x_{in} \sim \mathcal{X}_{in}, x_{out} \sim \mathcal{X}_{out})
        \label{eq:detection-auc}
\end{align}
as large as possible, where $f_M(\cdot)$ is the outlier score function learned by model $M$. Let $\mathcal{O}_{in}$ and $\mathcal{O}_{out}$ represent the distributions of $f_M(x)$, where $x$ is drawn from $\mathcal{X}{in}$ and $\mathcal{X}{out}$, respectively. Therefore, $v^{-} \sim \mathcal{O}_{in}$ and $v^{+} \sim \mathcal{O}_{out}$ denotes the corresponding random variable of outlier score.

\noindent \textbf{The relationship between $P(v^- < v^+)$ and AUC  .} 
AUC \cite{auc} is a widely-used metric to evaluate the outlier detection performance, which can be formulated as:
\begin{equation}
 AUC(M,D) = \frac{1}{|D_{in}| |D_{out}|} \sum_{\textbf{x}_i \in D_{in}} \sum_{\textbf{x}_j \in D_{out}}\mathbb{I}(f_M(\textbf{x}_i) < f_M(\textbf{x}_j)) 
 \label{eq: auc}
 \end{equation}
 where $\mathbb{I}$ is an indicator function. 
Note that in practice, AUC is discretely computed on a real dataset, and the expression $P(v^- < v^+)$ is the continuous form of AUC. $P(v^- < v^+)$ signifies the model's inherent capability to distinguish between inliers and outliers from a view of the data distribution instead of a certain dataset.

\section{Methodology}

\input{methodology}

\section{Experiments}

\input{Exp}

\section{Conclusion}
In this paper, we are dedicated to exploring the issue of training unsupervised outlier detection models on contaminated datasets.  Different from existing methods, we investigate a novel approach through data distribution analysis. Firstly, we introduce the concept of loss gap and explain the prevalence of inlier priority. Based on this, we propose a zero-label evaluation metric, Loss Entropy, to mirror changes in the model's detection capability. Based on the metric, an early stopping algorithm (EntropyStop) to automatically halt the model's training is devised. Meanwhile, theoretical proofs for our proposed metric are provided in detail. Comprehensive experiments are conducted to validate the metric and algorithm. The results demonstrate that our method not only shows effectiveness but also significantly saves training time.

Furthermore, EntropyStop can be integrated with various deep models, suggesting its potential for extensive application. We envisage that the proposed metric, loss entropy, could bring new vitality to the field of anomaly detection.
%To enable more effective training and reduce the training time, we propose a zero-label evaluation metric for assessing the detection capabilities of models. This metric reflects potential AUC changes. Based on this, we introduce an early stopping algorithm to automatically halt the model's training. We provide theoretical proof for our proposed metric and validate it through extensive experiments. It demonstrate that our method is not only effective but also significantly saves on training time.

\section*{Acknowledgement}
This work is supported by the National Science and Technology Major Project 2021ZD0114501.

\bibliographystyle{ACM-Reference-Format}
\bibliography{my-ref}

%%
%% If your work has an appendix, this is the place to put it.
\appendix
\input{proof}

\input{Appx-Exp}

\input{limit_study_appx}

\input{Appx-EntropyStop-guide}

\input{Appx-table}

\end{document}

%% file: methodology.tex
In this section, we elucidate how early stopping can enhance the training effectiveness of unsupervised OD models on contaminated datasets. Initially, we introduce the concept of \textit{loss gap} and explain the prevalence of \textit{inlier priority}, which refers to the phenomenon that the average loss of normal samples invariably remains lower than that of anomalous samples. Subsequently, we introduce a novel metric, Loss Entropy ($H_L$), which mirrors changes in the model's detection capability. Notably, the calculation of $H_L$ does not involve labels, making it a purely internal evaluation metric.  Finally, leveraging the proposed $H_L$, we design an early stopping algorithm $EntropyStop$ that can cease training automatically when the $H_L$ is sufficiently small.

\subsection{Loss Gap and Inlier Priority}
\subsubsection{\textbf{Loss Gap}} Firstly, we propose the concept of loss gap, which can reflect the fitting  difference between inliers and outliers.
Given that a batch of dataset $D^{b}$ can be divided into two parts, $D_{in}^b$ and $D_{out}^b$, the average loss for both the normal and abnormal part can be calculated as $\mathcal{L}_{in}$ and $\mathcal{L}_{out}$, respectively. The term "loss gap" refers to the gap between the two average loss values. Thus, we define the loss gap as follows:
\begin{equation}
    \mathcal{L}_{{in}} = \frac{1}{|D_{in}^b|} \sum f_M(\textbf{x}_i), \quad \textbf{x}_i \in D_{in}^b
    \label{def:inlier-loss-on-D}
\end{equation}
\begin{equation}
    \mathcal{L}_{out} = \frac{1}{|D_{out}^b|} \sum f_M(\textbf{x}_i), \quad \textbf{x}_i \in D_{out}^b
    \label{def:outlier-loss-on-D}
\end{equation}
\begin{equation}
  L_{gap} = \mathcal{L}_{out} - \mathcal{L}_{in}  
  \label{loss gap}
\end{equation}

\subsubsection{\textbf{The prevalence of the inlier priority}}
\label{inlier-priority-sec}
Typically, $L_{gap} > 0$ is usually observed  during the training, which is called as \textit{\textbf{inlier priority}} in the literature \cite{inlier-priority}.
The reason can be attributed as follows.
Outliers refer to points that deviate significantly from the vast majority, such as noise. A characteristic of outliers is their scarcity and the significant distinction in their pattern from most points. 
\newtext{In some scenarios, although outliers can be similar to inliers in attributes, they are still  relatively scarce and have patterns and distributions that are different from the majority of the dataset. This distinction can be utilized by UOD algorithms, assigning higher scores to outliers.}
Therefore, the model tends to generate greater losses for outlier samples compared to normal ones. Consequently, it is often observed during training that the loss associated with outlier samples exceeds that of normal samples, indicating a gap in loss values. This gap helps  outlier detectors identify outliers with greater loss.  Examples of loss gap are shown in Fig. \ref{Fig:inlier-priority} that there is a gap between $\mathcal{L}_{in}$ and $\mathcal{L}_{out}$ while $\mathcal{L}_{in} < \mathcal{L}_{out}$  holds during the training.  $\mathcal{L}_{in} < \mathcal{L}_{out}$ can also be explained in following two perspectives \cite{inlier-priority, DRAE-inlier-priority}:

\begin{figure}[!htbp]
  \centering
  \includegraphics[scale=0.255]{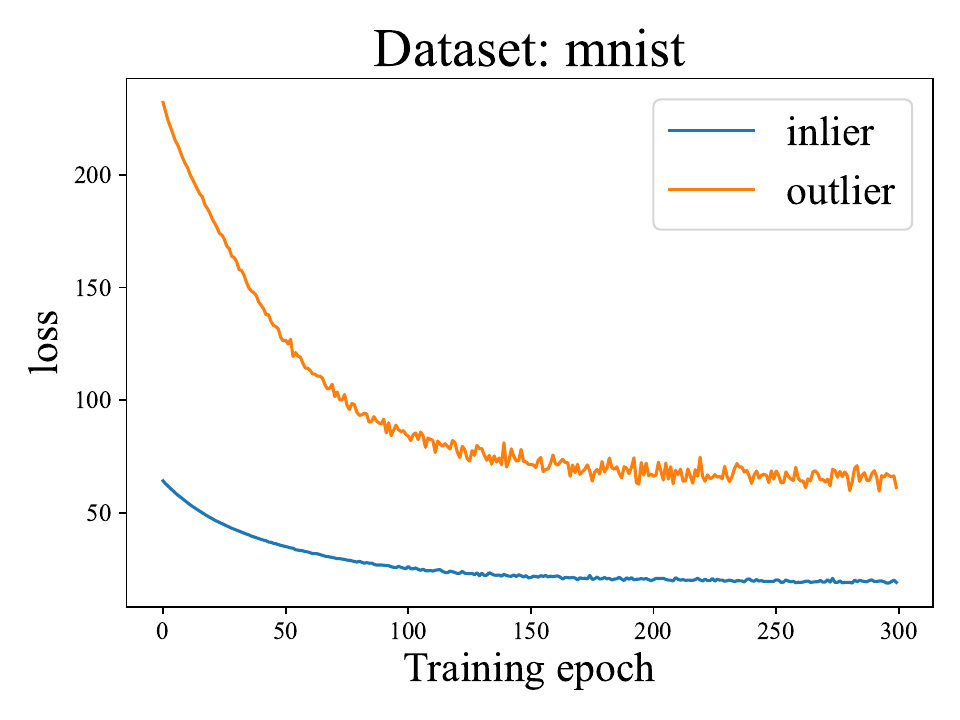}
  \includegraphics[scale=0.255]{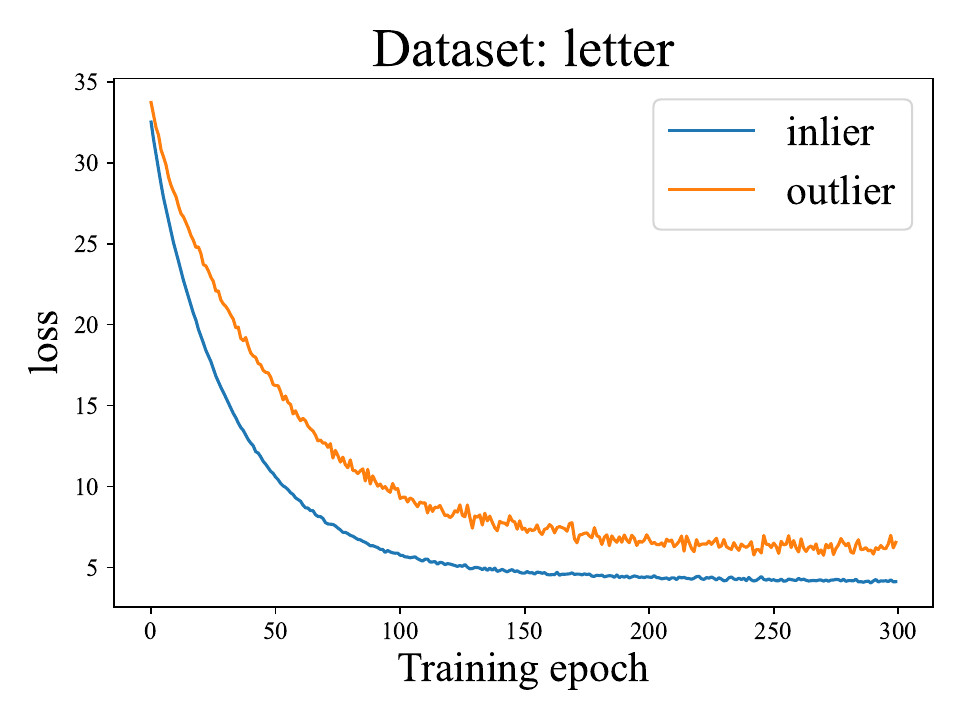}
  \caption{Loss Gap for inliers and outliers in AE models on MNIST and Letter datasets}
  \label{Fig:inlier-priority}
\end{figure}

\noindent \textbf{From the loss perspective:} The overall loss $\mathcal{L}$ can be represented by the weighted sum of $\mathcal{L}_{in}$ and $\mathcal{L}_{out}$:
\begin{equation}
    \mathcal{L} = \frac{|D_{out}^b|}{n}\mathcal{L}_{out} + \frac{|D_{in}^b|}{n} \mathcal{L}_{in}
    \label{priority-by-number}
\end{equation}
Due to the scarcity of outliers (i.e., $|D_{in}^b| \gg |D_{out}^b|$), the weight of $\mathcal{L}_{in}$ is  larger. Thus, the model puts more efforts to minimize  $\mathcal{L}_{in}$.

\vspace{1mm}
\noindent \textbf{From the gradient perspective:}
The learnable weights $\Theta$ of $M$ are updated by gradient descent:
\begin{equation}
    \overline{g} = \frac{1}{n} \sum g_i = \frac{1}{n} \sum \frac{\mathrm{d}f_M(\textbf{x}_i)}{\mathrm{d}\Theta}
    \label{average-gradient}
\end{equation}
where $g_i$ is the gradient contributed by $\textbf{x}_i$. 
The \newtext{normalized} reduction in loss for the \(i^{th}\) sample 
 is as follows:
\begin{equation}
    \Delta \mathcal{L}_i = \frac{<g_i, \overline{g}>}{|\overline{g}|} = |g_i| cos\theta(g_i,\overline{g})
    \label{effect-gradient}
\end{equation}
where $\theta(g_i,\overline{g})$ is the angle between  two gradient vectors. 
In most cases, outliers are arbitrarily scattered throughout the feature space, resulting in counterbalancing gradient directions; while inliers are densely distributed, and their gradient directions are relatively more consistent. Therefore, $\theta(g_i,\overline{g})$ for an inlier is  often smaller than that of an outlier, leading to a larger $\Delta \mathcal{L}_i$ for $\textbf{x}_i\in D_{in}$.

% \subsubsection{\textbf{loss variation gap}}
% Further, we can define the Loss Variation Gap $\delta L_{gap}$ as follows:
% \begin{gather}
%     \delta L_{gap} = \Delta \mathcal{L}_{out} - \Delta \mathcal{L}_{in} \\
%     \Delta \mathcal{L}_{class} = \sum_{i\in D_{class}} \Delta \mathcal{L}_{i} = \mathcal{L}^{t-1}_{class} - \mathcal{L}^{t}_{class}, class \in \{in, out\}
%     \label{eq: loss-variation-gqp}
% \end{gather} 
% It represents the difference in the change of average loss between inliers and outliers after the $t^{th}$ gradient update. The $\delta L_{gap}$ is crucial for analyzing changes in the model's detection performance (i.e., $P(f_{in}<f_{out})$), as it reflects to some extent whether, after a parameter update, the model reduces the $f_{in}$ or $f_{out}$ more significantly. 

% \subsubsection{\textbf{The impact of loss gap on AUC.}}
% The change of $L_{gap}$ is crucial for analyzing changes in the model's detection capability (i.e., $P(v^- < v^+)$). 
% Here, we prove that under the assumption that \textit{inlier priority} holds, a increase in $L_{gap}$ will lead the $AUC$ has more possibility to increase.

% \begin{theorem}
% \label{theorem-loss-gap-auc}  

% When $\mathcal{L}_{in} < \mathcal{L}_{out}$
% and  $L_{gap}$ increases, the AUC is more likely to increase.
% \end{theorem}

% \begin{proof}
% See Appx. \ref{proof: loss-gap-auc} for the proof.
% \end{proof}

\newtext{
 In this case, we can conclude that
 if $\mathcal{L}_{in} \approx \mathcal{L}_{out}$, then
 $\Delta \mathcal{L}_{in} > \Delta \mathcal{L}_{out}$, resulting in $L_{gap} > 0$ (i.e.,\textit{inlier priority}).
 Our subsequent proposed metric, loss entropy, is based on \textit{inlier priority}, as it works as a foundational assumption for the theoretical proof and  intuition understanding of our metric.
}

\begin{figure*}[!htbp]
  \centering
  \includegraphics[width=500.00pt]{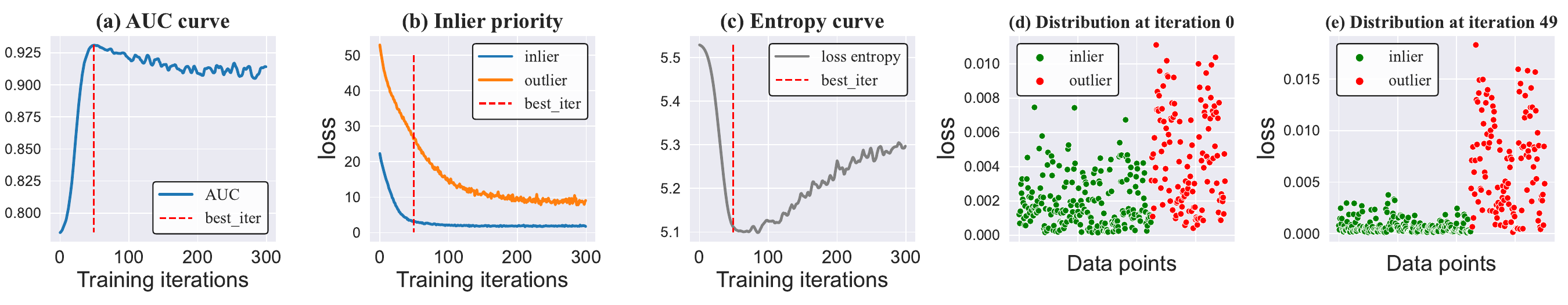}
  \caption{An example of the training process. The AE model is trained on the dataset Ionosphere with 300 iterations. In this example, the lowest $H_L$ exactly matches the optimal AUC at the $49^{th}$ iteration.  The y-axis of two  scatter plot (i.e. the $4^{th}$ figure and the $5^{th}$ figure) is normalized data loss value $u_i$.}
  \label{Fig:loss-distribution}
\end{figure*}

\subsection{Loss Entropy $H_L$: The Novel Internal Evaluation Metric}

Next, we introduce a metric that can be computed without labels. Importantly, this metric will be used to gain insights into changes in the model's AUC during the training process. In this subsection, we first define the metric and then look into how it works with both intuitive understanding and theoretical proofs.

\subsubsection{\textbf{Definition:}}
Loss Entropy, $H_L$, is the entropy of the loss distribution output by the model, and it can be defined as follows:
\begin{equation}
    u_i =  \frac{f_M(x_i)}{\sum_{x \in D_{eval}} f_M(x)}, x_i \in D_{eval} 
    \label{entropy-possbility-def}
\end{equation}
\begin{equation}
    H_L = - \sum_{i} (u_i \log{u_i}), \quad s.t. \sum_i u_i = 1, u_i\geq0
    \label{entropy}
\end{equation}
Eq. \ref{entropy-possbility-def} denotes the operation to convert the outlier scores to the loss distribution while Eq. \ref{entropy} denotes the operation to compute the entropy for the loss distribution.
\newtext{
Compared with computing on the entire dataset $D$, computing on a subset is significantly more efficient while maintaining nearly intact performance. Since input samples in each batch are stochastically selected, fixing another constant set to calculate $H_L$ eliminates the influence of the stochasticity of the input batch. Therefore, we randomly sample $N_{eval}$  instances from $D$ to create the evaluation dataset  $D_{eval}$, ensuring both the efficiency and consistency of computing $H_L$.}

To ensure the integrity and consistency, when calculating entropy, we disable randomization techniques such as dropout. These techniques, however, may remain active during training. This approach mitigates potential variability in loss entropy estimation, thereby providing a more stable measurement.

\subsubsection{\textbf{Intuition Understanding}}
First, we will present the intuition behind how entropy works.
The basic assumption is that if $\mathcal{L}_{in}$ decreases much more than $\mathcal{L}_{out}$ (i.e. $\Delta \mathcal{L}_{in} \gg \Delta \mathcal{L}_{out}$) , then the model learns more  useful signals, leading to an increase in model's detection performance. Contrarily, the model learns more harmful signals if $\Delta \mathcal{L}_{in} \ll \Delta \mathcal{L}_{out}$.

Due to the intrinsic class imbalance, the shape of loss distribution can give insights into which part of signals the model learns more. 
Specifically, if $\Delta \mathcal{L}_{in} \gg \Delta \mathcal{L}_{out}$,  then  the majority of loss (i.e. $\{f_M(x_i), x_i \in D_{in}\}$) has a dramatic decline while the minority of loss (i.e. $\{f_M(x_i), x_i \in D_{out}\}$) remains relatively  large, leading to  a  steeper distribution. Conversely, when $ \Delta \mathcal{L}_{in} \ll \Delta \mathcal{L}_{out}$,  the   distribution will become flatter. Thus, the changes in the shape of the distribution can give some valuable insights into the variation in the latent detection capability. 

Interestingly, entropy itself can be utilized to gauge the shape of a distribution. When the distribution is more balanced, entropy tends to be higher, whereas a steeper distribution (i.e., when certain events have a higher probability of occurring) exhibits lower entropy \cite{entropy}. Thus, entropy inherently captures the variations in the shape of the loss distribution.

An example is shown in Fig. \ref{Fig:loss-distribution} to exhibit our intuition.  The red dashed vertical line  marks the $49^{th}$ iteration where AUC reaches its peak.  As shown in the figure, (1) The lowest $H_L$ exactly matches the optimal AUC in this example. (2) The change in the loss distribution from the $0^{th}$ iteration to the $49^{th}$ iteration corroborates our analysis that the loss of inliers drops intensely while the loss of outliers remains large.

\subsubsection{\textbf{Theoretical Proof:}}
% To prove that there is a negative correlation between AUC and $H_L$, we will use the loss gap as an intermediary. We have already proven in Theorem \ref{theorem-loss-gap-auc} that an increase in the loss gap is more likely to lead to an increase in the AUC. Next, w

We will demonstrate that an increase in the AUC is more likely to result in a decrease in $H_L$, under the assumption that \textit{inlier priority} holds.

\begin{theorem}
    \label{theorem-loss-gap-entropy}
        \newtext{When $\mathcal{L}_{in} < \mathcal{L}_{out}$ and the AUC increases, the $H_L$ is more likely to decrease.}
\end{theorem}

\begin{proof}
See Appx. \ref{proof: loss-gap-entropy} for the proof.
\end{proof}

\begin{figure*}
  \centering
  \includegraphics[width=120.00pt]{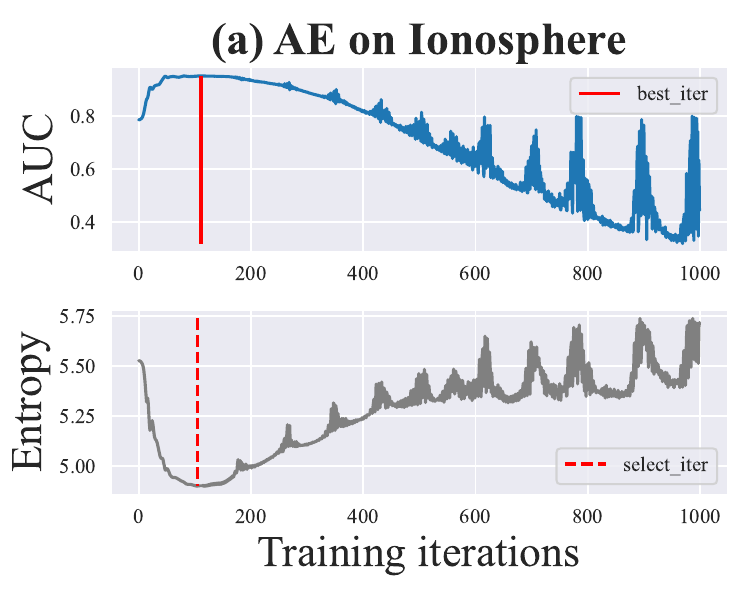}
  \includegraphics[width=120.00pt]{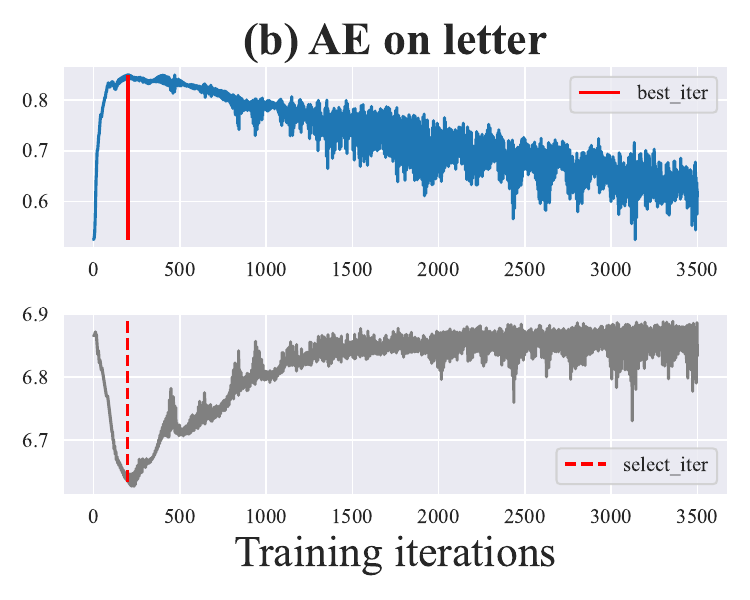}
  \includegraphics[width=120.00pt]{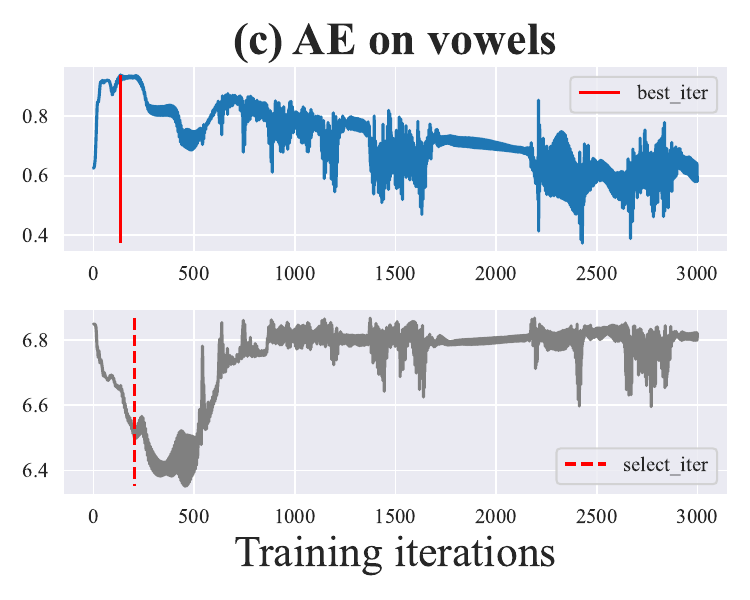}
  \includegraphics[width=120.00pt]{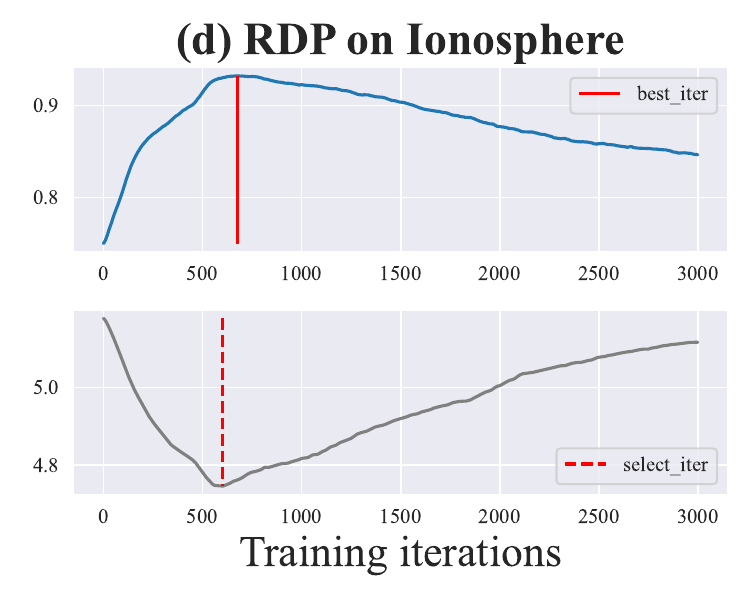}
  \caption{Examples of AUC and loss entropy curves during the training of AE and RDP \cite{RDP} on some datasets. ``select\_iter'' denotes the iteration selected by \textit{EntropyStop}. }
  \label{Fig:auc-entropy-corre}
  
\end{figure*}

Similarly, the converses of Theorems \ref{theorem-loss-gap-entropy} can also be proven by analogous reasoning. Thus, $H_L$ is expected to have a negative correlation with  detection capability, which paves the ways for our early stopping algorithm.

\subsection{EntropyStop: Automated Early Stopping Algorithm}
\label{entropy-stop-sec}
Based on the indicator $H_L$, we devise an algorithm to automated early stopping the unsupervised training before the model's detection performance is degraded by outlier.

Basically, we opt to stop training as soon as the entropy stops decreasing. 
Moreover, it is essential to ascertain that the curve which the lowest entropy lies on is relatively smooth with minor  fluctuations. Strong fluctuations may reflect analogous variations in the AUC, implying that the improvement in AUC lacks stability. 
We formulate our problem as  below.

\textbf{Problem Formulation.} Suppose $\mathcal{E} = \{e_j\}_{j=0}^E$ denotes the  entropy curve of model $M$. 
When $M$ finishs its $i^{th}$ training iteration, only the subcurve $\{e_j\}_{j=0}^i$ is available.
The goal is to select a point $e_i \in \mathcal{E}$ as early as possible that  (1) $\forall j < i, e_i < e_j$; (2) the subcurve $\{e_j\}_{j=0}^i$  has a smooth downtrend; (3) $\forall q \in (i,k + i)$, the subcurve $\{e_j\}_{j=i}^q$  has no smooth  downtrend.

% \begin{wrapfigure}{r}{0.5\textwidth}
%   \begin{center}
% \includegraphics[width=1.0\textwidth]{fig/alg-figure.pdf}
%   \end{center}
%   \caption{The whole process of \textit{EnttropyStop}}
%   \label{fig:alg}
% \end{wrapfigure}
\textbf{Algorithm.} In above formulation, $k$ is the patience parameter of algorithm. 
As an overview, our algorithm continuously explores new points within  $k$ iterations of the current lowest entropy point $e_i$, and tests whether the subcurve between the new point  and $e_i$ exhibits a smooth downtrend. Specifically,  when encountering a new point $e_q$,  we calculate $G = \sum_{j=i+1}^q (|e_j - e_{j-1}|)$ 
 as the total  variations of the subcurve $\{e_j\}_{j=i}^q$   and the downtrend of the subcurve is then quantified by $\frac{e_i - e_q}{G}$. Particularly, if the subcurve is monotonically decreasing, then $\frac{e_i - e_q}{G} = 1$. To test for a smooth downtrend, we use a threshold parameter $R_{down} \in (0,1)$. Only when $\frac{e_i - e_q}{G}$ exceeds $R_{down}$ will $e_q$ be considered as the new lowest entropy point. The complete process is shown in Algorithm \ref{alg:online}. In Fig. \ref{Fig:auc-entropy-corre}, we list a few examples to show the effectiveness of $EntropyStop$.

\RestyleAlgo{ruled}
\SetKwComment{Comment}{/* }{ */}
\SetKwInput{kwInput}{Input}
\SetKwInput{kwOutput}{Output}
\SetKwInput{kwReturn}{Return}
\SetKw{Break}{break}

\begin{algorithm}
\small
\caption{EntropyStop: An automated unsupervised training stopping algorithm for OD model}
\label{alg:online}
\kwInput{ Model $M$ with  learnable parameters $\Theta$, patience parameter $k$, downtrend threshold  $R_{down}$, dataset $D$,  iterations T, evaluation set size $N_{eval}$ }
\kwOutput{Outlier score list $\textbf{O}$}
Initialize  the parameter $\Theta$ of Model $M$\; Random sample $N_{eval}$ instances from  $D$ as the evaulation set $D_{eval}$ \;
$G \gets 0;$  $patience \gets 0; $ $ \Theta_{best} \gets \Theta; $ \; 
    Compute $H_L$ on $D_{eval}$. \;
  $e_0 \gets H_L$; $e^{min} \gets e_0$ ;
  \Comment*[r]{Model Training} 
\For {$j:= 1 \rightarrow T$}{
    Random sample a batch of training  data $D^{b}$ \;
    Calculate $\mathcal{L}_{train}$ on $D^{b}$\;
    Optimize the parameters $\Theta$  by minimizing $\mathcal{L}_{train}$\;
    Compute $H_L$ on $D_{eval}$ \;
  $ e_j \gets H_L$ ;
    $G \gets G + |e_j - e_{j-1}|$\;
\eIf{
$e_j < e^{min}$ and $\frac{e^{min}-e_j}{G} > R_{down}$
}{
$e^{min} \gets e_j;$ ; $G \gets 0;$  $patience \gets 0;$ $ \Theta_{best} \gets \Theta; $
}{
$patience \gets patience + 1$\;
}
\If{$patience = k$}{\Break}
}
Load the $\Theta_{best}$ to $M$ \;
\kwReturn{$ \{f_M(x), x\in D\}$}
\end{algorithm}

Two new parameters are introduced, namely $k$ and $R_{down}$. $k$ represents the patience for searching the optimal iteration, with larger value usually improving accuracy at the expense of longer training time. Then, $R_{down}$ sets the requirement for the smooth of downtrend. Apart from these two parameters, learning rate is also critical as it can significantly impact the training time. We recommend setting $R_{down}$ within the range of $[0.01,0.1]$, while the optimal value of $k$ and learning rate is associated with the actual entropy curve. We provide a guidance on tuning these HPs and parameter sensitivity study in Appx. \ref{appx:guide-for-tuning} and \ref{entropystop-hp-study}.

\subsection{Discussion}\label{discussion}
\noindent \textbf{Evaluation Cost.} Our algorithm incurs extra computational overhead with a time complexity of \(O(f_M(D_{eval}) + |D_{eval}|)\) due to the additional inference on \(D_{eval}\) for entropy calculation after each training iteration. However, as we observed in our experiments, deep UOD models often achieve its optimal AUC performance at an early stage, allowing training to be halted very soon. Therefore, employing our early stopping method can significantly reduce training time compared to arbitrarily setting a lengthy training duration.

\vspace{1mm}
\noindent \textbf{Pseudo inliers.} In dataset analysis, we found the existence of "Pseudo inliers" - \newtext{instances labeled as inliers but whose loss values are significantly larger than  the average of outlier losses.} The emergence of  pseudo inliers can be attributed to multiple factors: (1) multiple types of outliers exist in the dataset while the labels only cover one type; (2) As UOD methods make assumptions of outlier data distribution \cite{ADbench}, there is a mismatch between the assumptions of outlier distribution made by model and the labeled outlier distribution in the dataset. An extreme example of this is a breach of inlier priority, i.e., $L_{gap}<0$ throughout the training. 

The effectiveness of our proposed metric, $H_L$, may encounter challenges in such scenarios. This discrepancy often arises from the inherent limitations of unsupervised OD models or the dataset labels not comprehensively capturing all types of outliers. We delve into this issue through detailed case-by-case analyses in Appx. \ref{pseudo inlier-study}. The possible solution for this issue is to utilize a small number of labeled outliers to identify the alignment of the UOD assumptions and real datasets. We leave this as our future work.

%% file: Exp.tex
\definecolor{red-c}{rgb}{1, 0, 0}
\definecolor{blue-c}{rgb}{0, 0, 1}
\newcommand\best[1]{\textbf{\textcolor{red-c}{#1}}}
\newcommand\second[1]{\textit{\textcolor{blue-c}{#1}}}

In this section, we evaluate the effectiveness of our proposed metric ($H_L$) and the entropy-based early stopping algorithm ($EntropyStop$) through comprehensive experiments. Our key findings are summarized as follows:
\begin{itemize}
    \item  $EntropyStop$ remarkably improves AE model performance, surpassing ensemble AE models and significantly reducing  training time. (See Sec. \ref{sec: ensemble-cmp-exp})
     \item We observe a strong negative correlation between the $H_L$ curve and AUC curve across a larger number of real-world datasets, which verifies our analysis. (see Sec. \ref{sec:correlation-exp})
    \item 
    %Employing $EntropyStop$ can adpot to various HP configurations to improve the model's performance and enhance the robustness to HP sensitivity in some cases. (See Sec \ref{sec: hp-sen-exp})
    Our $EntropyStop$ can be applied to other deep UOD models, exhibiting their broad potential applicability. (See Sec. \ref{sec: model-expansion})
\end{itemize}

\subsection{Experiment setting}
\newtext{All experiments adopt a transductive setting, where the training set equals the test set, which is common in Unsupervised OD \cite{robod,randnet}.}
\subsubsection{\textbf{Dataset}} 
Experiments are carried on 47 widely-used real-world tabular datasets\footnote{https://github.com/Minqi824/ADBench/} collected by \cite{ADbench}, which cover many application domains, including healthcare, image processing, finance, etc. Details on dataset description can be found in Appx. \ref{appendix:dataset-pool}.

\subsubsection{\textbf{Evaluation Metrics}}
We evaluate performance w.r.t. two metrics that are based on AUC and Average Precision (AP). \newtext{Computing AUC and AP does not need a threshold for outlier scores outputted by model, as they are ranking-based metrics.}

\subsubsection{\textbf{Computing Infrastructures}}
All experiments are conducted on Ubuntu 22.02 OS, AMD Ryzen 9 7950X CPU, 64GB memory, and an RTX 4090 (24GB GPU memory) GPU.

\subsection{Improvements and Efficiency Study}
\label{sec: ensemble-cmp-exp}
We first study how much improvement can be achieved by employing $EntropyStop$ for the AE model. 
The simplest form of AE without any additional techniques, is denoted as VanillaAE.
We apply our early stopping method to VanillaAE to gain \textbf{\textit{EntropyAE}}.
We  compare our approach with two ensemble AEs, including the recent SOTA  hyper-ensemble ROBOD \cite{robod} and the widely-used  RandNet \cite{randnet}.  
The experiments of two ensemble models are based on the open-source code of ROBOD\footnote{https://github.com/xyvivian/ROBOD}. The detailed HP configuration of them can be found in Appx. \ref{appx-sec: HP config of Exp1}.
\begin{table}[!htbp]
\small
  \centering
  \caption{Detection performance  of  models from AE family. $p < 0.05$ means there is a signicant difference between the baseline and $EntropyAE$. See Detailed data in Table \ref{tab:ae-auc} and \ref{tab:ae-ap}.}
  \resizebox{\columnwidth}{!}{
    \begin{tabular}{c|c|ccc}
    \toprule
    \multicolumn{1}{c|}{} & VanillaAE & \multicolumn{1}{c}{EntropyAE} (Ours) & RandNet & ROBOD \\
    \midrule
    \multicolumn{1}{c|}{$\overline{AUC}$} & 0.741$\pm$0.001 & \multicolumn{1}{c}{\best{0.768$\pm$0.005}} & 0.728$\pm$0.00 & 0.736$\pm$0.00 \\
    \multicolumn{1}{c|}{$\overline{AP}$} & 0.299$\pm$0.005 & \multicolumn{1}{c}{\best{0.364$\pm$0.009}} & 0.358$\pm$0.00 & 0.360$\pm$0.00 \\
    \midrule
    $\overline{Rank}_{AUC}$ & 2.70  & \best{2.14} & 2.68  & 2.42  \\
    $\overline{Rank}_{AP}$ & 2.85  & \best{2.23} & 2.51  & 2.36  \\
    \midrule
     ${p}^{auc}$ & \best{0.006}  & -- & \best{0.013}  & \best{0.023}  \\
    ${p}^{ap}$ & \best{0.000}  & -- & 0.355  & 0.402  \\
    \bottomrule
    \end{tabular}%
    }
  \label{tab:AE-ensemble-compare}%
\end{table}%

\subsubsection{\textbf{Detection Performance Result}}
The average result of five runs is reported in Table \ref{tab:AE-ensemble-compare}.
We conducted a comparative analysis of four UOD methods across 47 datasets, evaluating average AUC, average AP, average ranking in AUC, and average ranking in AP. It is evident that EntropyAE not only significantly outperforms VanillaAE but also surpasses ensemble models in AUC and is marginally superior in AP. P-value from the one-sided paired Wilcoxon signed-rank test is presented as well, emphasizing the statistical significance of the improvements achieved by EntropyAE.  It is shown that, compared to VanillaAE, EntropyAE achieves a substantial enhancement by employing early stopping.

\subsubsection{\textbf{Efficiency Result}}
To quantify the extent to which early stopping reduces training time, we employ the following metric:
\begin{equation}
\text{\textit{Average Train Time}}(M) = \frac{1}{|\mathcal{D}|} \sum_{D \sim \mathcal{D}} \frac{\text{training time}(M, D)}{\text{training time}(VanillaAE, D)}
\end{equation}
\begin{equation}
\text{\textit{Total Train Time}}(M) =   \frac{\sum_{D \sim \mathcal{D}} \text{training time}(M, D)}{\sum_{D \sim \mathcal{D}} \text{training time}(VanillaAE, D)}
\end{equation}
where $D$ represents one of the 47 datasets, $\mathcal{D}$ denotes the collection of all 47 datasets, and $M$ signifies any model among VanillaAE, EntropyAE, RandNet, and ROBOD. We ensure that all models have the same batch size \newtext{of 64} and number of epochs \newtext{of 250} to guarantee identical iteration counts.
\newtext{$\text{\textit{Train Time}}(M)$ reveals the average relative training time required compared to VanillaAE while $\text{\textit{Total Train Time}}(M)$ reveals the total  time required compared to VanillaAE.} \newtext{In   Table \ref{tab:AE-ensemble-Rescompare},} we observe that, compared to VanillaAE, ROBOD, and RandNet, EntropyAE only requires \newtext{under 8\%, 2\%, and  0.3\%} of the \newtext{average} training time, respectively. \newtext{For the total training time, the advantage of $EntropyAE$ are more significant.} This demonstrates the effectiveness of early stopping in saving time.
% Similarly, we use $Memory(M)$ to calculate the relative memory consumption compared to VanillaAE. This metric clearly shows that the early stopping algorithm, compared to ensemble methods, does not impose a proportional overhead. The memory consumption is consistent with that of VanillaAE.
The detailed comparison result on efficiency are list in Table \ref{tab:AE-ensemble-Rescompare}.

\begin{table}[!htbp]
\small
  \centering
  \caption{Comparison of training time  for AEs.  See Detailed data in Table \ref{tab:ae-time}.}
    \begin{tabular}{c|c|ccc}
    \toprule
    \multicolumn{1}{c|}{} & VanillaAE & \multicolumn{1}{c}{EntropyAE} & RandNet  & ROBOD \\
    \midrule
       \textit{Average Train Time} & 1   & \best{0.077} & 23.05 
  &  3.51 \\
   \textit{Total Train Time} & 1   & \best{0.01} & 35.03 
  &  4.02 \\
    \bottomrule
    \end{tabular}%
  \label{tab:AE-ensemble-Rescompare}%
\end{table}%

\vspace{-2mm}
\newtext{Figure \ref{Fig:time-data-size} displays the time required by EntropyAE across 47 datasets. The early stopping mechanism is more effective on larger datasets, as they contain more batches per epoch, resulting in more iterations.   In some large datasets, EntropyAE stops training before completing a single epoch.}

\begin{figure}[!htbp]
\vspace{-3mm}
  \centering
  
  \includegraphics[scale=0.48]{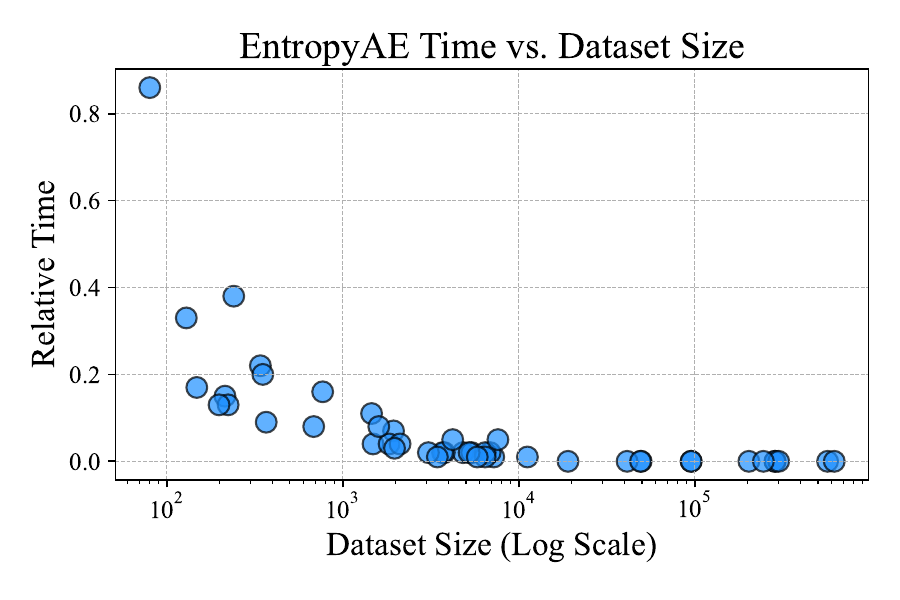}
    \vspace{-5mm}
  \caption{The relative time (compared to VanillaAE) taken by EntropyAE across different dataset sizes.}

  \label{Fig:time-data-size}
  \vspace{-2mm}
\end{figure}

\subsection{\textbf{Negative Correlation Study}}
In this experiment, our objective is to carefully evaluate the efficacy of our proposed zero-label metric, loss entropy (\(H_L\)), in accurately reflecting variations in the label-based AUC.
We commence our analysis by visualizing the AUC and \(H_L\) curves for each dataset. In addition, we utilize the Pearson correlation coefficient to statistically measure such negative correlation. 
Specifically, we run AE model \newtext{and linear DeepSVDD \cite{deep-svdd}} on 47 datasets with a 0.001 learning rate and 500 full batch training iterations.

\begin{figure}[!htbp]
  \centering
  \includegraphics[scale=0.35]{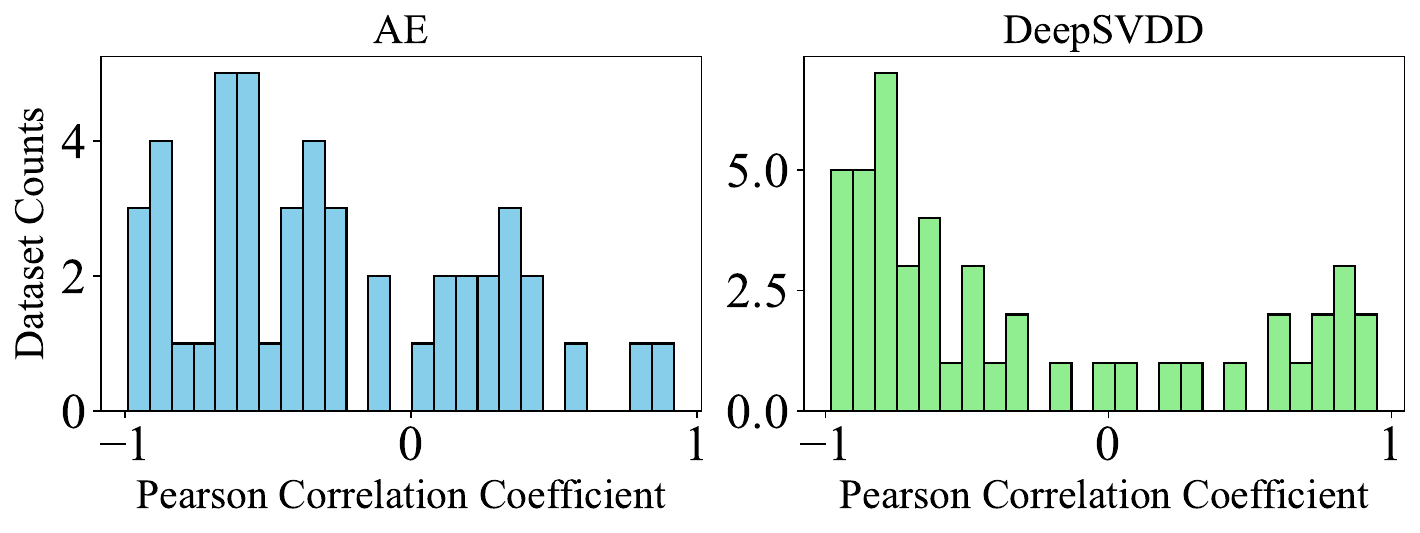}
  \caption{Analysis of Pearson correlation coefficients between AUC and $H_L$
  curves across 47 datasets: lower coefficients indicate stronger negative correlations}
  \label{Fig:pearsonr}
\end{figure}
\label{sec:correlation-exp}

\subsubsection{\textbf{Result.}}
All the AUC and $H_L$ curves 
  are shown in Fig. \ref{Fig:all-curve-1}, \ref{Fig:all-curve-2}, \ref{Fig:all-curve-3}, \ref{Fig:all-curve-4} (AE) and Fig. \ref{Fig:svdd-all-curve-1}, \ref{Fig:svdd-all-curve-2}, \ref{Fig:svdd-all-curve-3}, \ref{Fig:svdd-all-curve-4} (DeepSVDD) in Appx. \ref{pseudo inlier-study} to demonstrate the negative correlation between the two. The distribution of Pearson correlation coefficient values across 47 datasets are shown in Fig. \ref{Fig:pearsonr}.  These results show that while $H_L$ has a strong negative correlation with AUC on \newtext{more than half of the 47 datasets}, the remaining datasets show a weak or even positive correlation. Basically, the reason for invalidity can be attributed to the following aspects:
\begin{itemize}
    \item \textbf{Label misleading}: The existence of a large number of pseudo inliers on these datasets. \newtext{The pseudo inliers are regarded as outliers by UOD model  while labeled as inliers.}
    \item \textbf{The convergence of AUC}: the AUC is nearly stationary during the whole training process, thereby the entropy could not reflect the changes of AUC. In this case, the ineffectiveness of $H_L$ actually does not influence the final performance, while time is still saved by early stopping.
\end{itemize}
In Appx. \ref{pseudo inlier-study}, we conduct case-by-case analyses of the invalid reasons of AE on these datasets.
Interestingly, although $H_L$ does not perform well on some datasets, we view this as an opportunity to highlight the inherent limitations of unsupervised OD algorithms and to discuss these critical issues: (1) The labeling of outliers in the dataset is erroneous or exclusively focuses on a single type of outliers. (2) The model's outlier assumption does not align with the labeled outliers in the dataset, suggesting the need to explore other UOD models for outlier detection.

Through comprehensive analysis, we discovered that $H_L$
  demonstrates widespread applicability across a diverse range of datasets, while scenarios of inapplicability are specifically and reasonably explained. This provides future researchers with deeper understandings of our algorithm, features of outlier distribution and the general mechanism of UOD paradigm.

% In this experiment, we also conducted a comprehensive visualization of the training process for the AE model across a diverse set of 47 datasets. The corresponding AUC and entropy curves for each dataset are delineated in Fig \ref{Fig:all-curve-1} and Fig \ref{Fig:all-curve-2} in Appendix. \ref{appendix:experiment-detail}. 
% This negative relationship is substantiated across a substantial corpus of datasets, thus corroborating our theoretical conjectures. It is salient to note that the computation of score entropy is independent of labels, relying exclusively on the distribution of outlier scores generated by the model. 

% Figure \ref{Fig:sperman} enumerates the spearman correlation coefficients between the AUC and  $H_L$ curves.
% It is apparent that certain datasets exhibit a strong negative correlation between AUC and the $ H_L $ curve, while others do not manifest this relationship as strongly. This variability underscores that the underlying assumptions of our analysis may not hold uniformly across all real-world datasets. 
% Nevertheless, the observed trends underscore the utility of $ H_L $ as a novel internal evaluation metric. Moreover, these findings highlight the necessity of incorporating a limited number of labels to ascertain the existence of such a negative correlation, a facet we aim to address in our future work.

\subsection{Model Expansion Experiment}
\label{sec: model-expansion}
In this subsection, we include more deep UOD models for experiments, i.e., AE, DeepSVDD \cite{deep-svdd}, RDP \cite{RDP}, NTL \cite{NTL} and LOE \cite{LOE}. From another perspective, our early stopping algorithm can also be regarded as selecting the best model from all models - each at an arbitrary iteration - during the training process. Therefore, we can reduce the optimal iteration selection problem to the model selection problem.
In this case, we also investigate the improvement of $EntropyStop$ on some Unsupervised Outlier Model Selection (UOMS) \cite{Internal-evaluation-paper} methods. 

\vspace{1mm}
\noindent \textbf{UOMS Baselines}: 
UOMS solutions aim at selecting a best pair \{\verb|Algorithm, HP|\} among a pool of options, solely relying on the outlier scores and the input data (without labels). We compare $EntropyStop$ with baselines including Xie-Beni index (XB) \cite{xb}, ModelCentrality (MC) \cite{MC}, and HITS \cite{Internal-evaluation-paper}.
In additional,  we add  two additional baselines, \textit{Random} and \textit{Vanilla}, which refer to the average performance of all iterations and the performance of the final iteration, respectively. Moreover, \textit{Max} denotes the maximum performance among the whole training process   (i.e., the upper bound) is also shown. 
The detailed experiment setup can be found at Appx. \ref{appx-sec: model expansion config}. 
The experiments are conducted on 47 datasets and each item in a table represents the average value over all datasets. 
For each dataset $D$, the UOMS baselines receive a collection of outlier score lists among 300 training iterations, $\mathcal{S} = \{\textbf{s}_i\}_{i=0}^{300}$, as their input. From these, the models produce an output consisting of a single outlier score list, $\textbf{s}_i \in \mathbb{R}^{|D|}$, which represents the outlier scores from the chosen iteration. This specific score list is then utilized to calculate the AUC metric for performance evaluation.

\subsubsection{\textbf{Result}}

The AUC and AP results are shown in Table \ref{tab:auc-uoms} and Table \ref{tab:ap-uoms}, respectively. The second best score is marked in blue italics. It is observed that (1) our solution exhibits more effectiveness in selecting the optimal iteration, especially for AE and DeepSVDD.  It's important to highlight that our approach is also extendable to other deep UOD models.  (2) In addition, \textit{Random} baseline and \textit{Vanilla} baseline rank second on more than half the rows, which reveals that none of existing UOMS solutions can help select the optimal iteration, nor can they fulfill the task of early stopping.

% Table generated by Excel2LaTeX from sheet 'Sheet1'
\begin{table}[htbp]
\small
  \centering
  \caption{AUC for the optimal iteration selection}
   \setlength\tabcolsep{2.1pt}
    \begin{tabular}{l|c|cccccc}
    \toprule
          & Max & \textbf{Ours}  & XB    & MCS   & HITS  & Random & Vanilla \\
    \midrule
    AE    & 0.806  & \best{0.768} & 0.720  & \second{0.745}  & 0.734  & 0.742  & 0.744  \\
    RDP \cite{RDP}  & 0.798  & \best{0.754} & 0.734  & 0.737  & 0.739  & \second{0.741}  & 0.735  \\
    NeuTraL \cite{NTL} & 0.758  & \best{0.701} & 0.309  & 0.692  & 0.658  & 0.641  & \second{0.693}  \\
    NeuTraL+$LOE_H$ \cite{LOE} & 0.748  & \best{0.696} & 0.328  & 0.679  & 0.661  & 0.634  & \second{0.693}  \\
    DeepSVDD \cite{deep-svdd} & 0.747  & \best{0.679} & 0.654  & 0.652  & 0.657  & \second{0.664}  & 0.637  \\
    \bottomrule
    \end{tabular}%
  \label{tab:auc-uoms}%
\end{table}%

% Table generated by Excel2LaTeX from sheet 'Sheet1'
\begin{table}[htbp]
\small
  \centering
  \caption{AP for the optimal iteration selection}
  \setlength\tabcolsep{2.5pt}
    \begin{tabular}{l|c|cccccc}
    \toprule
          & Max & \textbf{Ours}  & XB    & MCS   & HITS  & Random & Vanilla \\
    \midrule
    AE    & 0.420  & \best{0.364} & 0.287  & 0.303  & 0.302  & \second{0.309}  & 0.303  \\
    RDP   & 0.412  & 0.343  & 0.313  & \second{0.351}  & \best{0.352} & 0.349  & 0.350  \\
    NeuTraL & 0.304  & \best{0.251} & 0.112  & \second{0.243}  & 0.240  & 0.227  & 0.242  \\
    NeuTraL+LOE & 0.297  & \best{0.234} & 0.121  & 0.229  & 0.226  & 0.212  & \second{0.230}  \\
    DeepSVDD & 0.402  & \best{0.331} & 0.308  & 0.312  & 0.312  & \second{0.318}  & 0.308  \\
    \bottomrule
    \end{tabular}%
  \label{tab:ap-uoms}%
\end{table}%

The running time on all datasets are shown in Fig. \ref{Fig:uoms-eff}. The training time of AE is also plotted as \textit{Train}. It reveals that existing UOMS solutions are quite inefficient, where MCS is even several orders of magnitude slower than the training time of AE. Our solution is much more efficient than UOMS baselines.

\begin{figure}
  \centering
  \includegraphics[width=180.00pt]{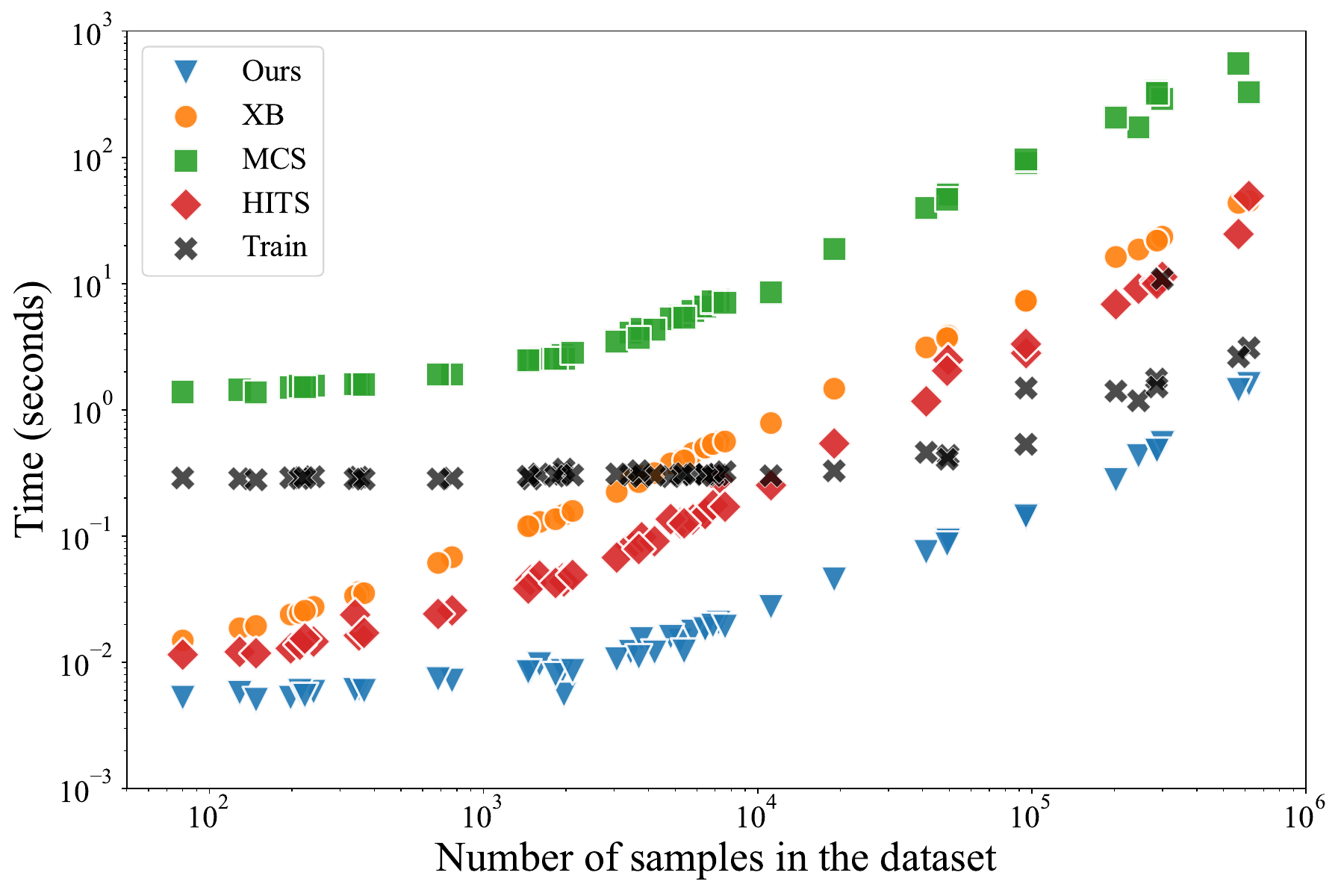}
  \caption{Efficiency of UOMS and our solution.}
  \label{Fig:uoms-eff}
\end{figure}

%% file: proof.tex
\clearpage
\section{Theoretical proof}\label{theproof}
% There are two proofs totally: the Effectiveness Proof for Entropy-stop (this section) and the Derivation of the range (next section). \textit{Notice that some notations in the two proofs have different meanings; 0 is used to denote inlier class and 1 for outlier class.}
\newtext{
In this section, we aim to provide a theoretical basis for the negative correlation between loss entropy and AUC. We demonstrate that when AUC increases, loss entropy is more likely to decrease.
% A correlation between AUC and loss entropy can be established with loss gap as a bridge. Specifically, when the loss gap increases, AUC has more chances to increase and loss entropy has more chances to decrease. The increment of loss gap affect both AUC and loss entropy simultaneously.
The converse can be proven by analogous reasoning. Therefore, we do not provide a separate proof for the converse.}

\subsection{Notations and definitions}
We summarize here the notations for the effectiveness proof for entropy-stop. \textit{At any time $t$}, we denote current training dynamics as:
\begin{itemize}
    \item $n$ : the number of samples used to evaluate the model $M$'s detection capability in each iteration, i.e. $N_{eval}$.
    \item $f_M(\cdot)$: the unsupervised loss function and outlier score function of model $M$.
    \item $\mathcal{X}$: the distribution of data.
    \item $\mathcal{X}_{in}$: the normal data distribution.
    \item $\mathcal{X}_{out}$: the anomalous data distribution.
    \item $\mathcal{O} = \{f_M(x) | x \sim \mathcal{X}\}$: the loss distribution outputted by $M$.
    \item $\mathcal{O}_{in}$: the loss distribution from inliers.
    \item $\mathcal{O}_{out}$: the loss distribution from outliers.
   \item $v$: the random variable of loss value, i.e., $v \sim \mathcal{O}$.
    \item $v^+$: the random variable that $v^+ \sim \mathcal{O}_{out}$.
    \item $v^-$: the random variable that $v^- \sim \mathcal{O}_{in}$.
    \item $\mathcal{V} = \left \{ v_1, \dots , v_n \right \} $: the set of unsupervised losses calculated over all samples. $\forall v_i, v_i>0$.
    \item $\mathcal{V}^+$: the set of unsupervised losses calculated over all abnormal samples.
    \item $\mathcal{V}^-$: the set of unsupervised losses calculated over all normal samples.
    \item $\rho(\cdot)$: the Probability Density Function (PDF) of $\mathcal{O}$
    \item $\rho(\cdot)^+$: the Probability Density Function (PDF) of $\mathcal{O}_{out}$
    \item $\rho(\cdot)^-$: the Probability Density Function (PDF) of $\mathcal{O}_{in}$
    % \item $p^{i}(v)$: the probability density function of $i$-class losses. Then we define random variable $v^i \sim f^{i}(v)$.
    \item $\alpha$: the ratio of outliers in all data samples, $\alpha \in \left ( 0,1 \right )$.
    % \item $f(v)$: the probability distribution function of $v$. Since $v$ is sampled from inlier or outlier losses, we have $v\sim f(v)=\alpha f_{1}(v)+(1-\alpha)f_{0}(v)$.
    \item $S = \sum_{i=1}^{n} v_i $: the sum of all losses in $V$. 
    \item $u_i = \frac{v_i}{S}$: the normalized loss value.
    \item $U = \{u_i\}_{i=0}^n$: the set of normalized loss values.
    \item $H_L = H(U) = - \sum_{i=1}^{n} u_i \log u_i$: \textit{Loss entropy}.
    \item $\mathcal{N}^\prime$: Corresponding value of notation $\mathcal{N}$ at time $t+1$. For example, $v_i^\prime$ means the $i$-th loss value in the next iteration. Then we define $\Delta\mathcal{N}=\mathcal{N}^\prime-\mathcal{N}$.
\end{itemize}

Then we make following definitions:
\begin{enumerate}
\item AUC: the performance indicator, which is:
\[\frac{1}{|\mathcal{V}^-| |\mathcal{V}^+|} \sum_{v^-_i \in \mathcal{V}^-} \sum_{v^+_j \in \mathcal{V}^+}\mathbb{I}(v^-_i < v^+_j) = P(v^-<v^+)\]
\item loss gap: $E(v^+) - E(v^-) = E(v^+-v^-)$, the  difference of average loss value between two classes. 
\item $\delta=v^+-v^-$: the random variable of loss gap. 
\item speed gap: $E(\Delta v^+)-E(\Delta v^-) = E(\Delta v^+-\Delta v^-)$, the difference of the decreasing speed of averaged loss value between two classes.  \newtext{Note that $\Delta v = v' - v$}.
\item $\Delta \delta=\Delta v^+-\Delta v^-$: the random variable of speed gap.
\end{enumerate}

\subsection{AUC and Entropy}
\label{proof: loss-gap-entropy}
We aim to prove that when AUC increases, $H(V_t)$ also has more possibility to decrease, which has following mathematical form:

\begin{gather}
P(H(V_t)>H(V_{t+1})\mid P(\delta + \Delta \delta >0) > P(\delta>0) )>0.5 \notag
\end{gather}

Basically, $P(\delta + \Delta \delta >0) > P(\delta>0) )$ means that the new AUC is larger than the original AUC. We divide the proof into 2 steps, providing them in Section \ref{sub_proof_1} and \ref{sub_proof_2}.

\subsubsection{Assumptions}
\begin{assumption}[inlier priority]\label{E1>E0}
$E(\delta)>0$.
\end{assumption}
\newtext{First, we assume that the outliers have a larger expectation of averaged loss value, which is the concept of \textit{inlier priority} mentioned in Section \ref{inlier-priority-sec}.}

\begin{assumption}
$\Delta S<0
% , E(\Delta v^-<0), E(\Delta v^+<0)
$.
\end{assumption}
Second, we assume that the losses continue to be minimized by the optimizer.

\begin{assumption}\label{AUC<1}
$P(\delta>0)<1$.
\end{assumption}
We also assume $\text{AUC}<1$. Otherwise, there is no room for AUC to increase anymore.

\begin{assumption}
The random variable $v$ is distributed according to the probability density function $\rho(v)=\alpha \rho^+(v)+(1-\alpha)\rho^-(v), \alpha \in \left [ 0,1 \right ]$, in which $\rho^+(v)$ and $\rho^-(v)$ are the PDFs of the distribution of $v^+$ and $v^-$, respectively. 
\label{distribute-sample}
\end{assumption}
\newtext{Here, $\alpha$ denotes the outlier ratio of data. Assumption \ref{distribute-sample} implies that the random variable v has a $\alpha$ probability of being sampled from $\rho^+(v)$ and a 1-$\alpha$ probability of being sampled from $\rho^-(v)$.}

\begin{assumption}
$$P(\delta > 0, \Delta \delta >0) = P(\delta >0)P(\Delta \delta >0)$$
\end{assumption}
\newtext{Since $\delta$ and $\Delta \delta$ do not strongly correlate, we assume that $\delta >0$ and $\Delta \delta >0$ are unrelated for simplifying our analysis.
}

\begin{assumption}
\label{AUC>0.5}
 $AUC\ge 0.5$
\end{assumption}
We assume that the detector 's performance is better than random guess.
\newtext{In most cases, this assumption can be easily satisfied due to the effectiveness of UOD algorithms.}

\begin{assumption}\label{assumption3}
$u_i \in (0, \frac{1}{e})$
\end{assumption}
\newtext{Given that $\sum_{i=1}^{|D|} u_i = 1, u_i > 0$, and the dataset size $|D|$ usually satisfies $|D| \gg e$, we assume that $u_i < \frac{1}{e}$.}

\begin{assumption}\label{small delta v}
$\Delta v_i$ is sufficiently small.
\end{assumption}
\newtext{Basically, a small learning rate is set to ensure the convergence of the learning algorithm, thereby resulting in minimal changes in loss values.}

\begin{assumption}
\label{key-assumption}
 $$E(\delta)>0, P(\delta + \Delta \delta >0) > P(\delta >0) \rightarrow P(\Delta \delta > 0) > 0.5$$
 $$E(\delta)>0, P(\delta + \Delta \delta >0) < P(\delta >0)  \rightarrow P(\Delta \delta > 0) < 0.5$$
\end{assumption}
\newtext{Here, we assume that if the loss gap exists and the $AUC$ increases (or decreases) after a single gradient update, the decrease in outliers' losses is more (or less) likely to be smaller than the decrease in inliers' losses.}

\subsubsection{Subproof 1}
\label{sub_proof_1}
The first subproof is: if $$P(\delta + \Delta \delta >0) > P(\delta>0)$$
then $$\Delta u_i > \Delta u_j \to P(u_i > u_j) > 0.5$$

\begin{proof}\label{SP1}

With $\Delta S < 0$ and $\Delta u_i > \Delta u_j$, we can deduce $\Delta v_i > \Delta v_j$. Since both losses $v_i$ and $v_j$ can be sampled from either $\mathcal{O}_{out}$ and $\mathcal{O}_{in}$, $P(u_i>u_j \mid \Delta v_i > \Delta v_j)$ equals to the sum of four conditional probabilities:
\begin{align}
     &P(u_i>u_j \mid \Delta v_i > \Delta v_j)=P(v_i>v_j\mid \Delta v_i > \Delta v_j)  \notag
% \\  =&P(v_i\sim \mathcal{O}_{out}, v_j\sim \mathcal{O}_{in} \mid \Delta v_i > \Delta v_j) \label{fxfy}
% \\  &\qquad P(v_i>v_j \mid v_i\sim \mathcal{O}_{out}, v_j\sim \mathcal{O}_{in}, \Delta v_i > \Delta v_j) \notag
% \\  &+P(v_i\sim \mathcal{O}_{in}, v_j\sim \mathcal{O}_{out} \mid \Delta v_i > \Delta v_j)\notag
% \\  &\qquad P(v_i>v_j \mid v_i\sim \mathcal{O}_{in}, v_j\sim \mathcal{O}_{out}, \Delta v_i > \Delta v_j)\notag
\\  =& P(v_i>v_j, v_i\sim \mathcal{O}_{out}, v_j\sim \mathcal{O}_{in} \mid \Delta v_i > \Delta v_j)\label{fxfy}
\\  &+P(v_i>v_j, v_i\sim \mathcal{O}_{in}, v_j\sim \mathcal{O}_{out} \mid \Delta v_i > \Delta v_j)\notag
\\  &+P(v_i>v_j, v_i\sim \mathcal{O}_{out}, v_j\sim \mathcal{O}_{out} \mid \Delta v_i > \Delta v_j)\label{fxfx}
\\  &+P(v_i>v_j, v_i\sim \mathcal{O}_{in}, v_j\sim \mathcal{O}_{in} \mid \Delta v_i > \Delta v_j)\notag
% \\  =&P(v_i\sim \mathcal{O}_{out}, v_j\sim \mathcal{O}_{in} \mid \Delta v_i > \Delta v_j)P(v_i>v_j \mid v_i\sim \mathcal{O}_{out}, v_j\sim \mathcal{O}_{in}, \Delta v_i > \Delta v_j)\notag
% \\  &+P(v_i\sim \mathcal{O}_{in}, v_j\sim \mathcal{O}_{out} \mid \Delta v_i > \Delta v_j)P(v_i>v_j \mid v_i\sim \mathcal{O}_{in}, v_j\sim \mathcal{O}_{out}, \Delta v_i > \Delta v_j)\notag
% \\  &+0.5\alpha^2+0.5(1-\alpha)^2\notag
\end{align}
where
\begin{align}
    &P(v_i>v_j, v_i\sim \mathcal{O}_{out}, v_j\sim \mathcal{O}_{out} \mid \Delta v_i > \Delta v_j) \notag
    \\=&\frac{P(v_i>v_j, \Delta v_i>\Delta v_j \mid v_i\sim \mathcal{O}_{out}, v_j\sim \mathcal{O}_{out})P(v_i\sim \mathcal{O}_{out}, v_j\sim \mathcal{O}_{out})}{P(\Delta v_i > \Delta v_j)} \notag
    \\=&\frac{0.25\alpha^2}{0.5}=0.5\alpha^2 \notag
\end{align}
and
\begin{align}
    &P(v_i>v_j, v_i\sim \mathcal{O}_{out}, v_j\sim \mathcal{O}_{in} \mid \Delta v_i > \Delta v_j) \notag
    \\=&\frac{P(v_i>v_j, \Delta v_i>\Delta v_j \mid v_i\sim \mathcal{O}_{out}, v_j\sim \mathcal{O}_{in})P(v_i\sim \mathcal{O}_{out}, v_j\sim \mathcal{O}_{in})}{P(\Delta v_i > \Delta v_j)} \notag
    % \\=&\frac{P(v_i>v_j \mid v_i\sim \mathcal{O}_{out}, v_j\sim \mathcal{O}_{in})P(\Delta v_i>\Delta v_j \mid v_i\sim \mathcal{O}_{out}, v_j\sim \mathcal{O}_{in})P(v_i\sim \mathcal{O}_{out}, v_j\sim \mathcal{O}_{in})}{P(\Delta v_i > \Delta v_j)} \notag
    \\=&(P(\Delta v_i > \Delta v_j))^{-1}P(v_i>v_j \mid v_i\sim \mathcal{O}_{out}, v_j\sim \mathcal{O}_{in}) \notag
    \\ &\qquad P(\Delta v_i>\Delta v_j \mid v_i\sim \mathcal{O}_{out}, v_j\sim \mathcal{O}_{in})P(v_i\sim \mathcal{O}_{out}, v_j\sim \mathcal{O}_{in}) \notag
    \\=&\frac{AUC \cdot P(\Delta \delta > 0)\alpha(1-\alpha)}{0.5}=2\alpha(1-\alpha)AUC \cdot P(\Delta \delta > 0) \notag
\end{align}

% As AUC increases,
% i.e., $P(\delta + \Delta \delta >0 ) > P(\delta>0)$, we can infer that $P(\Delta \delta >0) > 0.5$.

Similarly we calculate the other two terms in the equation. Then,
\begin{align}
&P(u_i>u_j \mid \Delta v_i > \Delta v_j) \notag
\\=&0.5\alpha^2+0.5(1-\alpha)^2 \notag
\\+&2\alpha(1-\alpha)\Big(AUC\cdot P(\Delta \delta >0)+(1-AUC)\cdot \big(1-P(\Delta \delta >0)\big)\Big) \notag
\end{align}
With $AUC\ge 0.5, P(\Delta \delta >0)>0.5$ from Assumption \ref{AUC>0.5} and \ref{key-assumption}, we can infer $P(u_i>u_j \mid \Delta v_i > \Delta v_j)>0.5$.
\end{proof}
\subsubsection{Subproof 2}
\label{sub_proof_2}
The second sub-proof is dedicated to demonstrating that it is more likely for the loss entropy to decrease, i.e.,
$$
P(H_L \searrow) > 0.5
$$

From Subproof \ref{SP1}, we have:

\begin{equation}
    \Delta u_i > \Delta u_j \to P(u_i > u_j) > 0.5 
    % \label{lemmaSP1} 
    \label{lemmaSP1}
\end{equation}

\begin{proof}
Loss entropy equals to:
\begin{align}
    H(U) &= - \sum_{i=1}^{n} u^i \log u^i \notag\\
           &= \sum_{i=1}^{n} h(u_i) \notag
\end{align}
where $h(u_i) = -u_ilog(u_i)$.
We can derive that 
$$h'(u) =-(log(u) + 1) $$
$$h''(u) = -\frac{1}{u}$$
where $h'(u)$ is the first derivative of $h(u)$ and $h''(u)$ is the second derivative of $h(u)$. 
This suggests that 
in the domain \( u \in (0, \frac{1}{e}) \), the variable \( u \) exhibits a monotonic increase, with its impact on \( h(u) \) being inversely proportional to its magnitude; namely,
\begin{equation}
     h'(u) >0, u \in (0, \frac{1}{e})
\end{equation}
\begin{equation}
u_i > u_j \rightarrow h'(u_i) < h'(u_j)
\label{eq-entropy-proof=1}
\end{equation}

According to Eq. \ref{lemmaSP1}, we can derive that 
\begin{equation}
  \Delta u_i > \Delta u_j \rightarrow P(h'(u_i) < h'(u_j)) > 0.5
  \label{delta u key eq}
\end{equation}

As $\sum_i u_i = \sum_i u'_i = 1$. Therefore, 
 \begin{equation}
     \sum_{i: \Delta u_i > 0} \Delta u_i = - \sum_{i: \Delta u_i < 0} \Delta u_i  
     \label{delta u equal negative}
 \end{equation}
which means the sum of all positive $\Delta u_i$ equals the negative of the sum of all negative $\Delta u_i$. 

Given that \(\Delta u\) is sufficiently small (i.e., Assumption \ref{small delta v}), we can perform a Taylor expansion on \(H(U')\):
\begin{align}
H(U') & = \sum_{i: \Delta u_i > 0} h(u_i + \Delta u_i) + \sum_{i: \Delta u_i < 0} h(u_i + \Delta u_i) \\
& \approx \sum_{i} h(u_i) + \sum_{i} h'(u_i)\Delta u_i \\
& = H(U) + \sum_{i} h'(u_i)\Delta u_i
\end{align}

\newtext{
Accoring to Eq. \ref{delta u key eq} and Eq. \ref{delta u equal negative},we can derive:
\begin{align}
\Delta u_i > \Delta u_j & \rightarrow  P(\sum_{i: \Delta u_i > 0} h'(u_i) \Delta u_i < - \sum_{i: \Delta u_i < 0} h'(u_i) \Delta u_i) > 0.5 \\
& \rightarrow  P(\sum_{i} h'(u_i)\Delta u_i<0) > 0.5 \\
& \rightarrow  P(H(U') < H(U)) > 0.5 \\
& \rightarrow  P(H_L \searrow) > 0.5
\end{align}
}

\end{proof}

\newtext{Thus, we prove that if AUC increases, then $P(H_L\searrow)>0.5$.
Similarly, the converse of theorem can also be proven by analogous reasoning.
% Since the theorem has a symmetric form with respect to class label, it is also true that if $E(\Delta \delta)<0$, then $P(\Delta \delta >0)<0.5$ and $P(Entropy\searrow)<0.5$.
This means during the training, the trend of AUC and loss entropy have a negative correlation with each other, giving the theoretical guarantee of the algorithm.}

%% file: Appx-Exp.tex
\clearpage
\section{Experiment Details}
\label{appendix:experiment-detail}
\subsection{Real-world Outlier Detection Datasets}
\label{appendix:dataset-pool}

We construct our experiments using 47 benchmark datasets commonly employed in outlier detection research, as shown in Table \ref{pool-dataset}.

% Table generated by Excel2LaTeX from sheet 'data_info'
\begin{table}[htbp]
\small
  \centering
  \caption{Real-world dataset pool}
    \begin{tabular}{cl|rrr}
    \toprule
          & \textbf{Dataset} & \textbf{Num Pts} & \textbf{Dim} & \textbf{\% Outlier} \\
    \midrule
    1     & ALOI  & 49534 & 27    & 3.04  \\
    2     & annthyroid & 7200  & 6     & 7.42  \\
    3     & backdoor & 95329 & 196   & 2.44  \\
    4     & breastw & 683   & 9     & 34.99  \\
    5     & campaign & 41188 & 62    & 11.27  \\
    6     & cardio & 1831  & 21    & 9.61  \\
    7     & Cardiotocography & 2114  & 21    & 22.04  \\
    8     & celeba & 202599 & 39    & 2.24  \\
    9     & census & 299285 & 500   & 6.20  \\
    10    & cover & 286048 & 10    & 0.96  \\
    11    & donors & 619326 & 10    & 5.93  \\
    12    & fault & 1941  & 27    & 34.67  \\
    13    & fraud & 284807 & 29    & 0.17  \\
    14    & glass & 214   & 7     & 4.21  \\
    15    & Hepatitis & 80    & 19    & 16.25  \\
    16    & http  & 567498 & 3     & 0.39  \\
    17    & InternetAds & 1966  & 1555  & 18.72  \\
    18    & Ionosphere & 351   & 32    & 35.90  \\
    19    & landsat & 6435  & 36    & 20.71  \\
    20    & letter & 1600  & 32    & 6.25  \\
    21    & Lymphography & 148   & 18    & 4.05  \\
    22    & magic & 19020 & 10    & 35.16  \\
    23    & mammography & 11183 & 6     & 2.32  \\
    24    & mnist & 7603  & 100   & 9.21  \\
    25    & musk  & 3062  & 166   & 3.17  \\
    26    & optdigits & 5216  & 64    & 2.88  \\
    27    & PageBlocks & 5393  & 10    & 9.46  \\
    28    & pendigits & 6870  & 16    & 2.27  \\
    29    & Pima  & 768   & 8     & 34.90  \\
    30    & satellite & 6435  & 36    & 31.64  \\
    31    & satimage-2 & 5803  & 36    & 1.22  \\
    32    & shuttle & 49097 & 9     & 7.15  \\
    33    & skin  & 245057 & 3     & 20.75  \\
    34    & smtp  & 95156 & 3     & 0.03  \\
    35    & SpamBase & 4207  & 57    & 39.91  \\
    36    & speech & 3686  & 400   & 1.65  \\
    37    & Stamps & 340   & 9     & 9.12  \\
    38    & thyroid & 3772  & 6     & 2.47  \\
    39    & vertebral & 240   & 6     & 12.50  \\
    40    & vowels & 1456  & 12    & 3.43  \\
    41    & Waveform & 3443  & 21    & 2.90  \\
    42    & WBC   & 223   & 9     & 4.48  \\
    43    & WDBC  & 367   & 30    & 2.72  \\
    44    & Wilt  & 4819  & 5     & 5.33  \\
    45    & wine  & 129   & 13    & 7.75  \\
    46    & WPBC  & 198   & 33    & 23.74  \\
    47    & yeast & 1484  & 8     & 34.16  \\
    \bottomrule
    \end{tabular}%
  \label{pool-dataset}
\end{table}%

\subsection{Configuration of Improvement Study}
\label{appx-sec: HP config of Exp1}

%  Here, we detail the hyperparameter (HP) configuration settings for the experiments in Section \ref{sec: ensemble-cmp-exp}.
% For Randnet and ROBOD, we employed the default HP conf from the publicly available code  of ROBOD \footnote{https://github.com/xyvivian/ROBOD}, which are epoch=250, batchsize=256, and lr=0.001. We retained the default Autoencoder (AE) architecture provided in the code. Regarding the number of models in the ensemble, Randnet ensembles 10 models with different random seeds and pre-trains for 100 epochs, while ROBOD ensembles 16 models with varying HP configurations. The AE without any additional technique is named as VanillaAE. The design of VanillaAE is relatively straightforward, which dimensions are $[d_{in}, 64, d_{in}]$, where $d_{in}$ is the dimension of input samples. We set epoch=250, batch_size=1024, lr=0.001 for VanillaAE and use Adam as the optimizer.
% We apply $EntropyStop$ for early stopping on top of VanillaAE, naming the resulting model EntropyAE. The parameter settings for EntropyStop are $k$=100, $R_{down}$=0.1, $N_{eval}$=1024.
% In Section x, we explain how to set the parameters of EntropyStop and present the sensitivity of EntropyAE to different $R_{down}$ and $N_{eval}$ values.

In this segment, we elaborate on the HP configuration settings utilized for the experiments delineated in Sec. \ref{sec: ensemble-cmp-exp}.
For Randnet and ROBOD, the default HP configurations from ROBOD's publicly accessible repository\footnote{https://github.com/xyvivian/ROBOD} were adopted, specified as epochs=250, batch size=1024, and learning rate (lr) of 0.001. The Autoencoder (AE) architecture defined within the codebase was maintained without modifications. Concerning ensemble size, Randnet amalgamates ten models, each initialized with distinct random seeds and subjected to a pre-training phase of 100 epochs, whereas ROBOD aggregates sixteen models, each featuring unique HP configurations. Our $EntropyStop$ is applied to a simple AE model. The simplest form of AE, devoid of any supplementary techniques, is denoted as VanillaAE. VanillaAE's architecture is designed for simplicity, with dimensions $[d_{in}, 64, d_{in}]$, where $d_{in}$ represents the dimensionality of the input vectors. For VanillaAE, we designated epochs=250, batch size=1024, lr=0.001, and employed Adam as the optimizer.
The $EntropyStop$ technique is integrated for early termination within VanillaAE's training process, with the modified model termed as EntropyAE. Parameters for $EntropyStop$ are set to $k$=100, $R_{down}$=0.1, and $N_{eval}$=1024.
In Appx. \ref{appx:guide-for-tuning} and \ref{entropystop-hp-study}, we explain how to set the parameters of EntropyStop and present the sensitivity of EntropyAE to different $R_{down}$ and \textit{batch size}.

\subsection{Configuration of Model Expansion Experiment}
\label{appx-sec: model expansion config}
More deep-based OD models are experimented based on their original open-source code\footnote{https://github.com/billhhh/RDP}\footnote{https://github.com/boschresearch/LatentOE-AD}. Among them, NeuTraL\footnote{https://github.com/boschresearch/NeuTraL-AD} \cite{NTL} and DeepSVDD \cite{deep-svdd} are two  OD models that are actually trained on clean dataset. 
 For these models, we trained them for 300 epochs using a full batch size approach. Additionally, we adhered to the default hyperparameter settings as specified in their original codebases.
 For UOMS solutions,
Xie-Beni index (XB) \cite{xb}, ModelCentrality (MC) \cite{MC}, and HITS
 \cite{Internal-evaluation-paper} are the baselines for comparison. These baselines have been evaluated their effectiveness in selecting models among a large pool of traditional UOD algorithms in \cite{Internal-evaluation-paper} with published open-source code\footnote{http://bit.ly/UOMSCODE}.  We follow  \cite{Internal-evaluation-paper} to use a lightweight version of MC, called MCS, to reduce its time complexity and $logN$ models are sampled for computing the Kendall $\tau$ coefficient.  For each dataset $D$, the input of these baselines is the set of outlier score lists $\mathcal{S}$ ($|\mathcal{S}|=300$) while the output is the outlier score list $\textbf{s}_i \in \mathbb{R}^{|D|}$ of the selected epoch.   The average result with three runs is reported.

\begin{figure*}
  \centering
\includegraphics[width=0.32\textwidth]{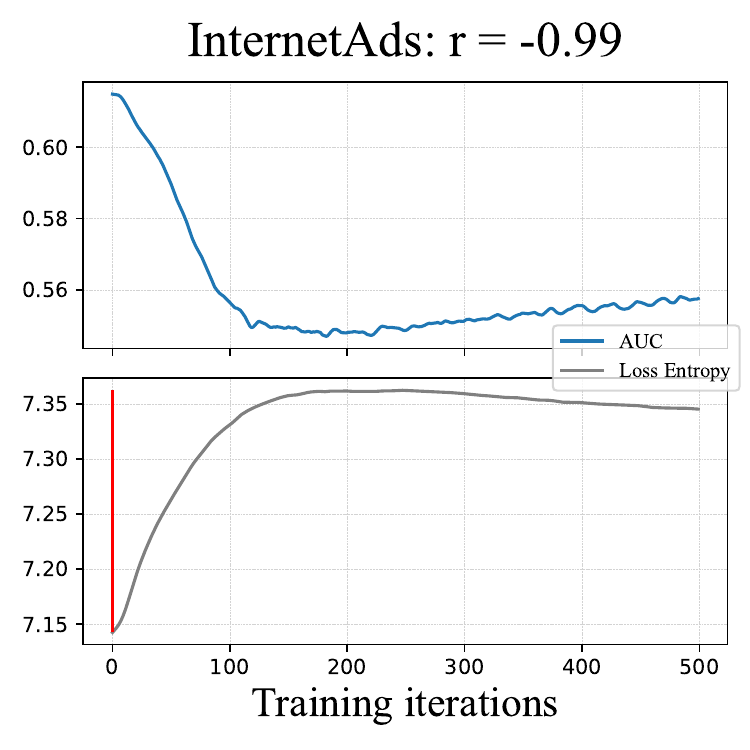}
\includegraphics[width=0.32\textwidth]{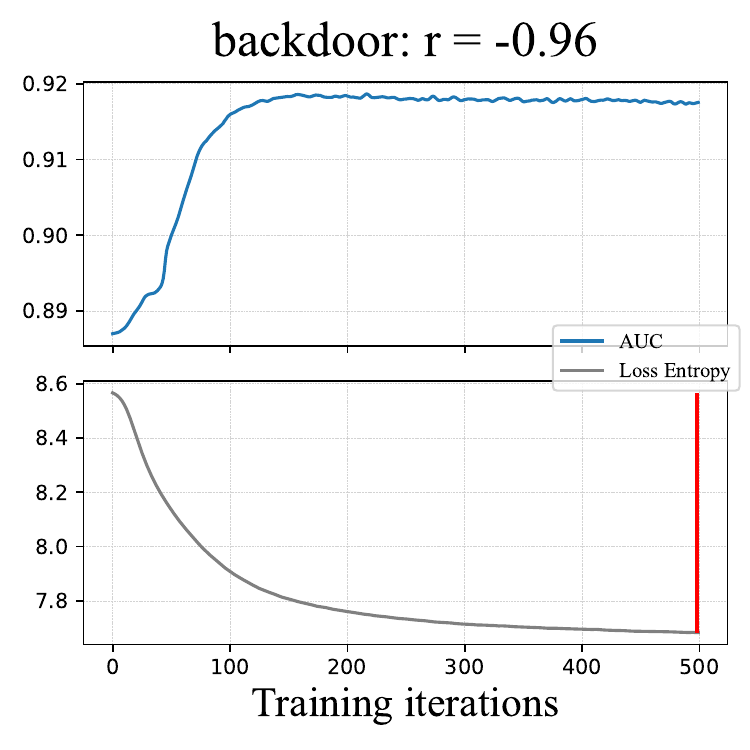}
\includegraphics[width=0.32\textwidth]{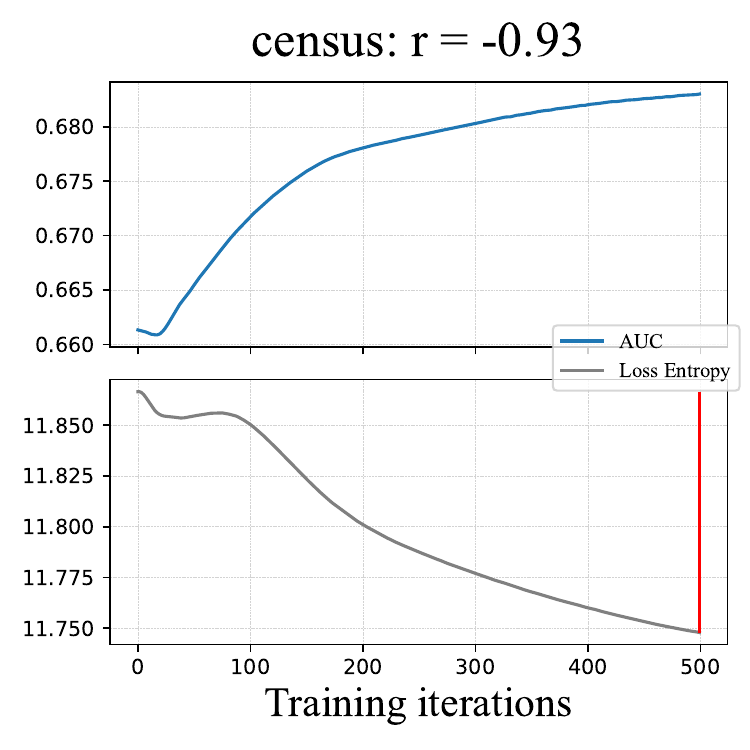}
\includegraphics[width=0.32\textwidth]{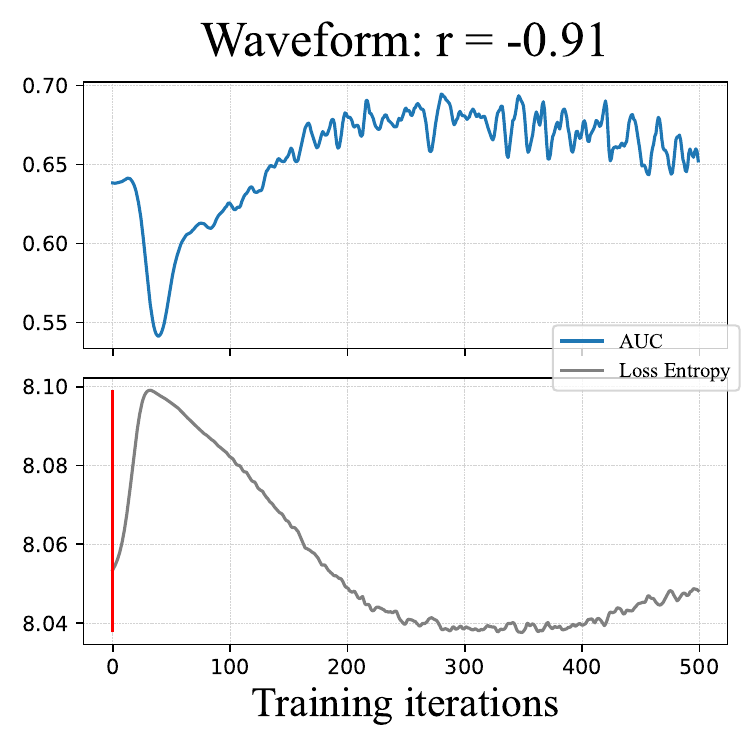}
\includegraphics[width=0.32\textwidth]{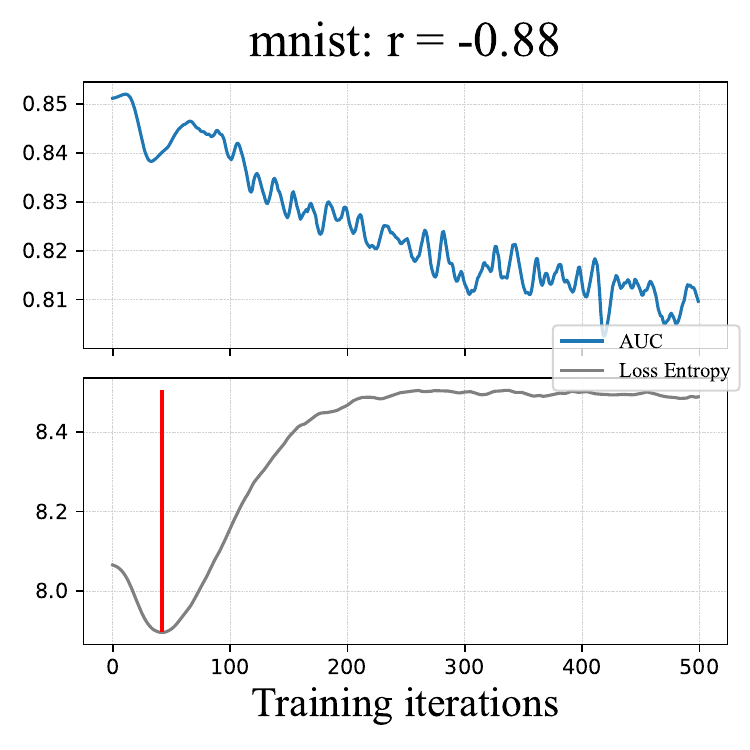}
\includegraphics[width=0.32\textwidth]{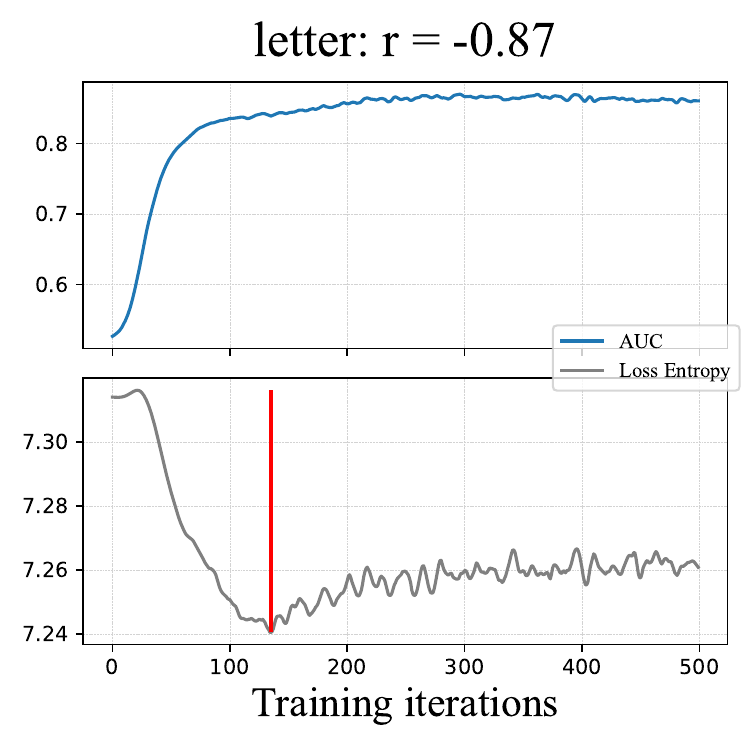}
\includegraphics[width=0.32\textwidth]{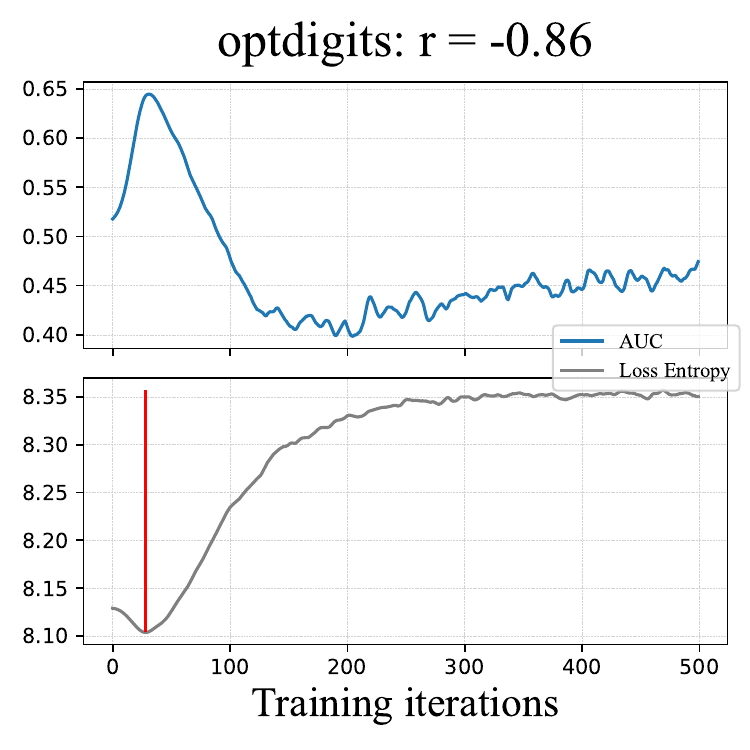}
\includegraphics[width=0.32\textwidth]{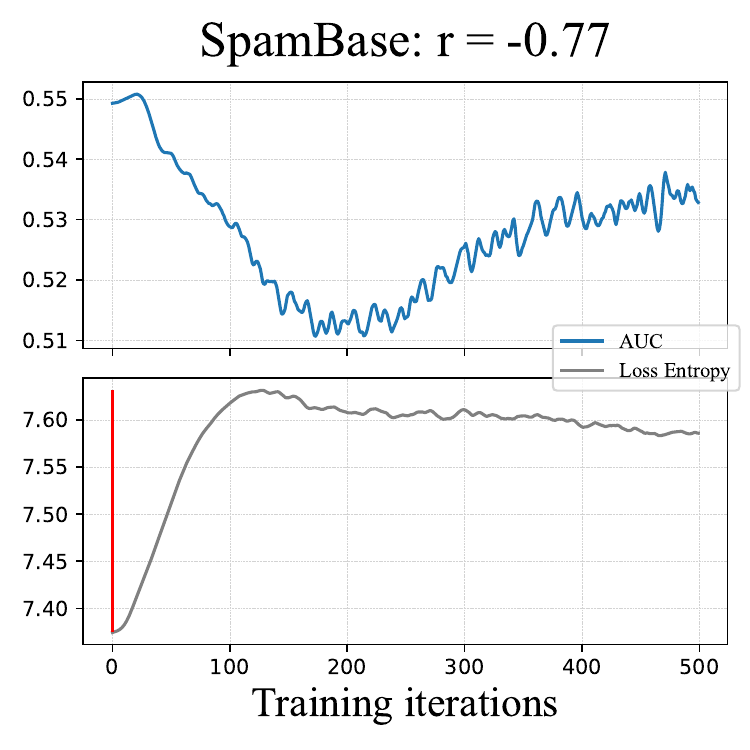}
\includegraphics[width=0.32\textwidth]{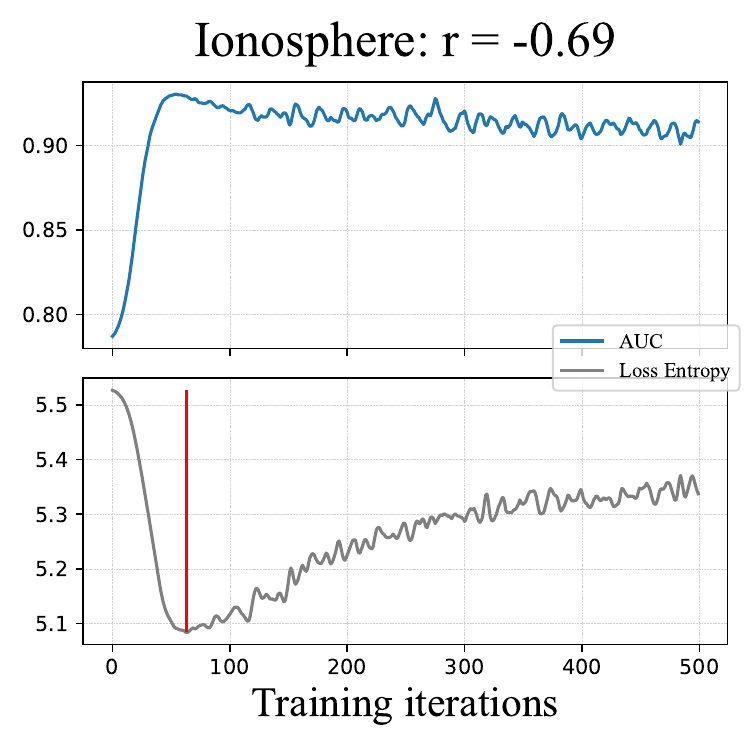}
\includegraphics[width=0.32\textwidth]{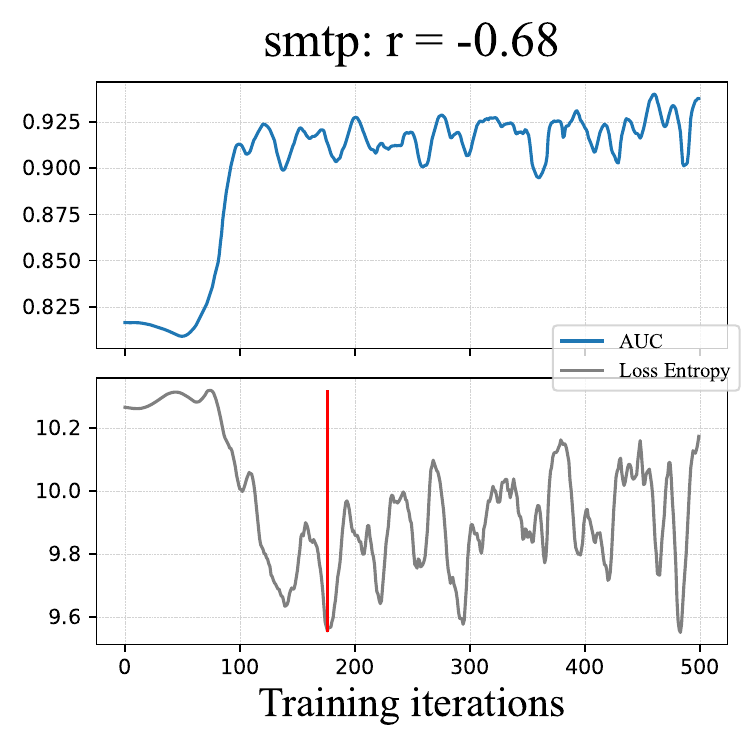}
\includegraphics[width=0.32\textwidth]{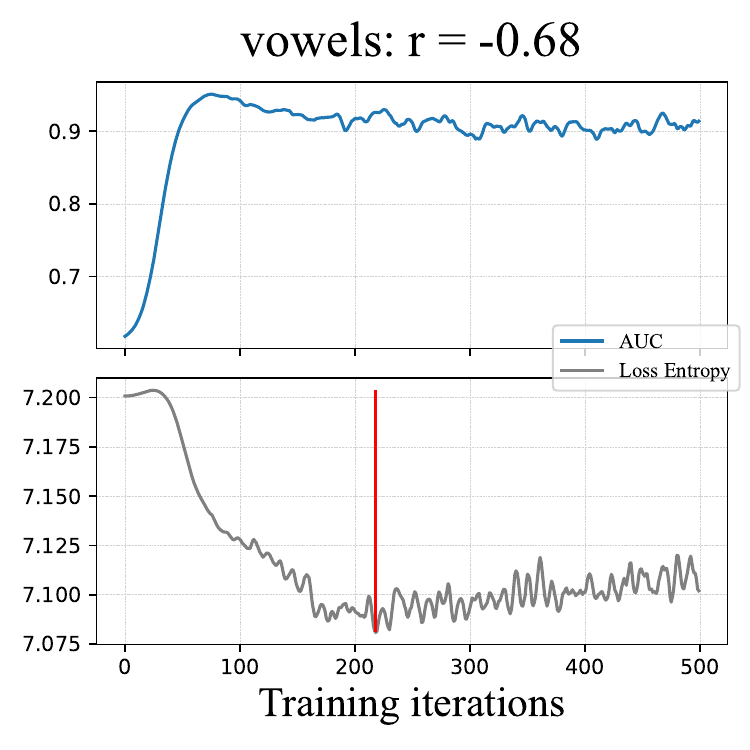}
\includegraphics[width=0.32\textwidth]{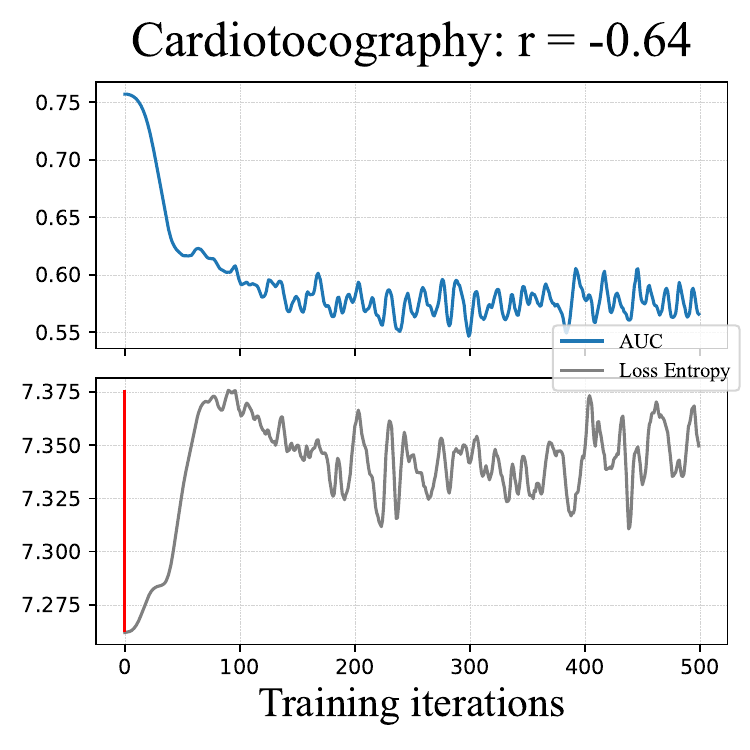}
  \caption{AE:  AUC curves vs  $H_L$ curves. The red vertical line is the epoch selected by $EntropyStop$. $r$ denotes the Pearson correlation coefficient between AUC and $H_L$.}
  \label{Fig:all-curve-1}
\end{figure*}

\begin{figure*}
  \centering
\includegraphics[width=0.32\textwidth]{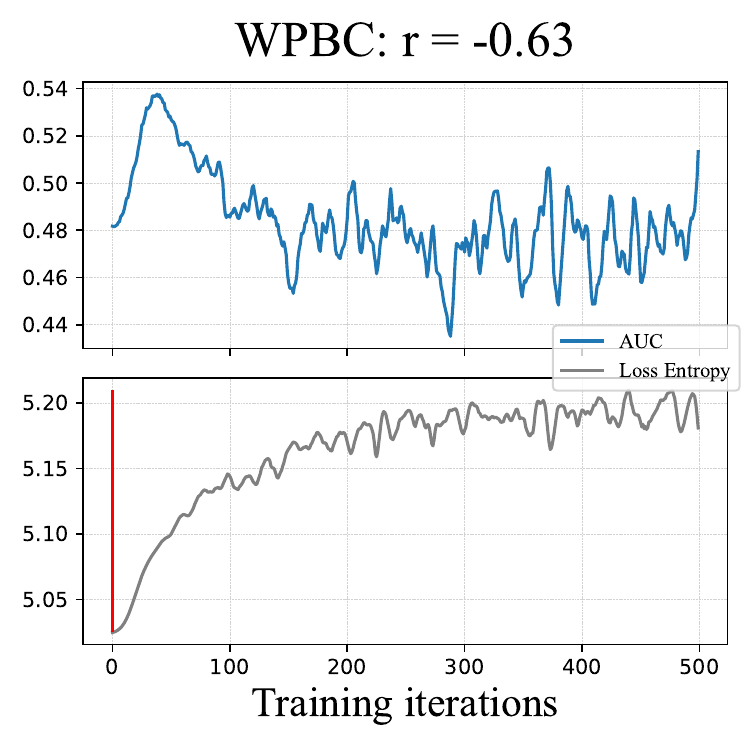}
\includegraphics[width=0.32\textwidth]{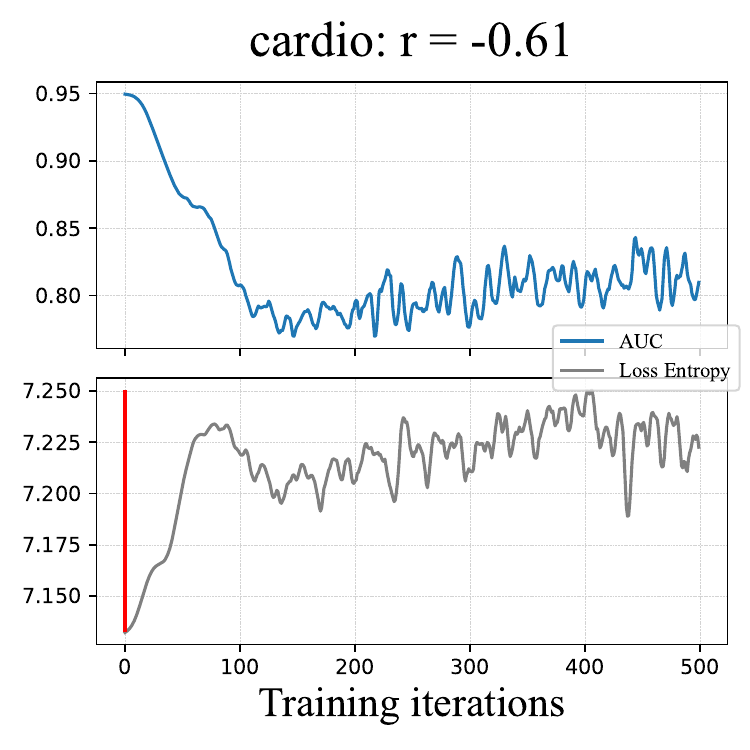}
\includegraphics[width=0.32\textwidth]{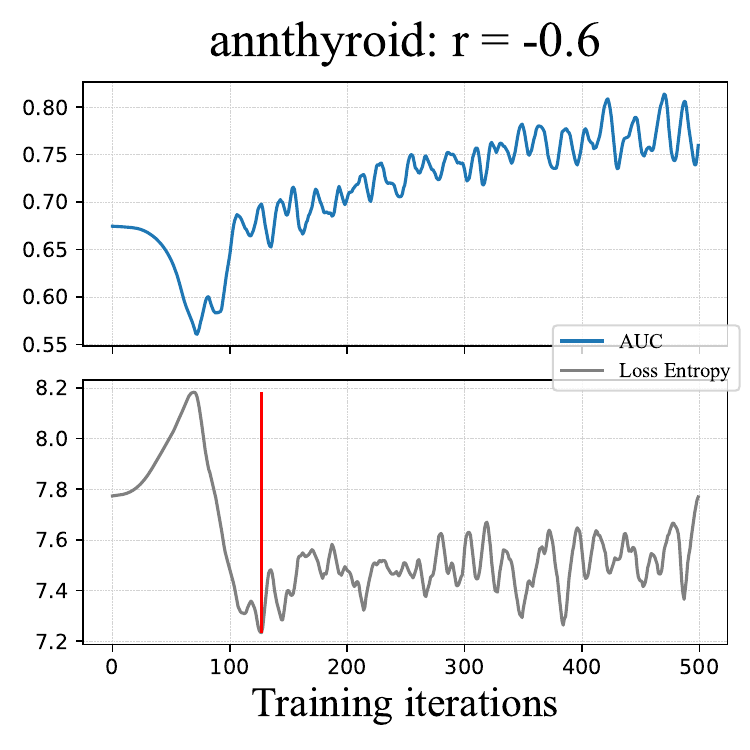}
\includegraphics[width=0.32\textwidth]{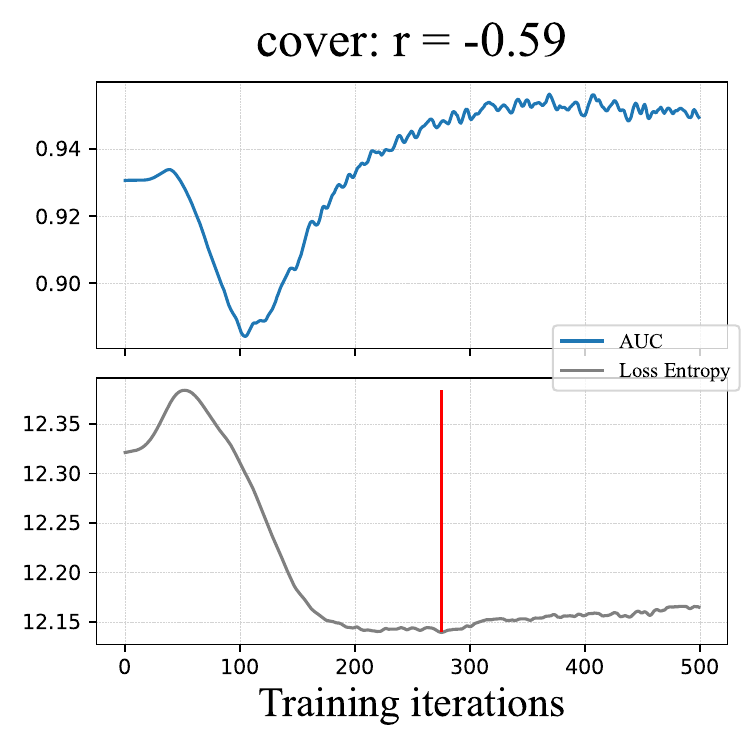}
\includegraphics[width=0.32\textwidth]{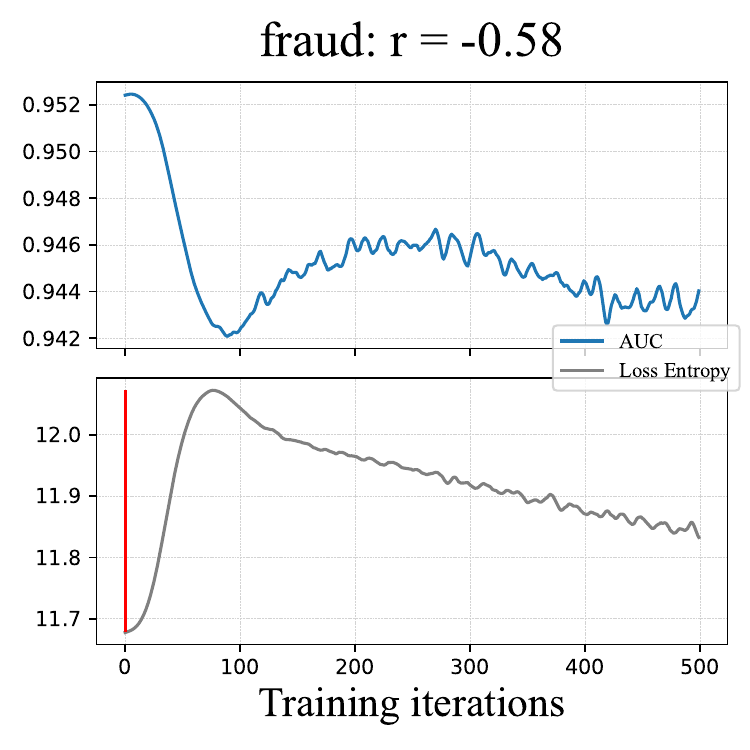}
\includegraphics[width=0.32\textwidth]{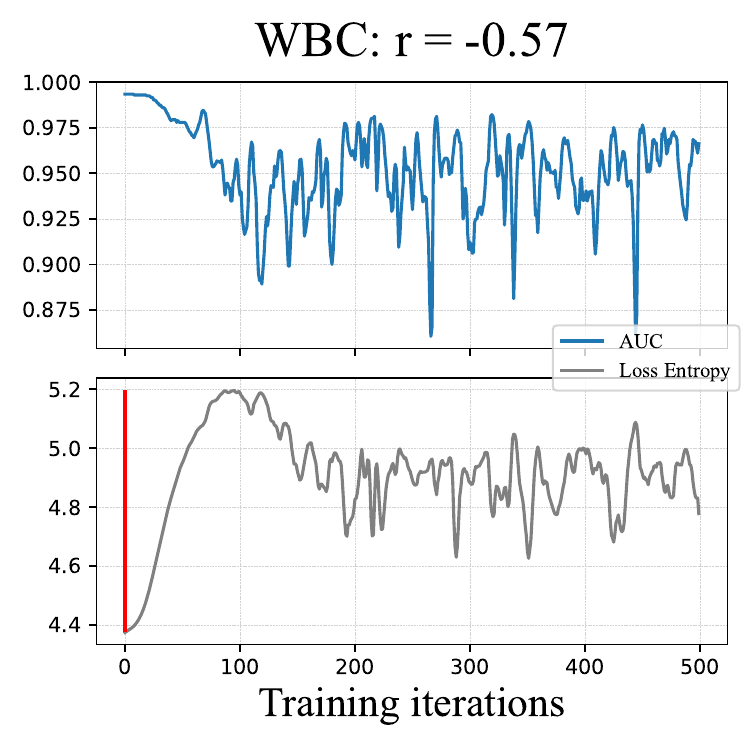}
\includegraphics[width=0.32\textwidth]{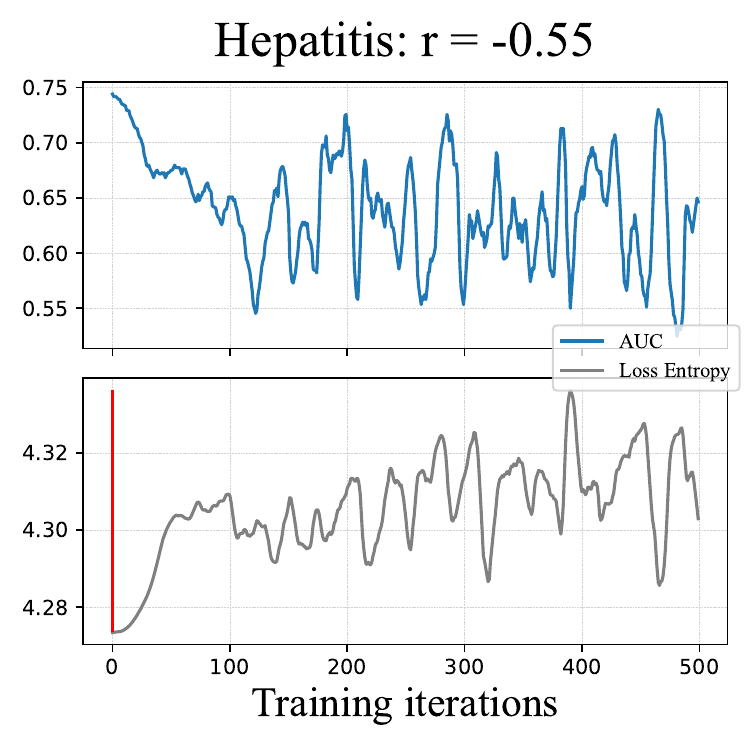}
\includegraphics[width=0.32\textwidth]{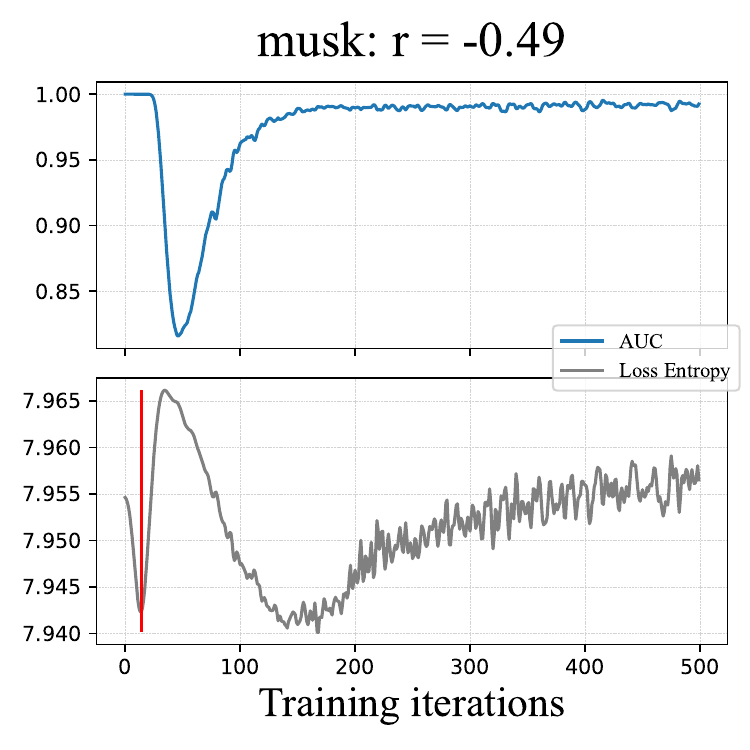}
\includegraphics[width=0.32\textwidth]{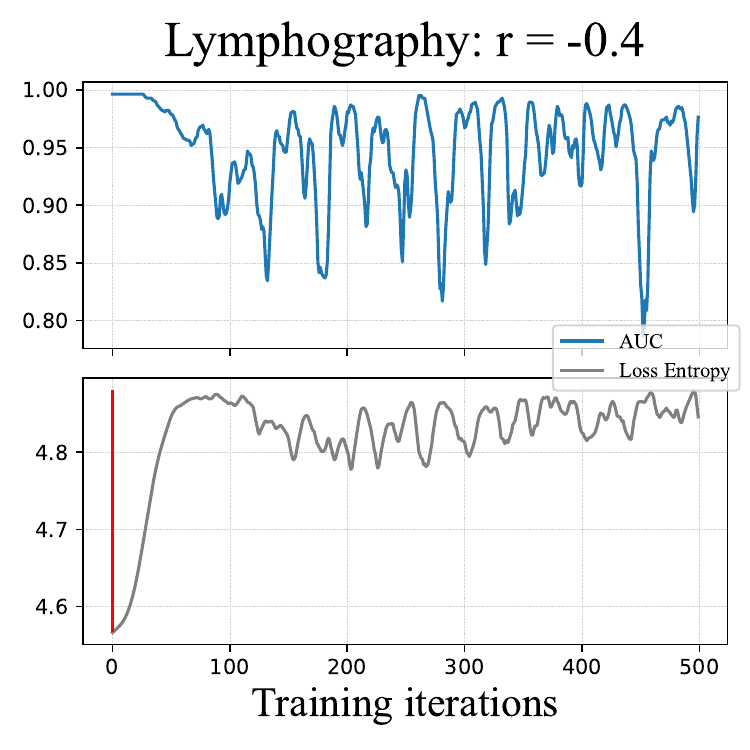}
\includegraphics[width=0.32\textwidth]{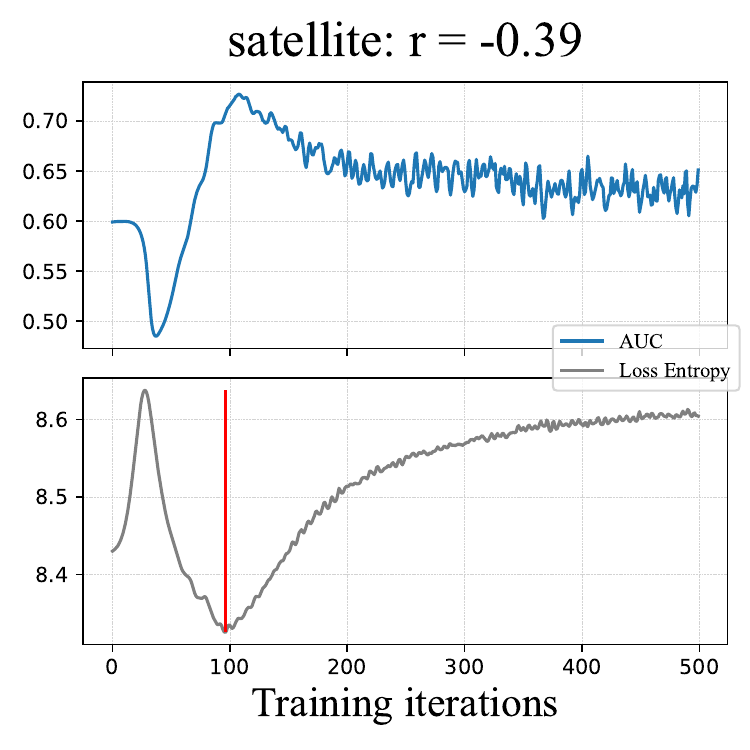}
\includegraphics[width=0.32\textwidth]{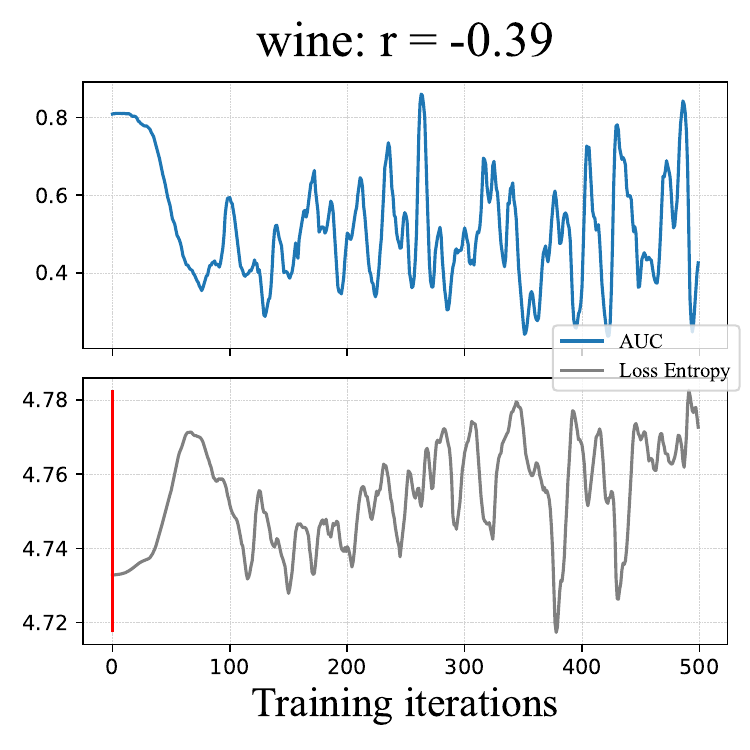}
\includegraphics[width=0.32\textwidth]{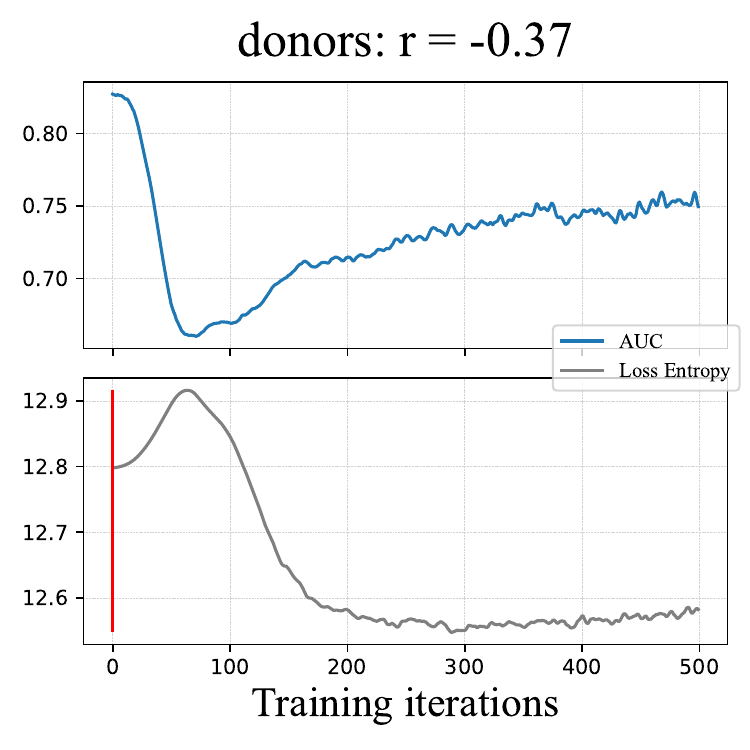}
  \caption{AE: AUC curves vs  $H_L$ curves. The red vertical line is the epoch selected by $EntropyStop$. $r$ denotes the Pearson correlation coefficient between AUC and $H_L$.}
  \label{Fig:all-curve-2}
\end{figure*}

\begin{figure*}
  \centering
  \includegraphics[width=0.32\textwidth]{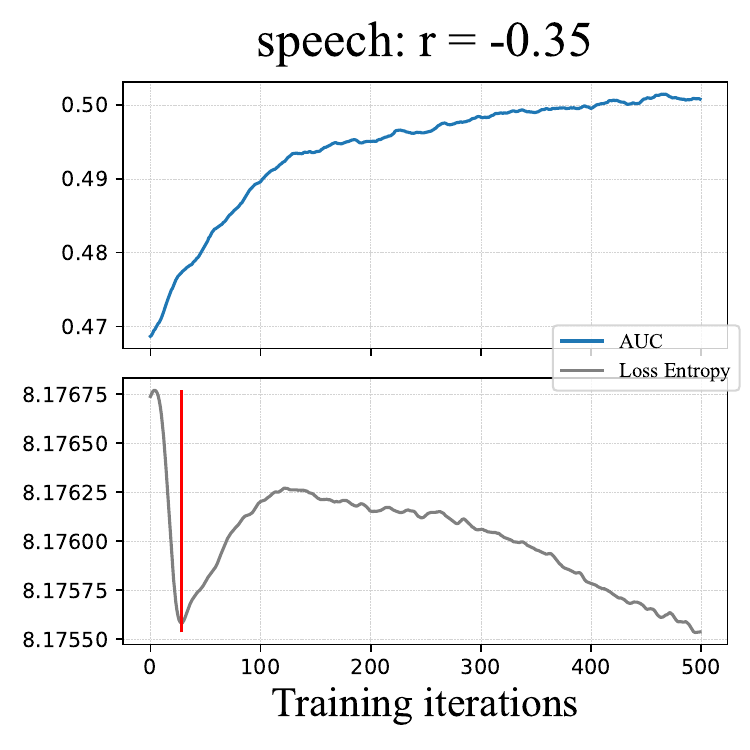}
\includegraphics[width=0.32\textwidth]{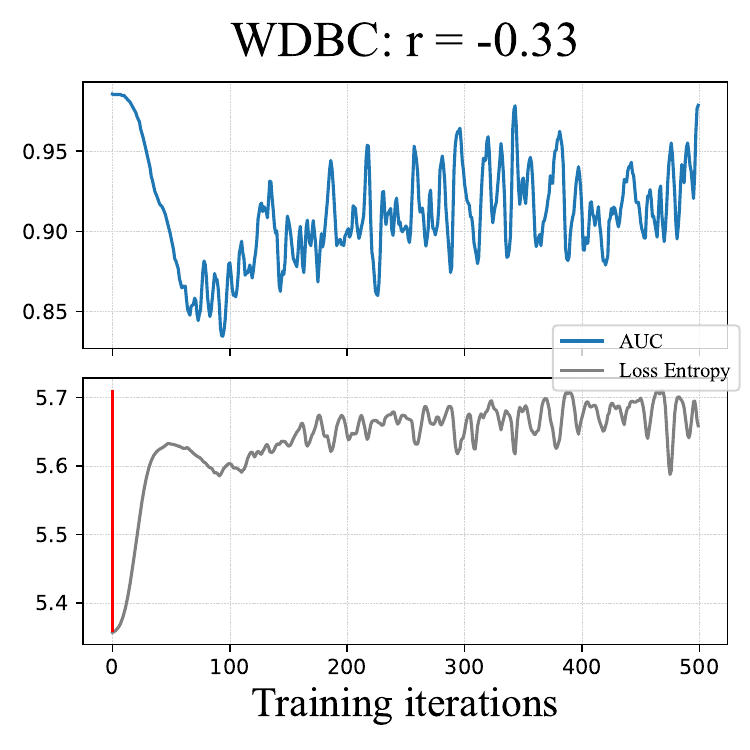}
\includegraphics[width=0.32\textwidth]{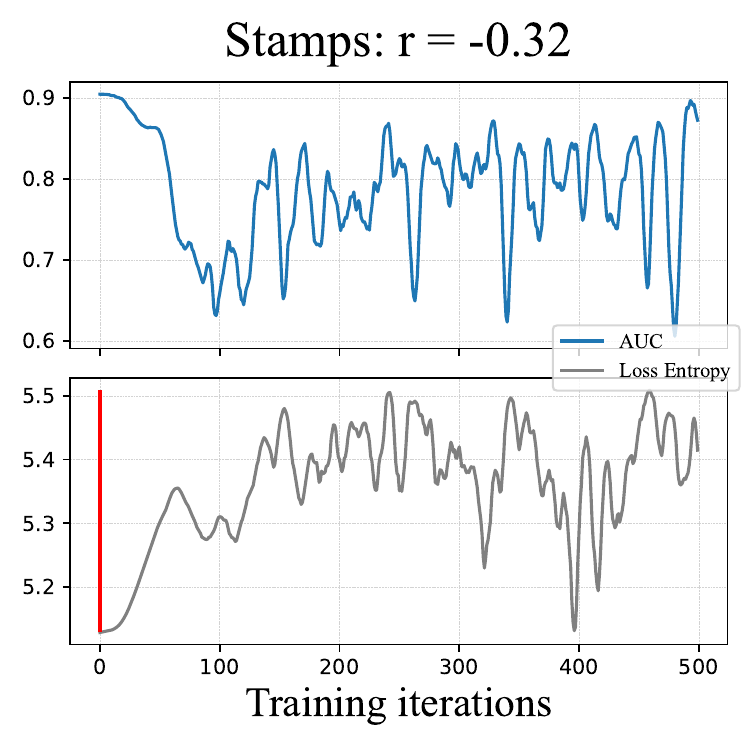}
\includegraphics[width=0.32\textwidth]{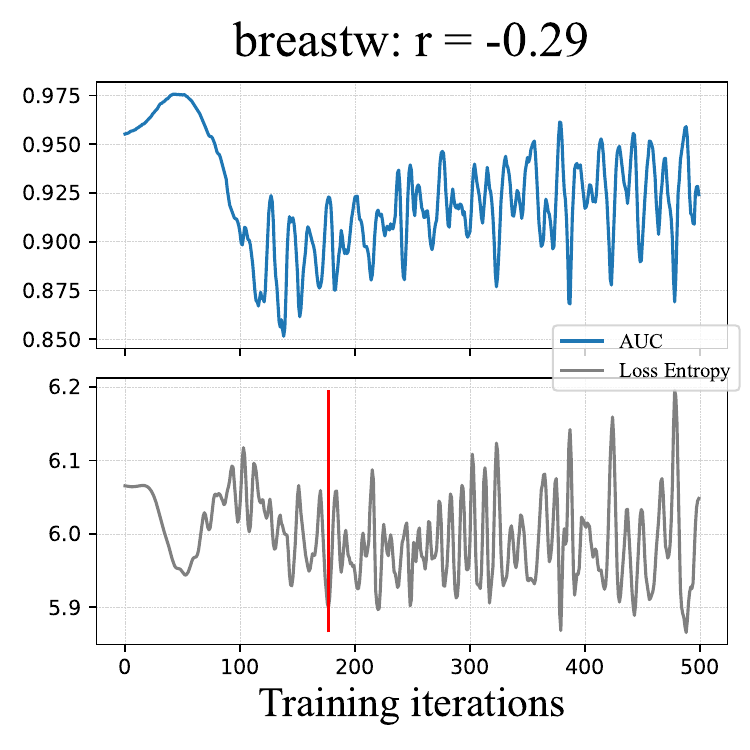}
\includegraphics[width=0.32\textwidth]{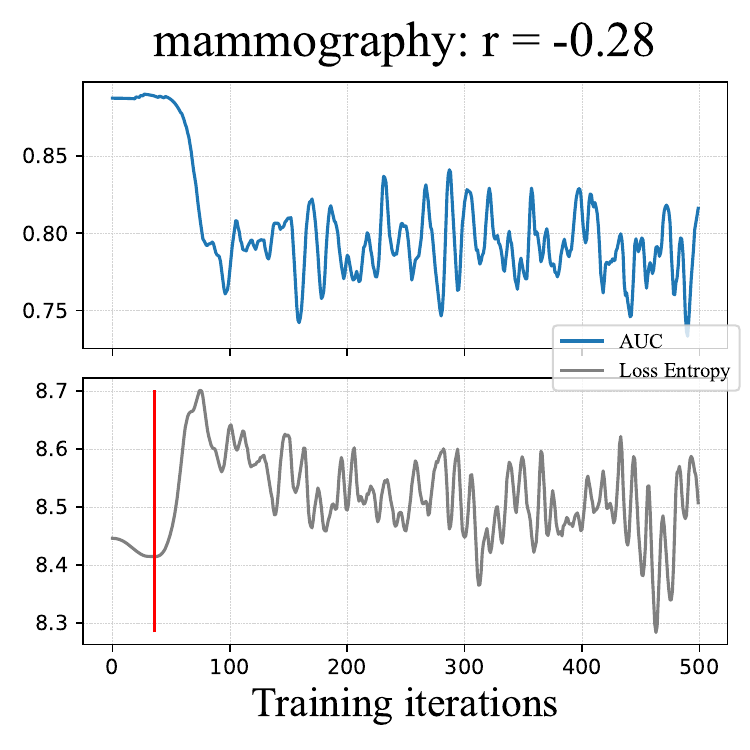}
\includegraphics[width=0.32\textwidth]{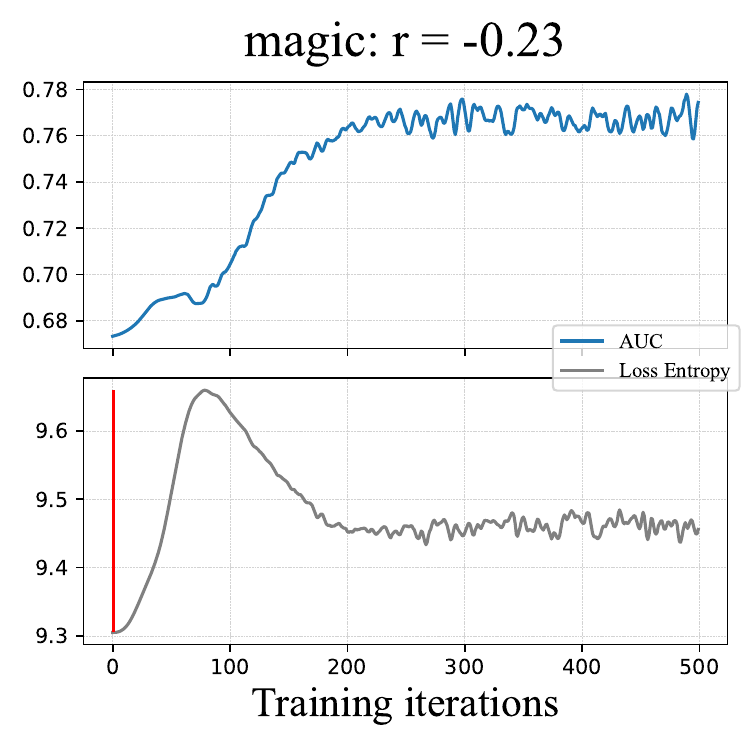}
\includegraphics[width=0.32\textwidth]{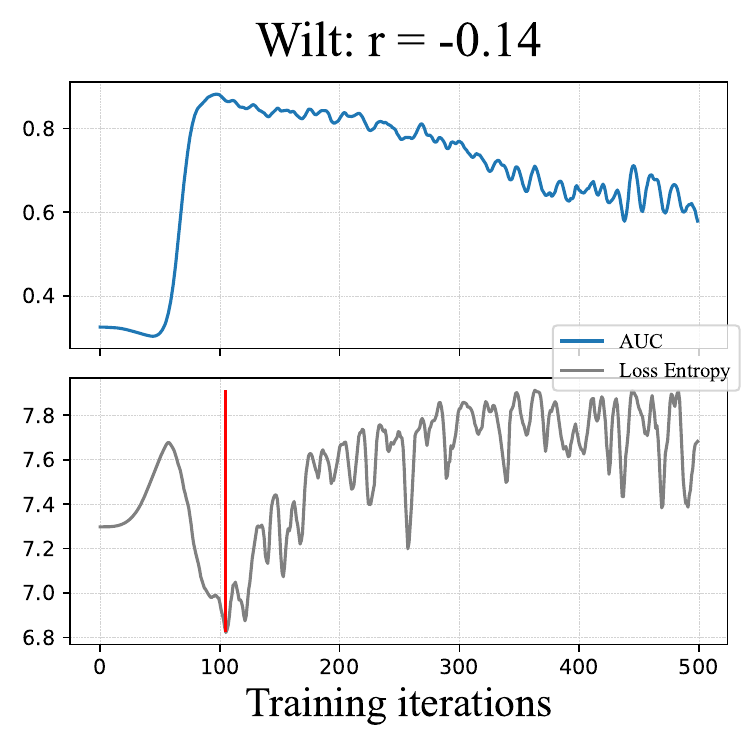}
\includegraphics[width=0.32\textwidth]{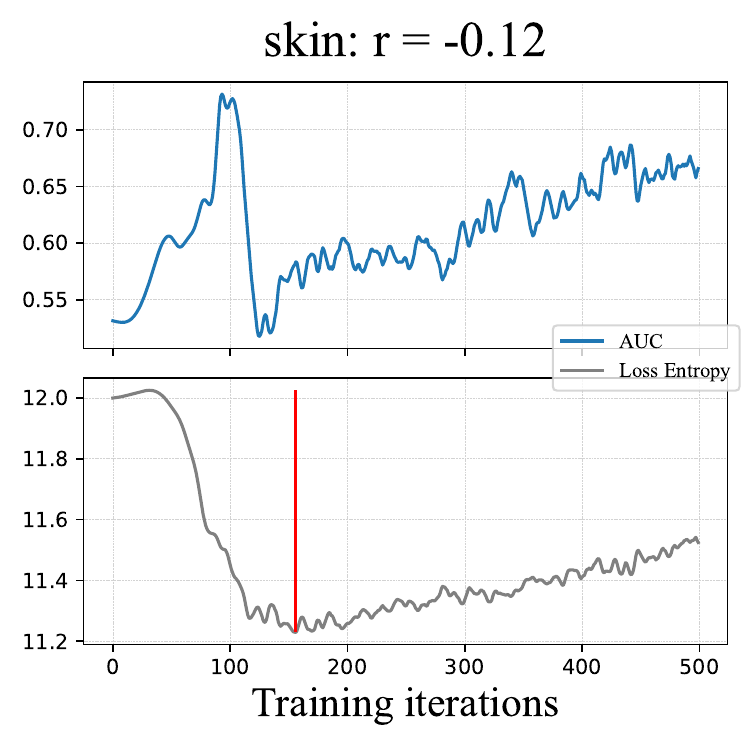}
\includegraphics[width=0.32\textwidth]{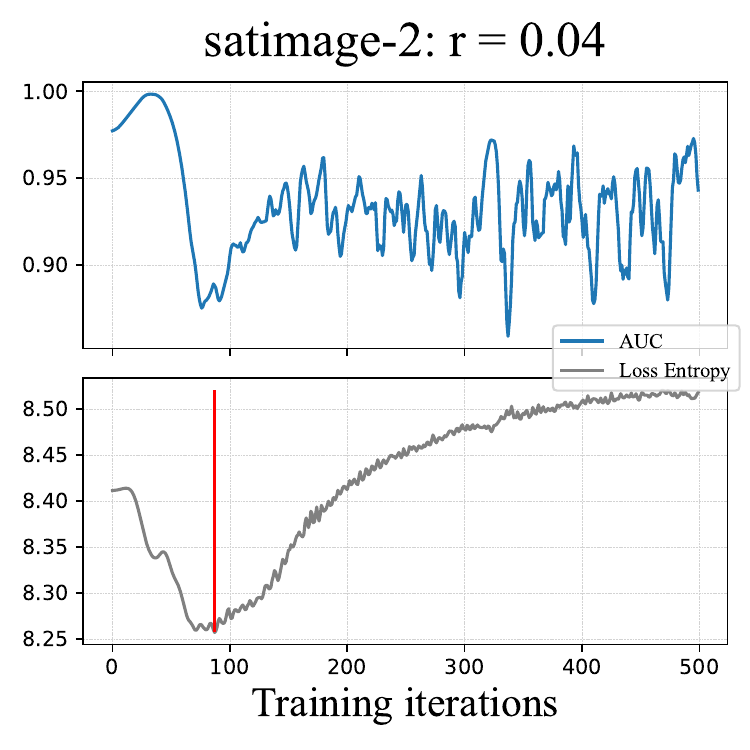}
\includegraphics[width=0.32\textwidth]{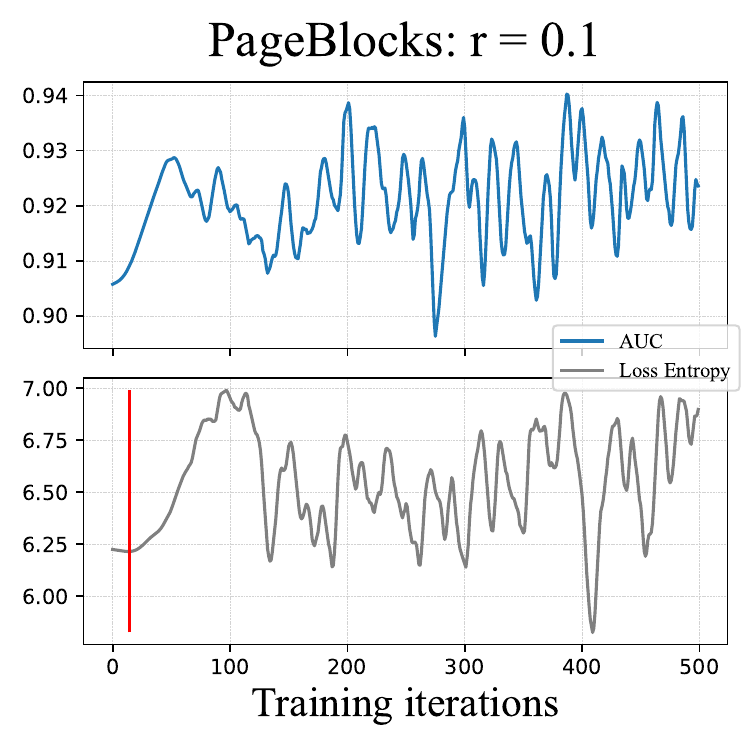}
\includegraphics[width=0.32\textwidth]{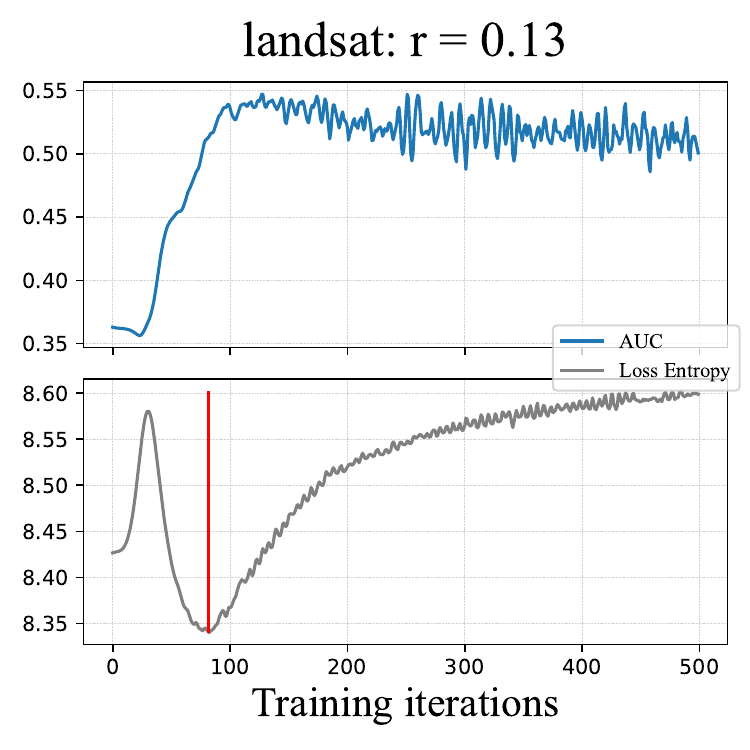}
\includegraphics[width=0.32\textwidth]{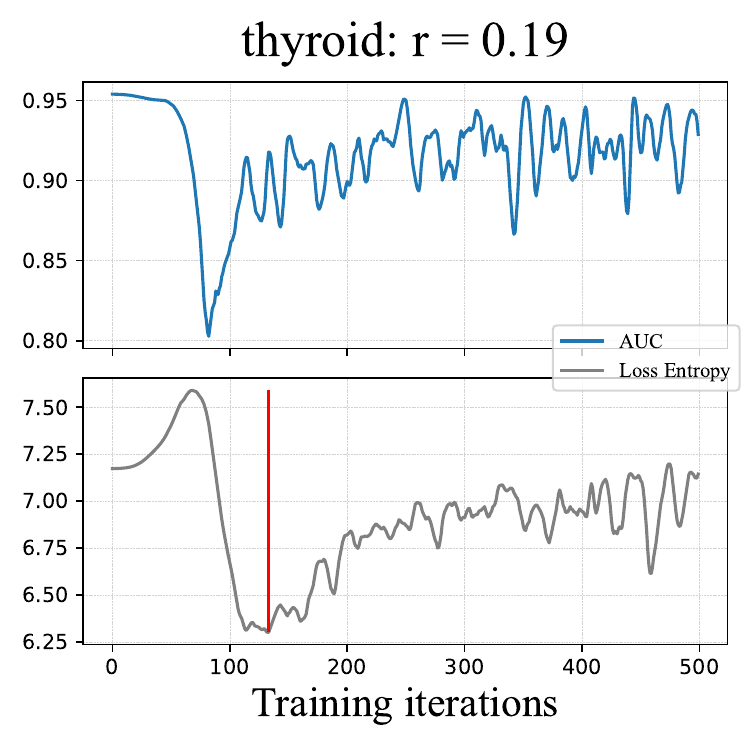}
  \caption{AE:  AUC curves vs  $H_L$ curves. The red vertical line is the epoch selected by $EntropyStop$. $r$ denotes the Pearson correlation coefficient between AUC and $H_L$.}
  \label{Fig:all-curve-3}
\end{figure*}

\begin{figure*}
  \centering
\includegraphics[width=0.32\textwidth]{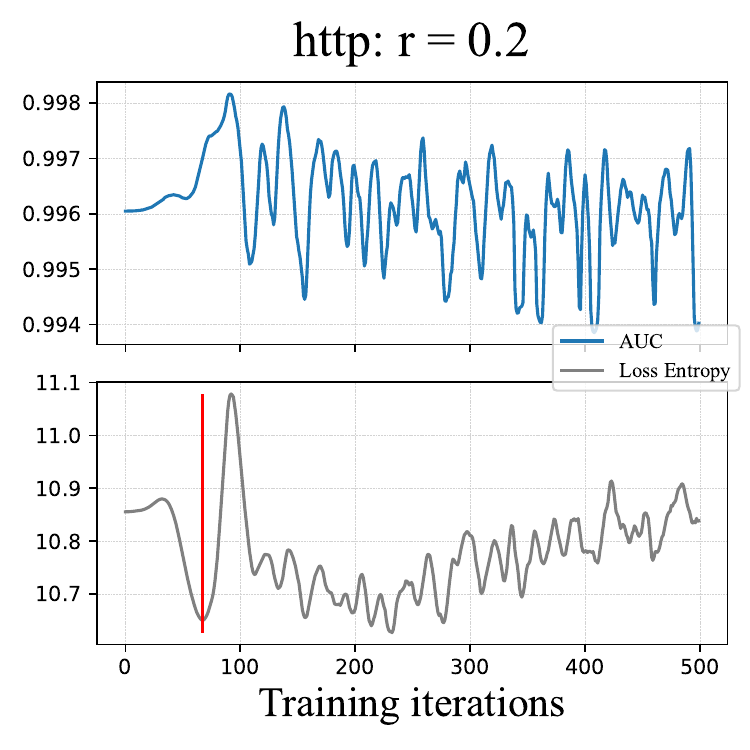}
\includegraphics[width=0.32\textwidth]{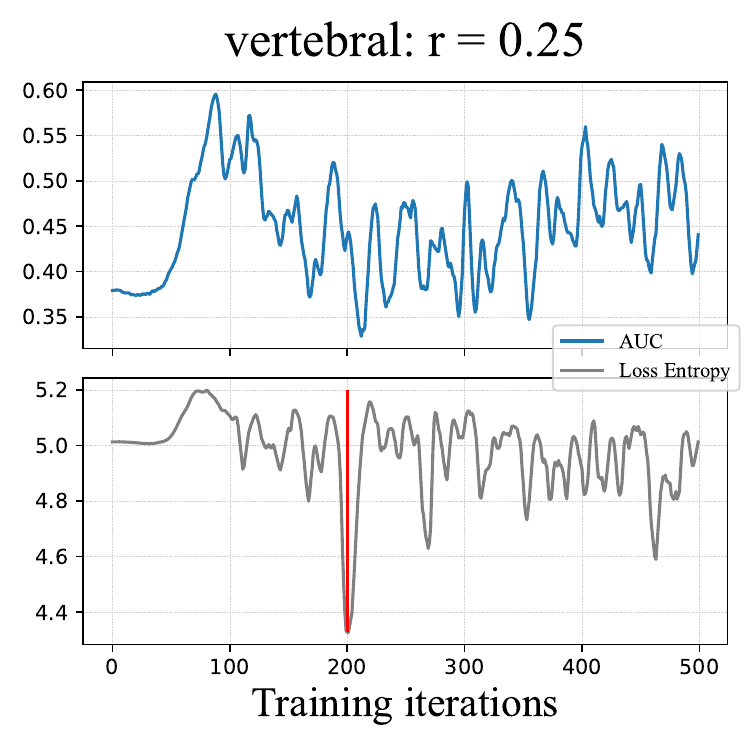}
\includegraphics[width=0.32\textwidth]{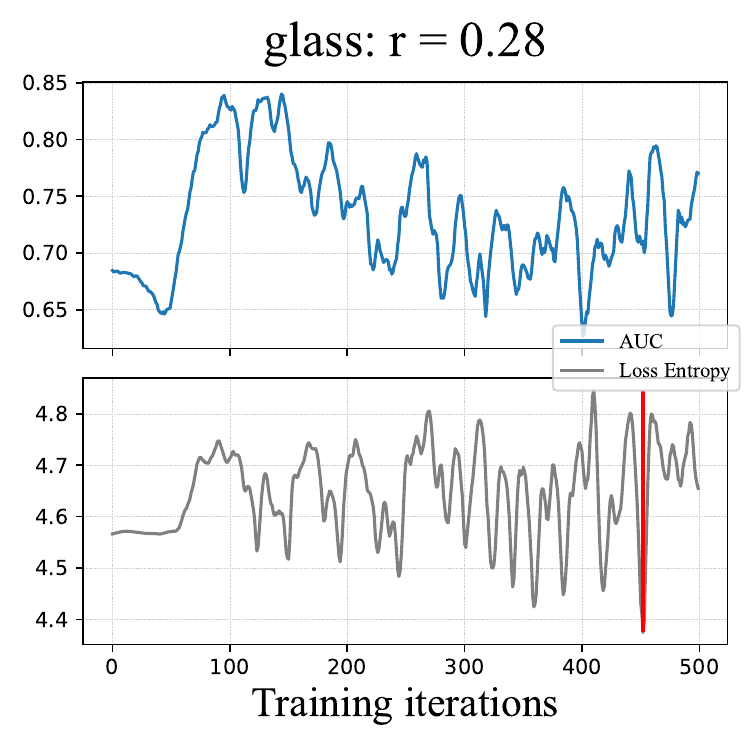}
\includegraphics[width=0.32\textwidth]{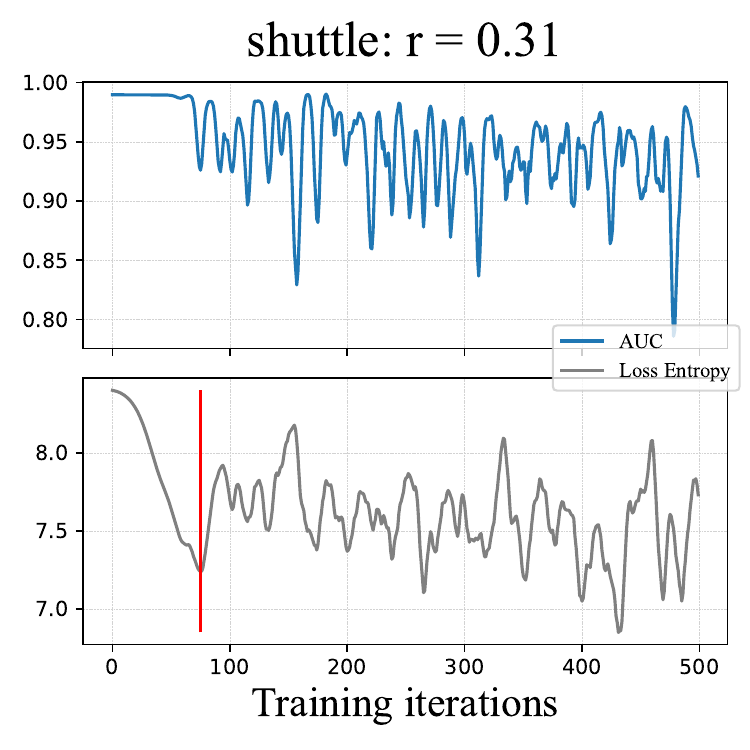}
\includegraphics[width=0.32\textwidth]{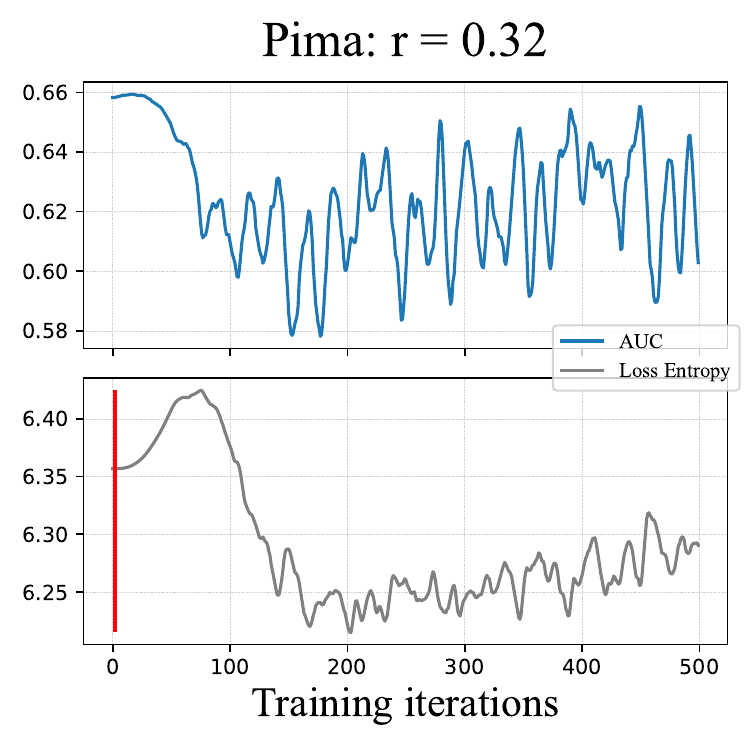}
\includegraphics[width=0.32\textwidth]{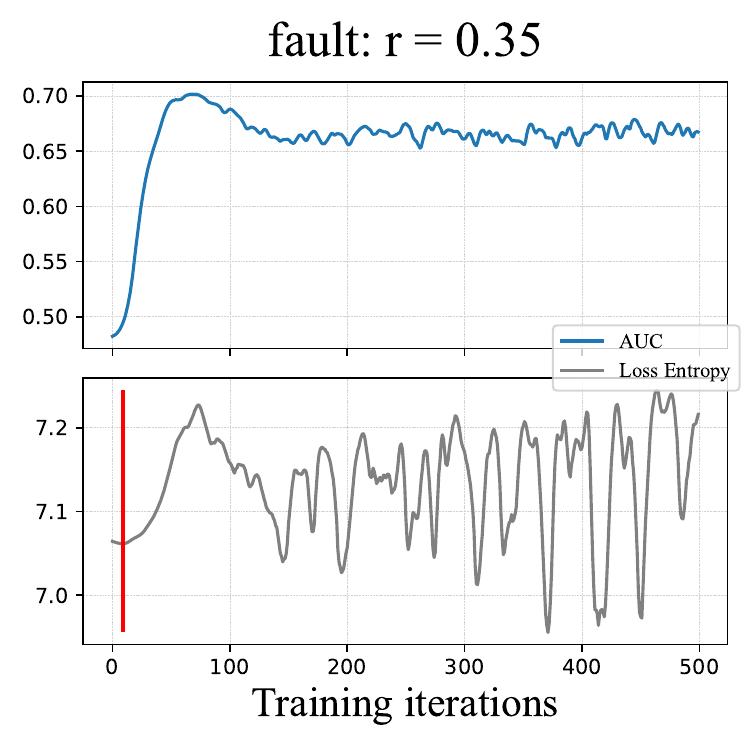}
\includegraphics[width=0.32\textwidth]{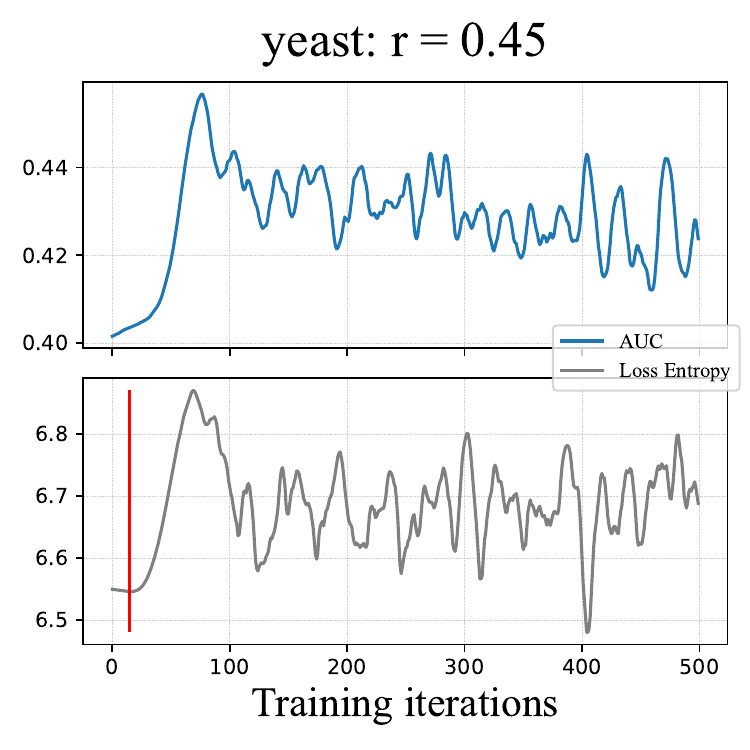}
\includegraphics[width=0.32\textwidth]{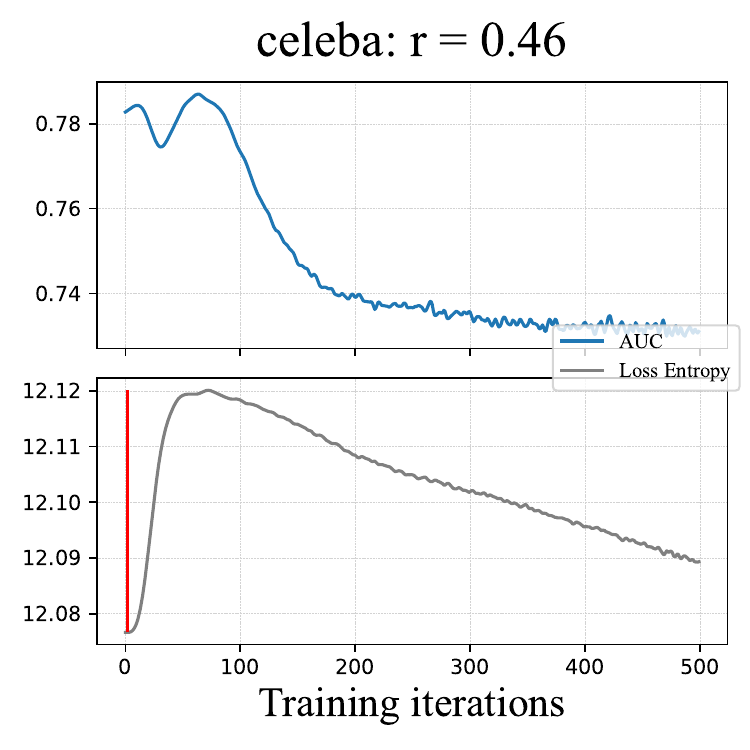}
\includegraphics[width=0.32\textwidth]{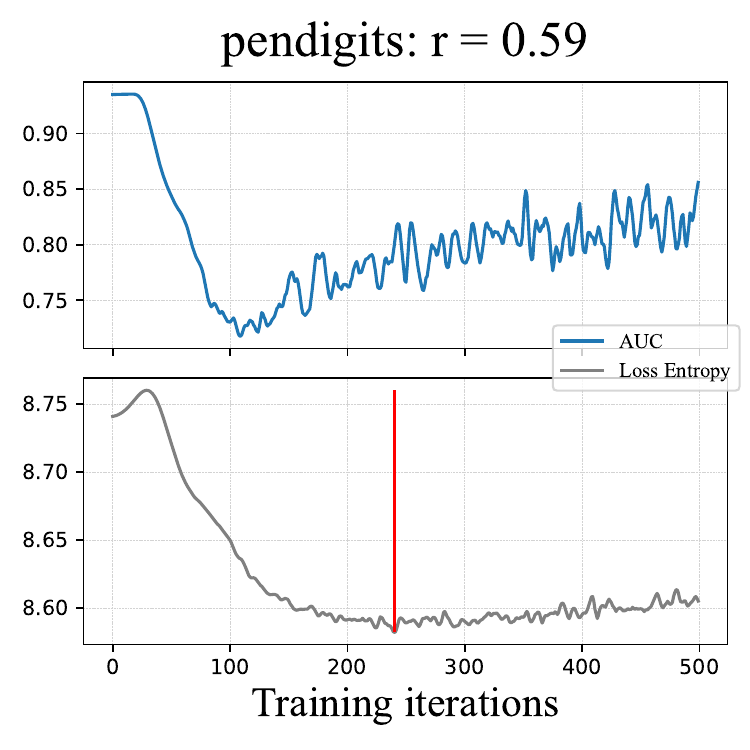}
\includegraphics[width=0.32\textwidth]{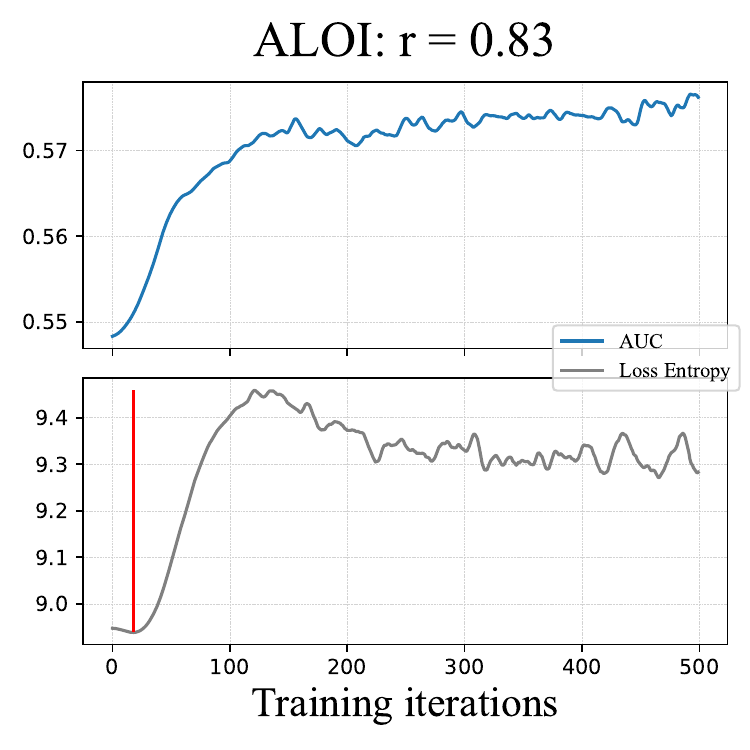}
\includegraphics[width=0.32\textwidth]{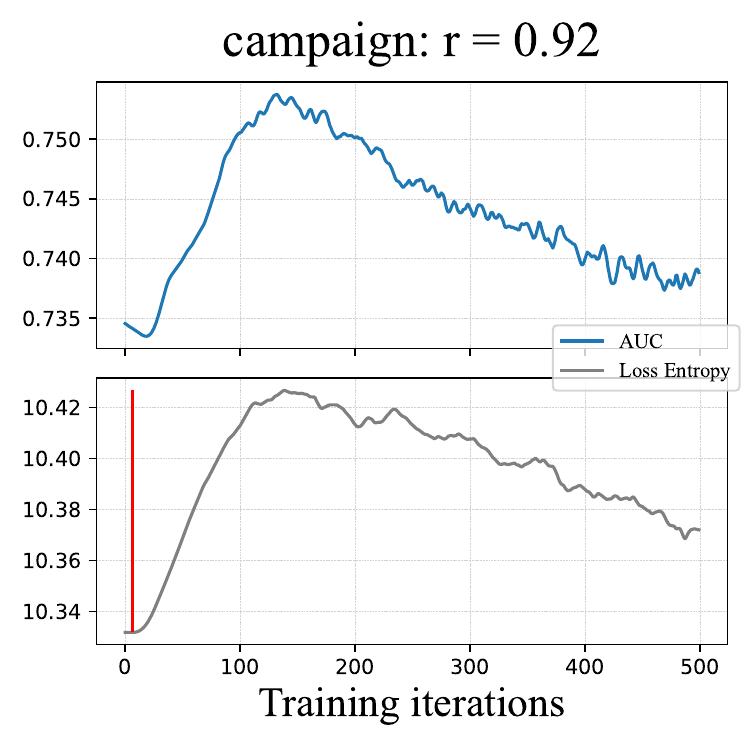}
  \caption{AE: AUC curves vs  $H_L$ curves. The red vertical line is the epoch selected by $EntropyStop$. $r$ denotes the Pearson correlation coefficient between AUC and $H_L$.}
  \label{Fig:all-curve-4}
\end{figure*}

\begin{figure*}
  \centering
\includegraphics[width=0.32\textwidth]{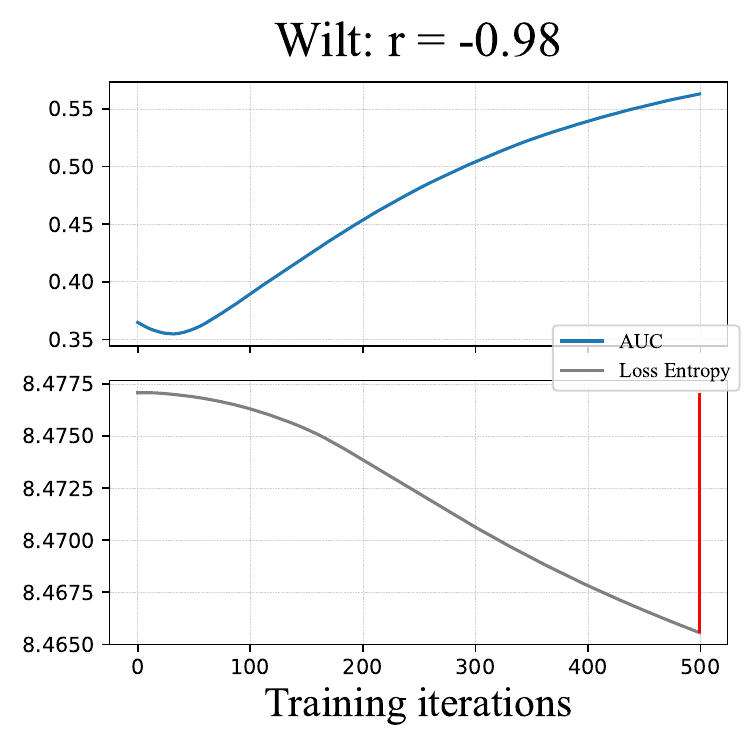}
\includegraphics[width=0.32\textwidth]{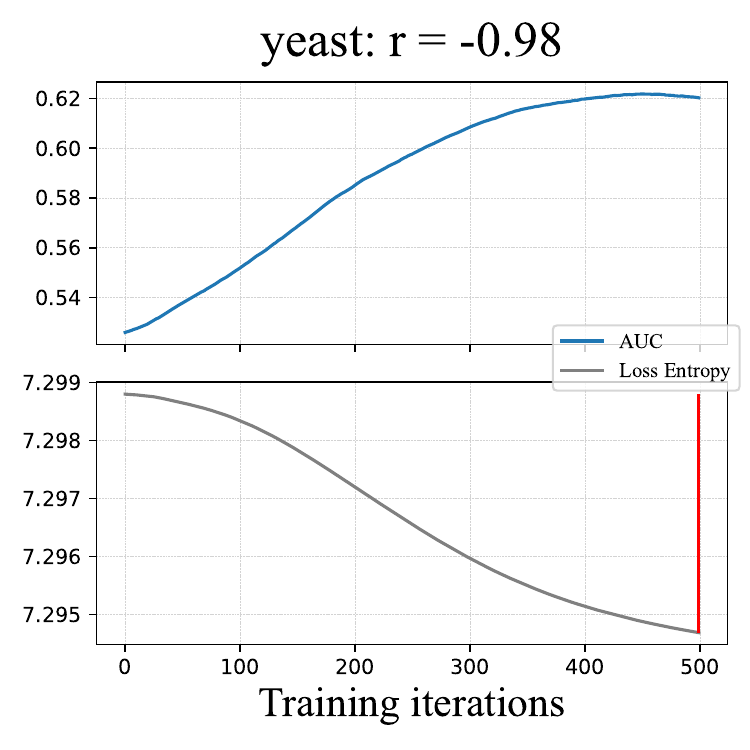}
\includegraphics[width=0.32\textwidth]{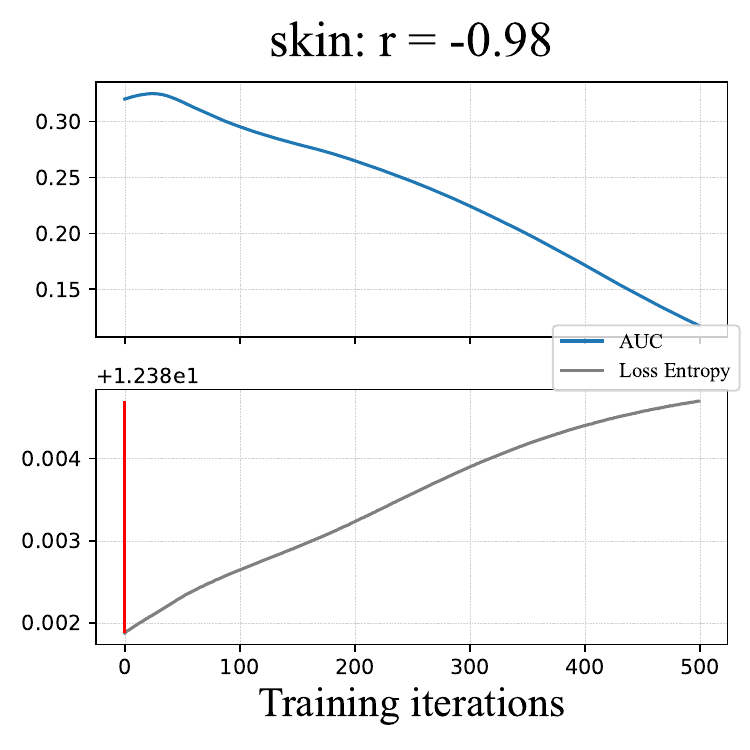}
\includegraphics[width=0.32\textwidth]{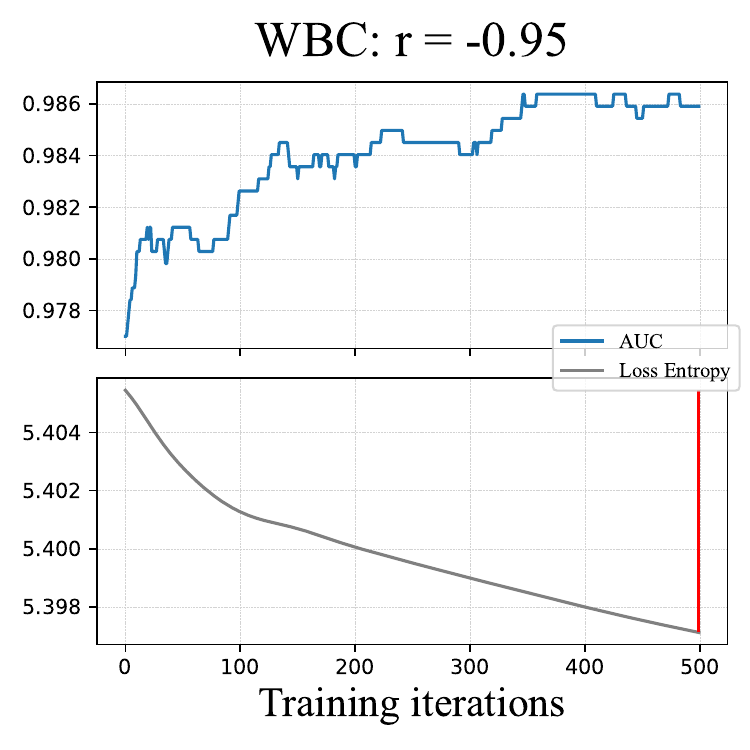}
\includegraphics[width=0.32\textwidth]{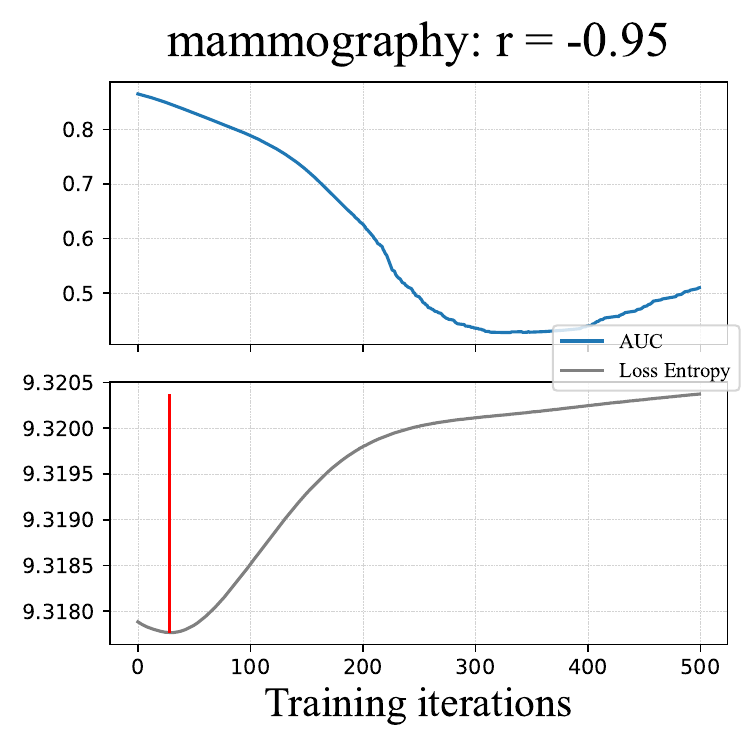}
\includegraphics[width=0.32\textwidth]{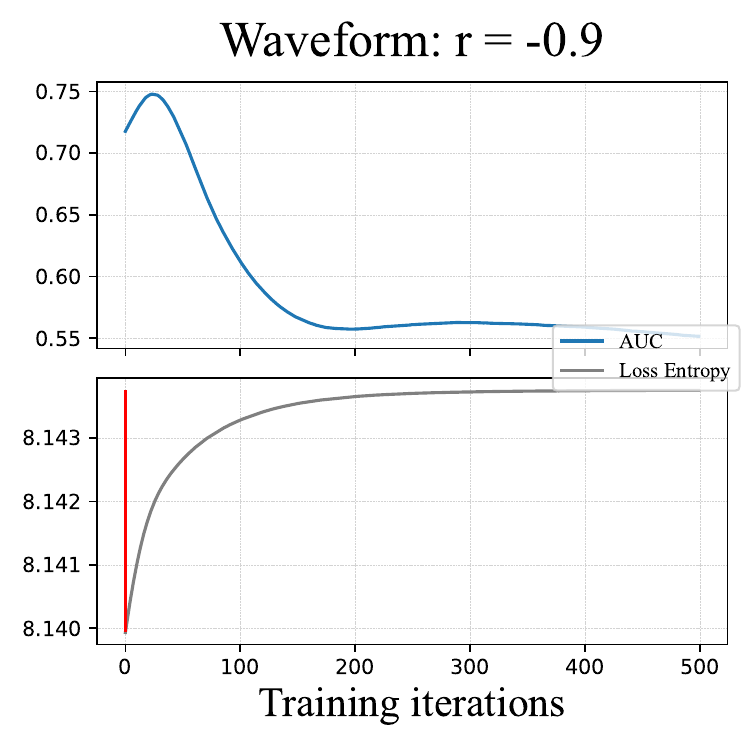}
\includegraphics[width=0.32\textwidth]{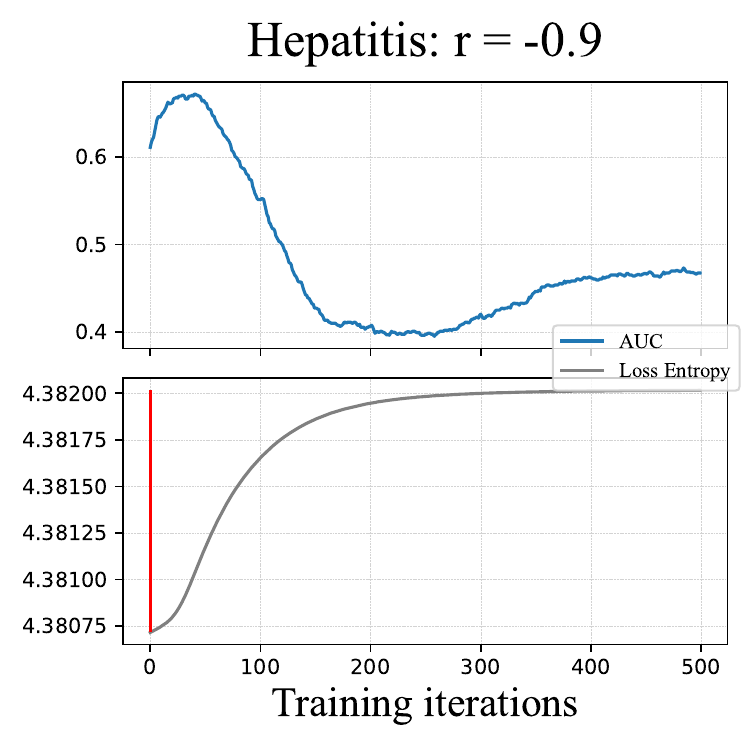}
\includegraphics[width=0.32\textwidth]{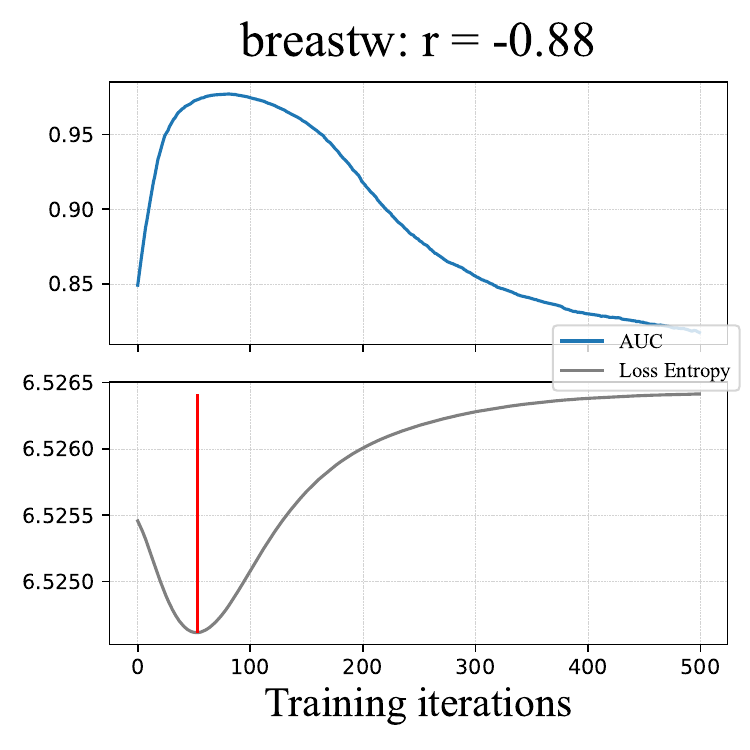}
\includegraphics[width=0.32\textwidth]{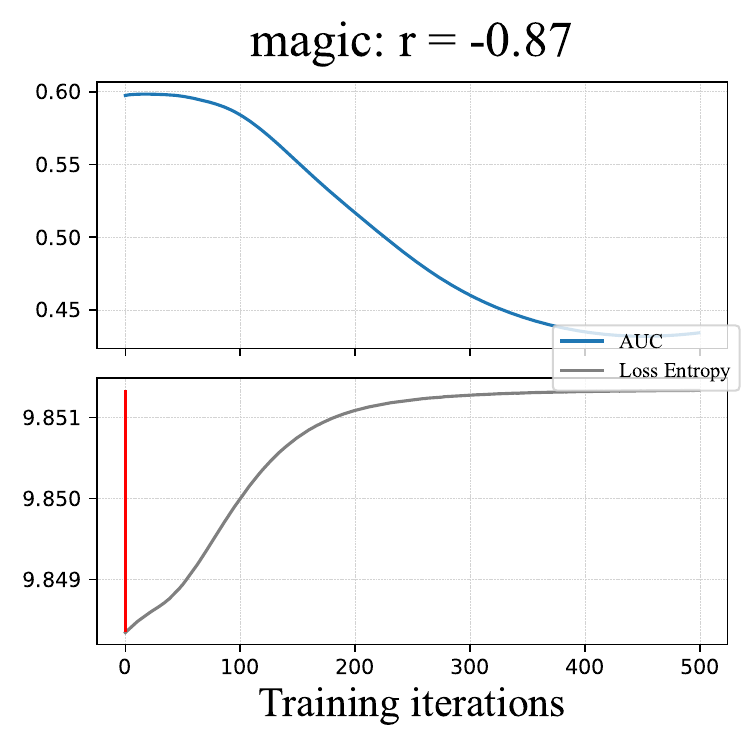}
\includegraphics[width=0.32\textwidth]{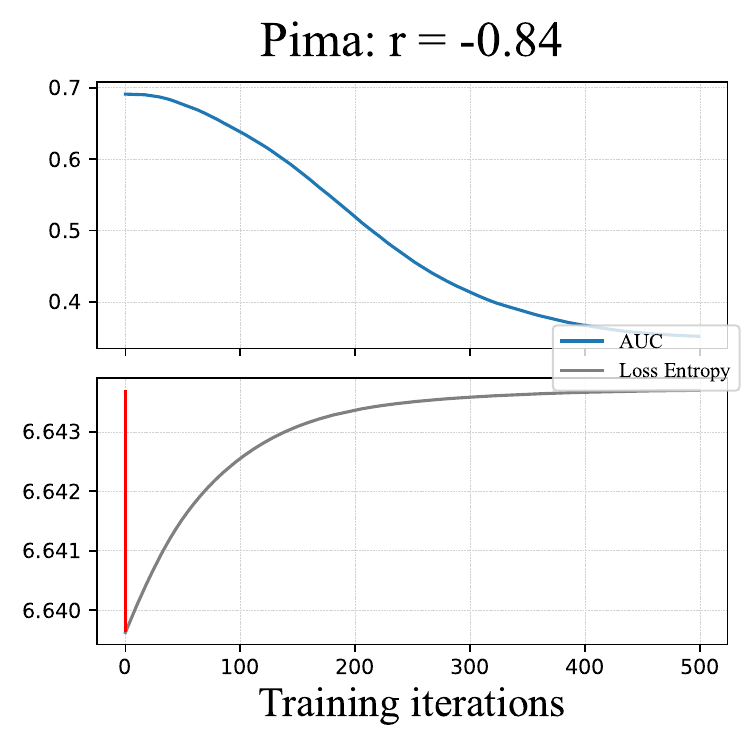}
\includegraphics[width=0.32\textwidth]{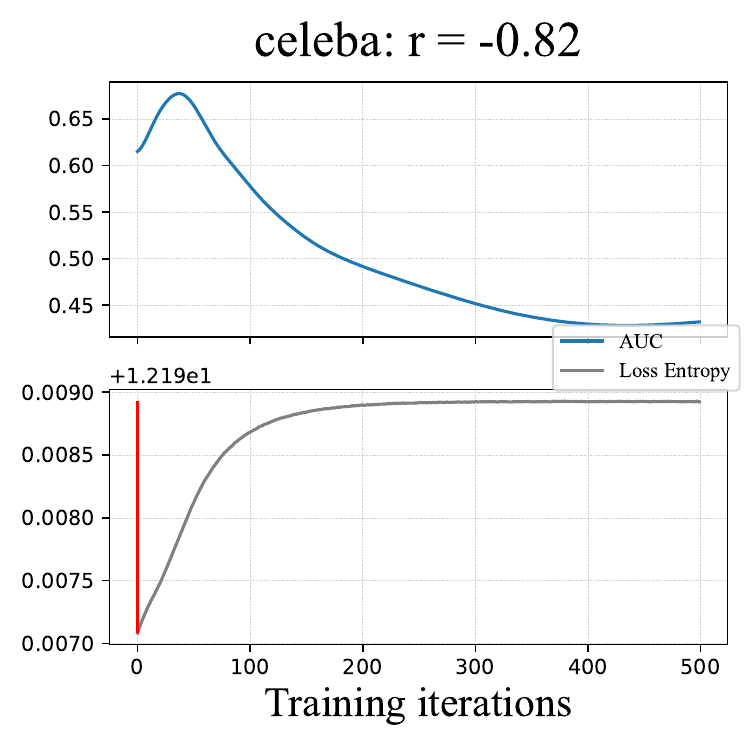}
\includegraphics[width=0.32\textwidth]{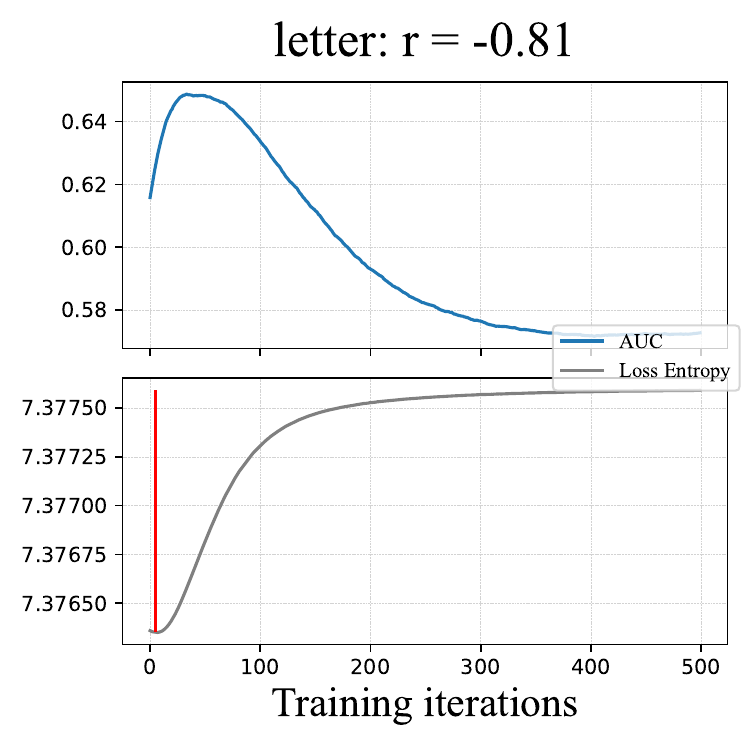}
  \caption{DeepSVDD: AUC curves vs  $H_L$ curves. The red vertical line is the epoch selected by $EntropyStop$. $r$ denotes the Pearson correlation coefficient between AUC and $H_L$.}
  \label{Fig:svdd-all-curve-1}
\end{figure*}

\begin{figure*}
  \centering
\includegraphics[width=0.32\textwidth]{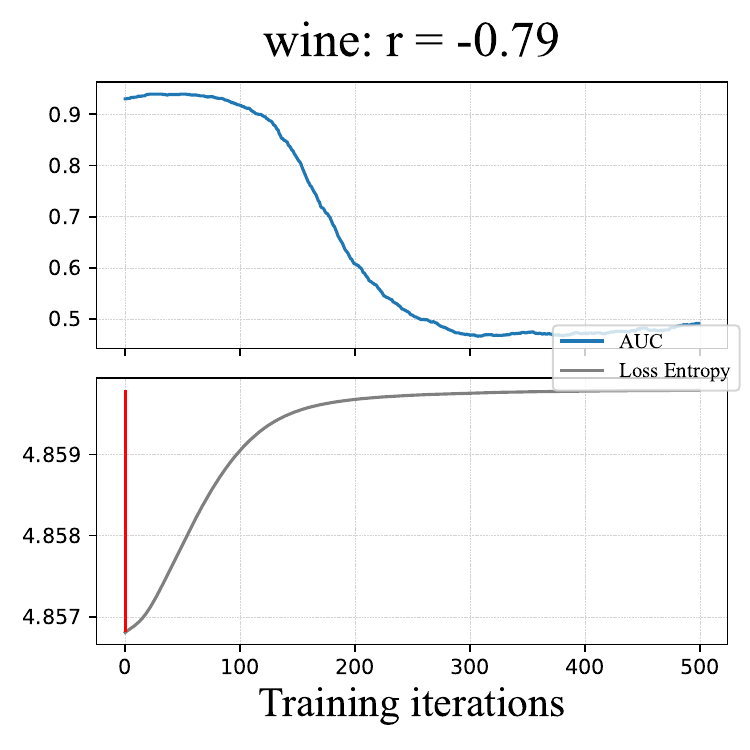}
\includegraphics[width=0.32\textwidth]{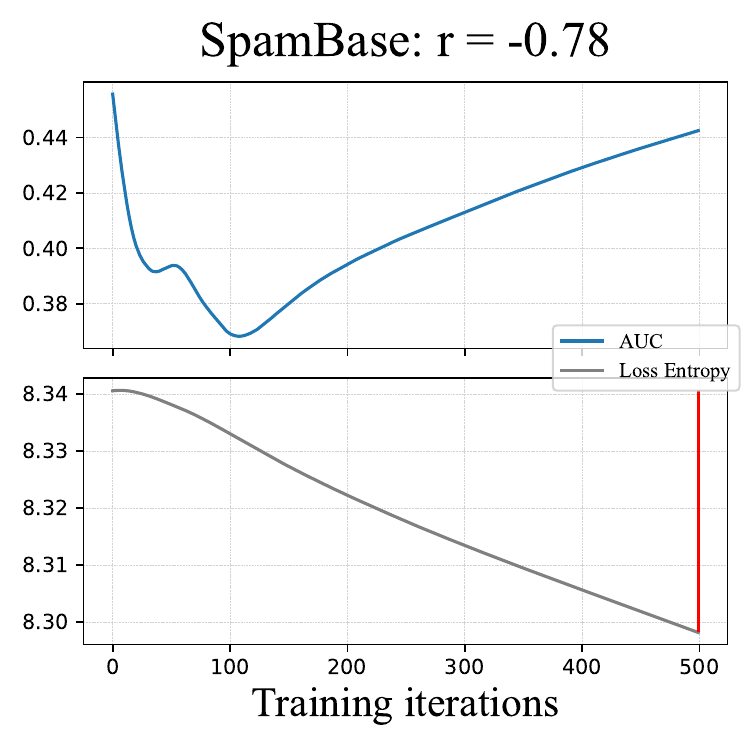}
\includegraphics[width=0.32\textwidth]{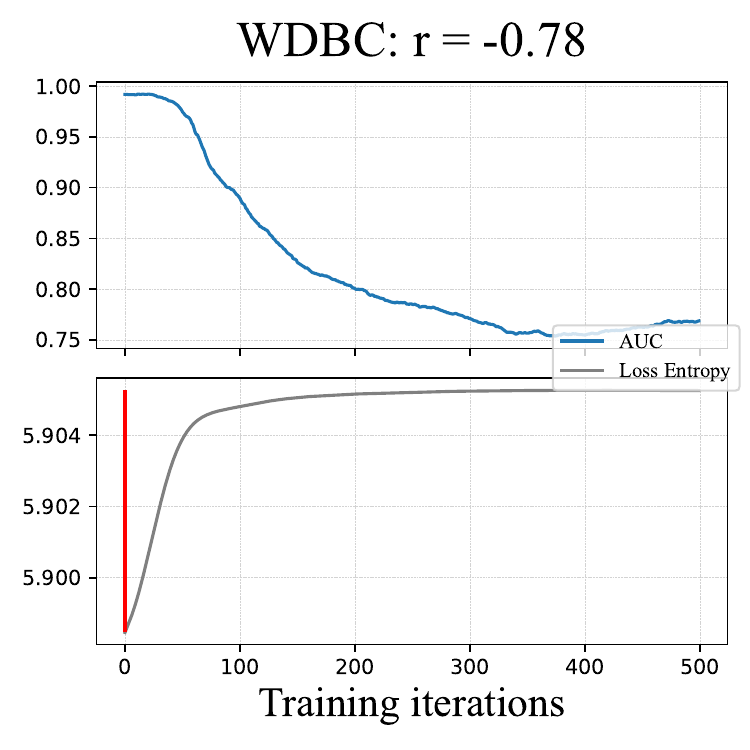}
\includegraphics[width=0.32\textwidth]{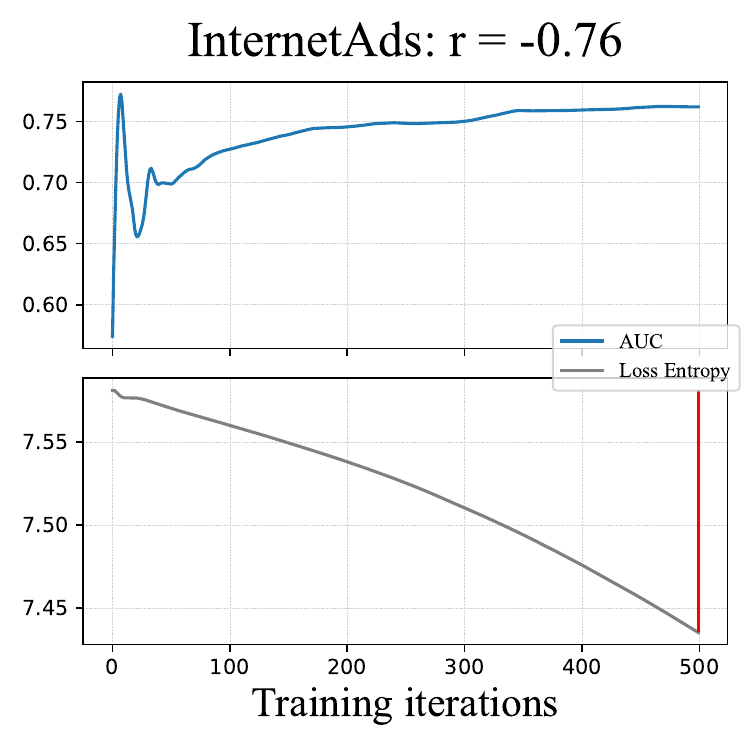}
\includegraphics[width=0.32\textwidth]{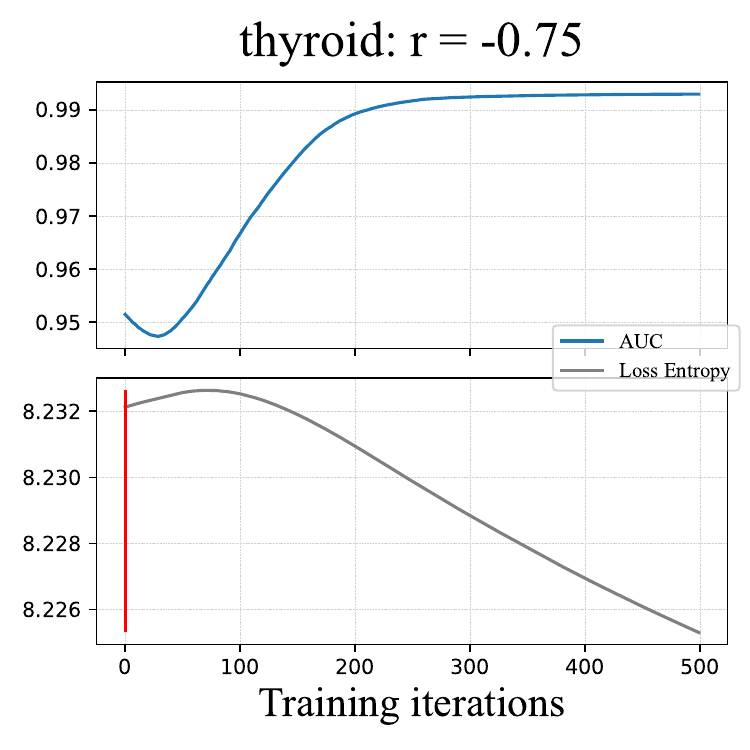}
\includegraphics[width=0.32\textwidth]{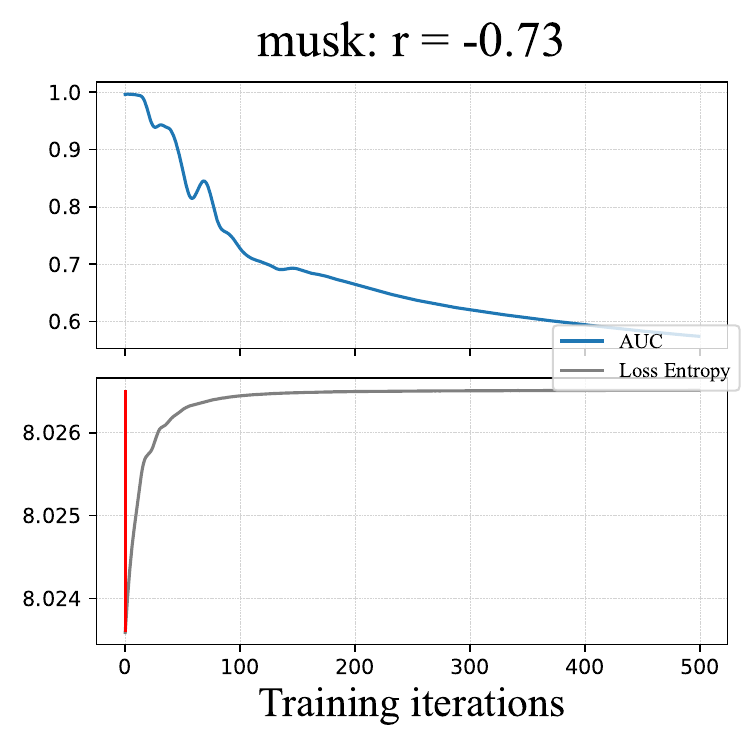}
\includegraphics[width=0.32\textwidth]{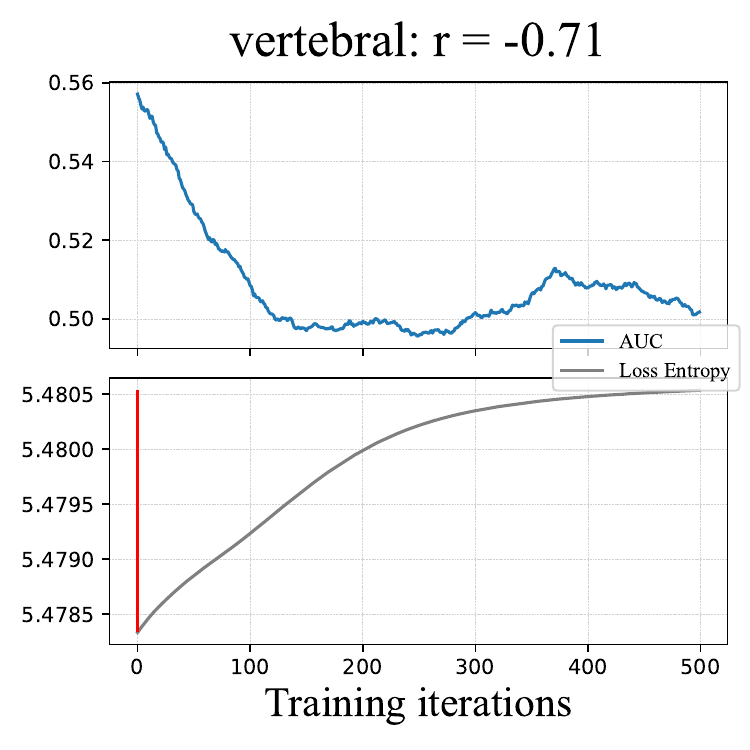}
\includegraphics[width=0.32\textwidth]{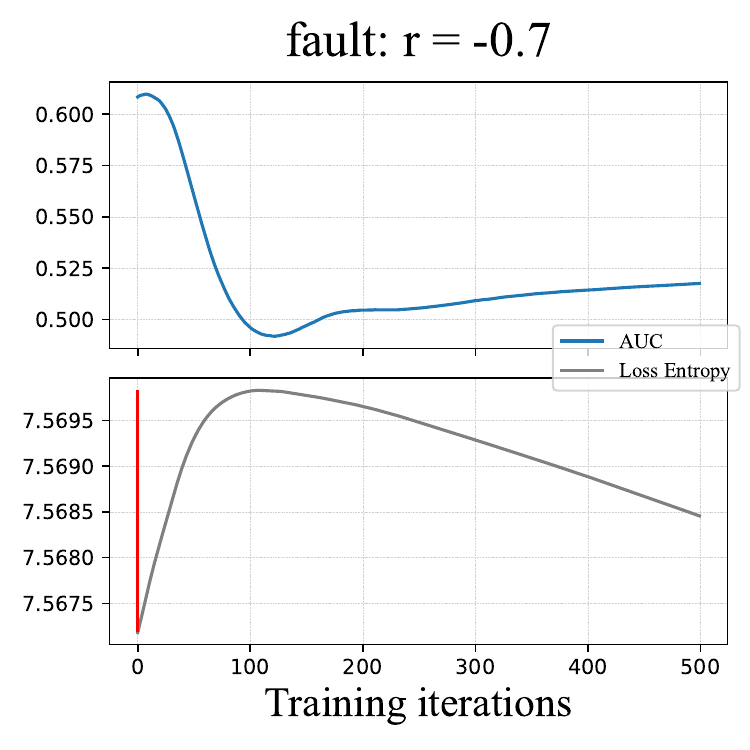}
\includegraphics[width=0.32\textwidth]{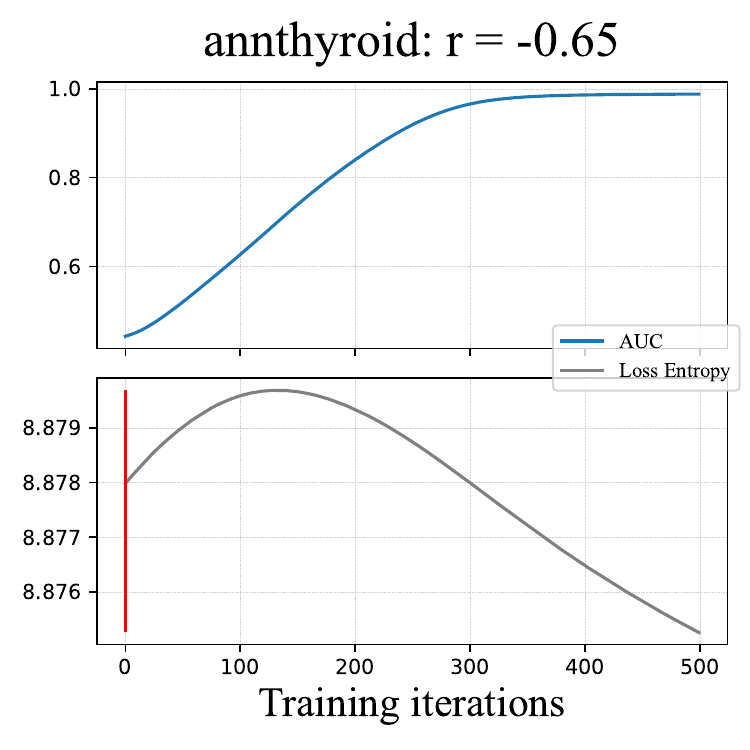}
\includegraphics[width=0.32\textwidth]{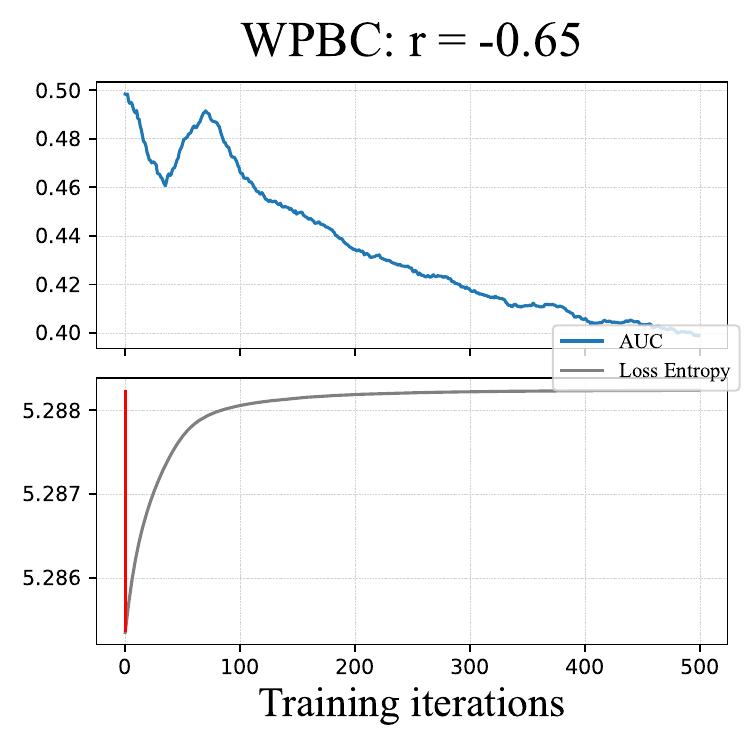}
\includegraphics[width=0.32\textwidth]{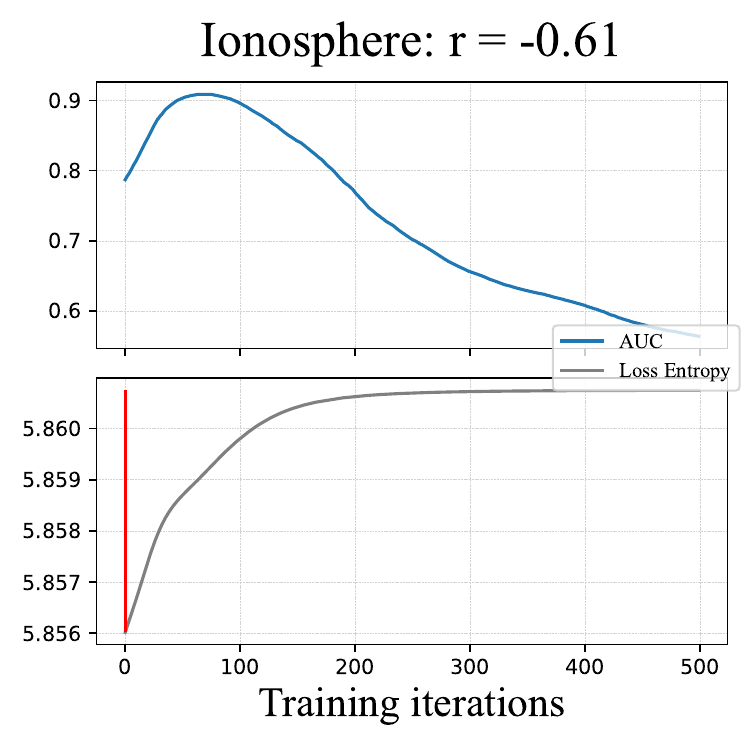}
\includegraphics[width=0.32\textwidth]{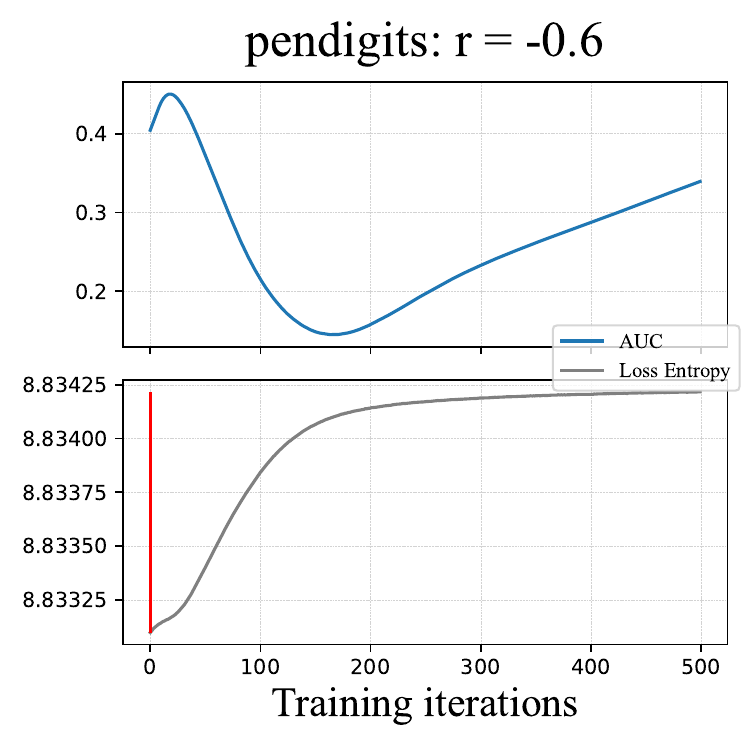}
  \caption{DeepSVDD: AUC curves vs  $H_L$ curves. The red vertical line is the epoch selected by $EntropyStop$. $r$ denotes the Pearson correlation coefficient between AUC and $H_L$.}
  \label{Fig:svdd-all-curve-2}
\end{figure*}

\begin{figure*}
  \centering
\includegraphics[width=0.32\textwidth]{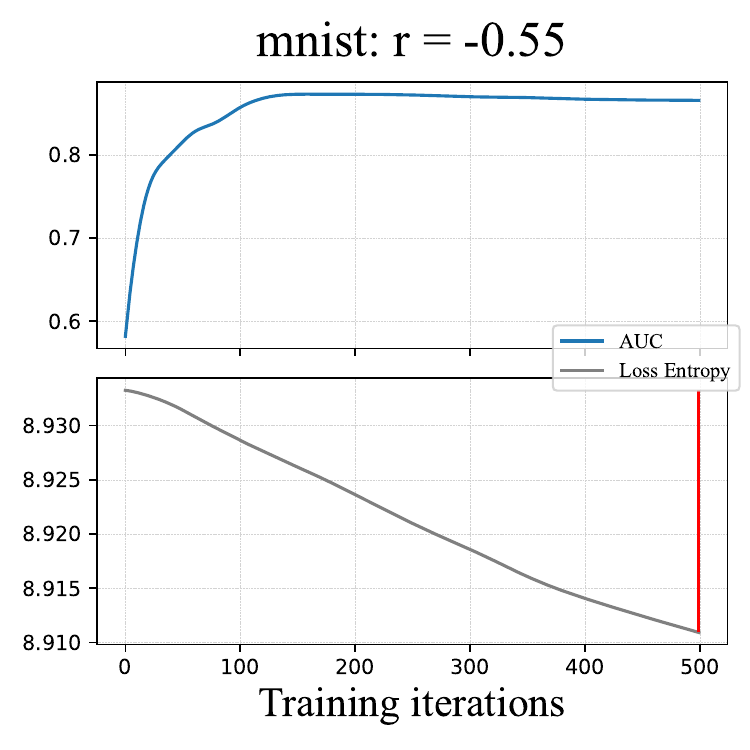}
\includegraphics[width=0.32\textwidth]{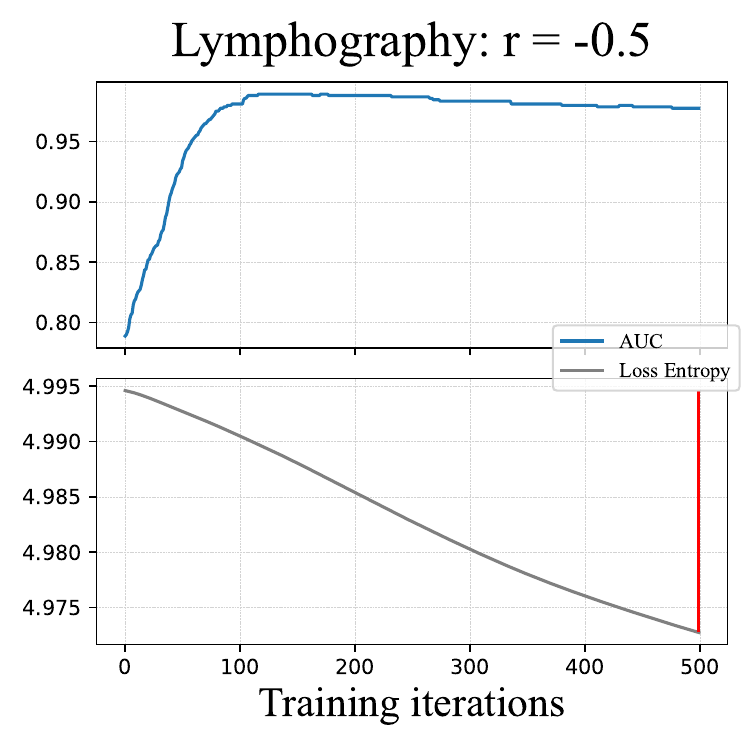}
\includegraphics[width=0.32\textwidth]{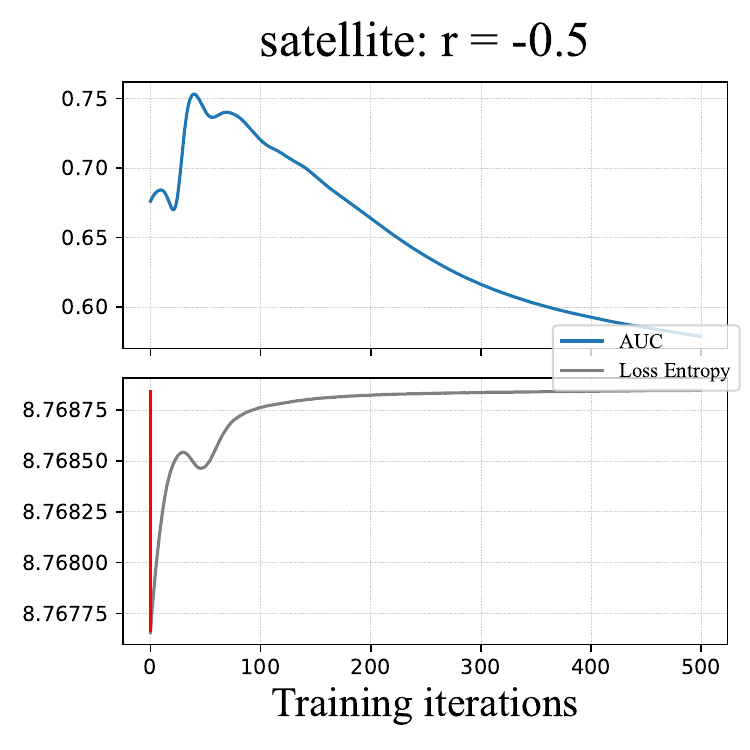}
\includegraphics[width=0.32\textwidth]{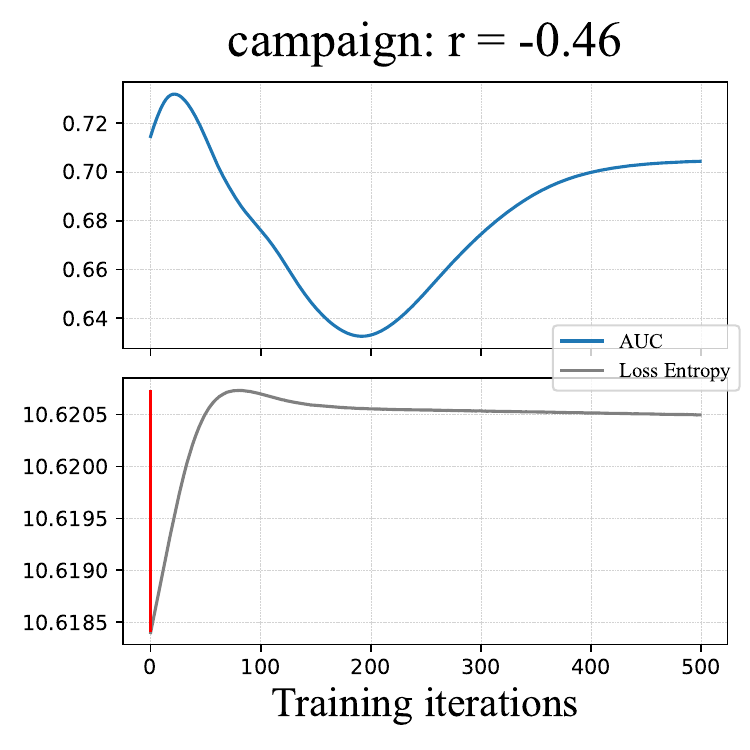}
\includegraphics[width=0.32\textwidth]{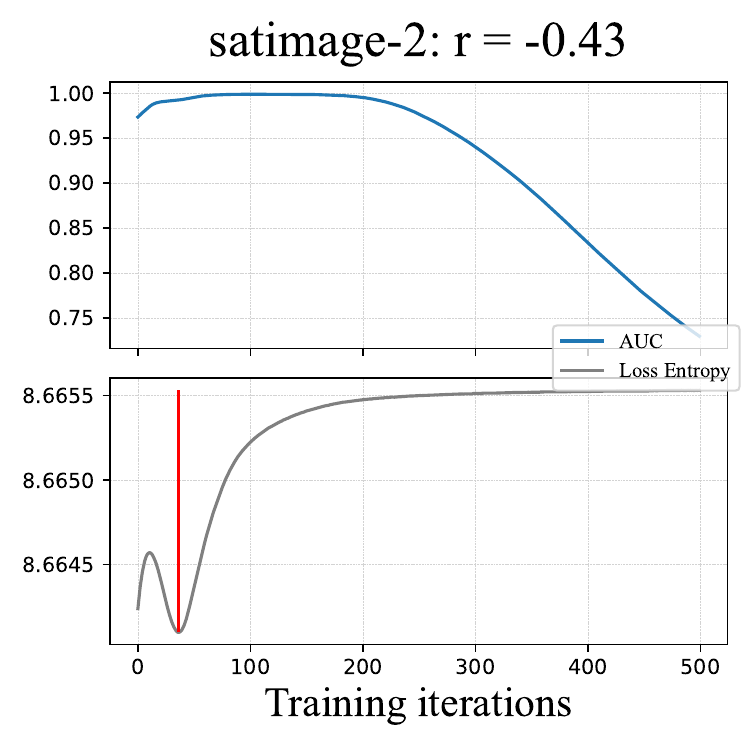}
\includegraphics[width=0.32\textwidth]{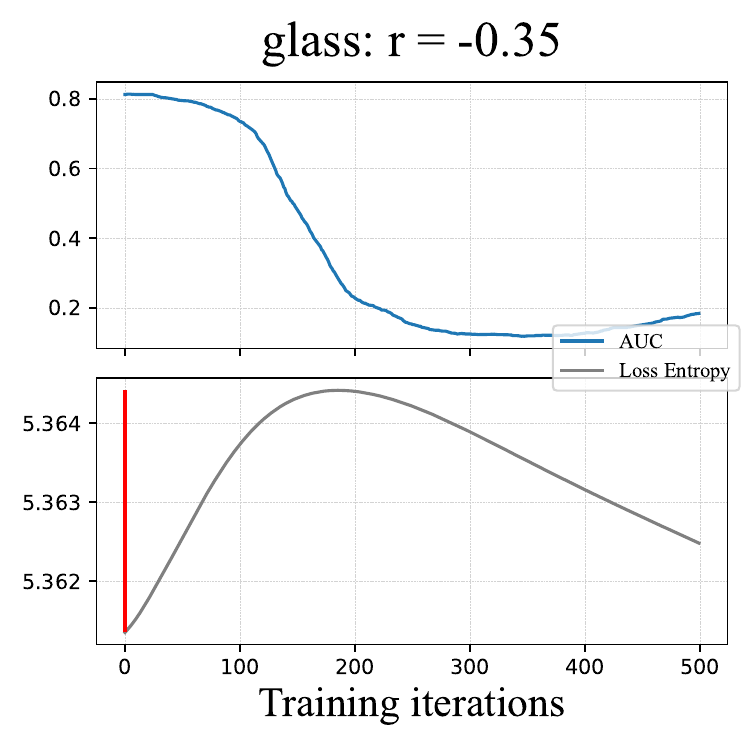}
\includegraphics[width=0.32\textwidth]{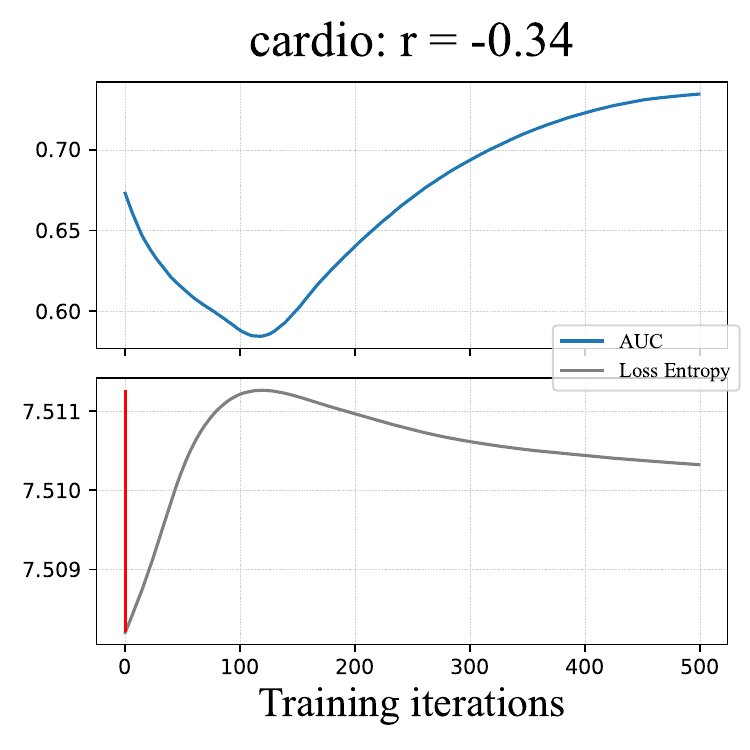}
\includegraphics[width=0.32\textwidth]{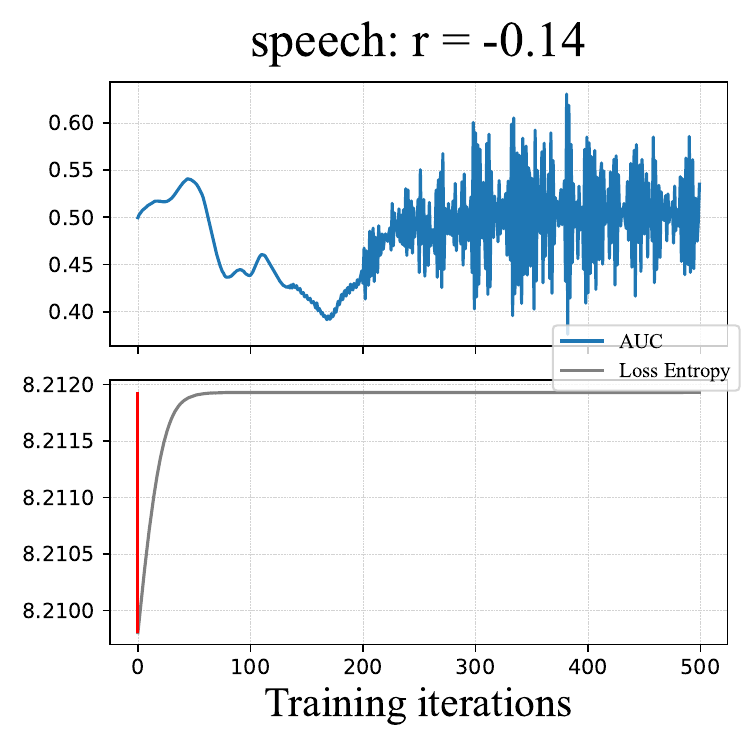}
\includegraphics[width=0.32\textwidth]{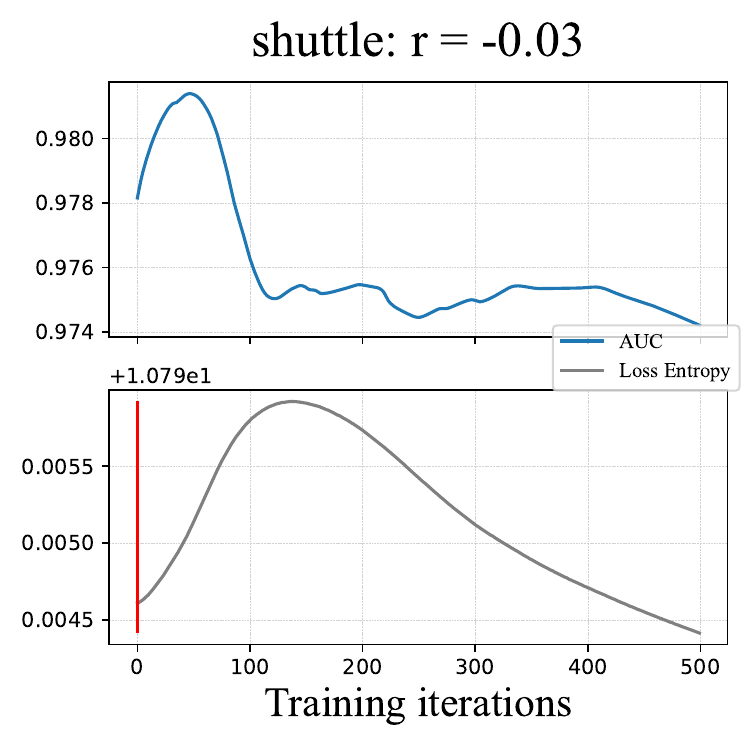}
\includegraphics[width=0.32\textwidth]{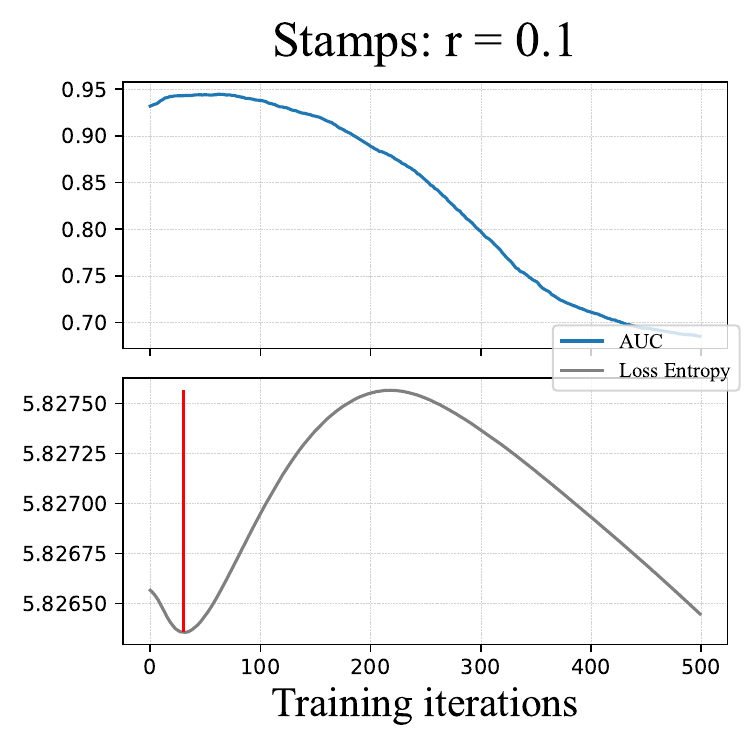}
\includegraphics[width=0.32\textwidth]{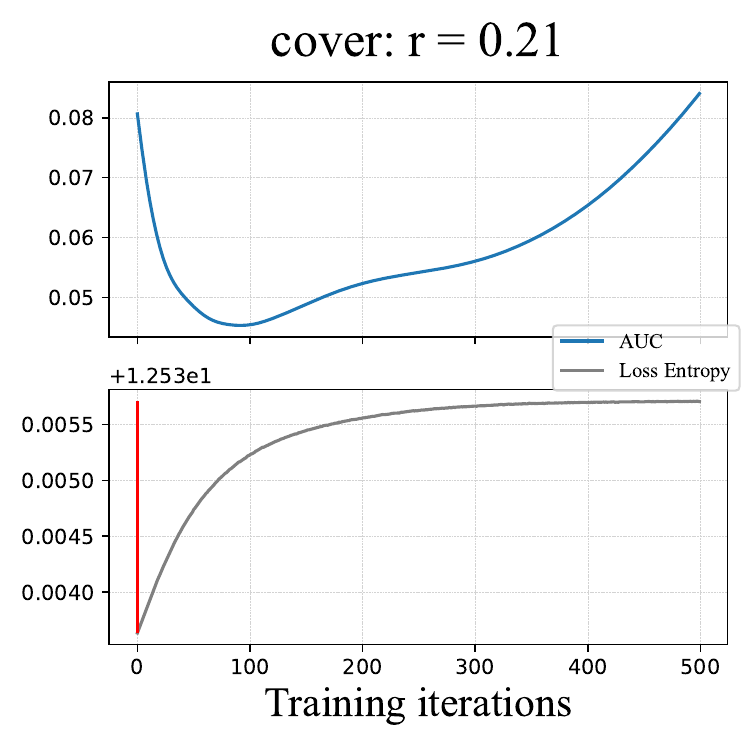}
\includegraphics[width=0.32\textwidth]{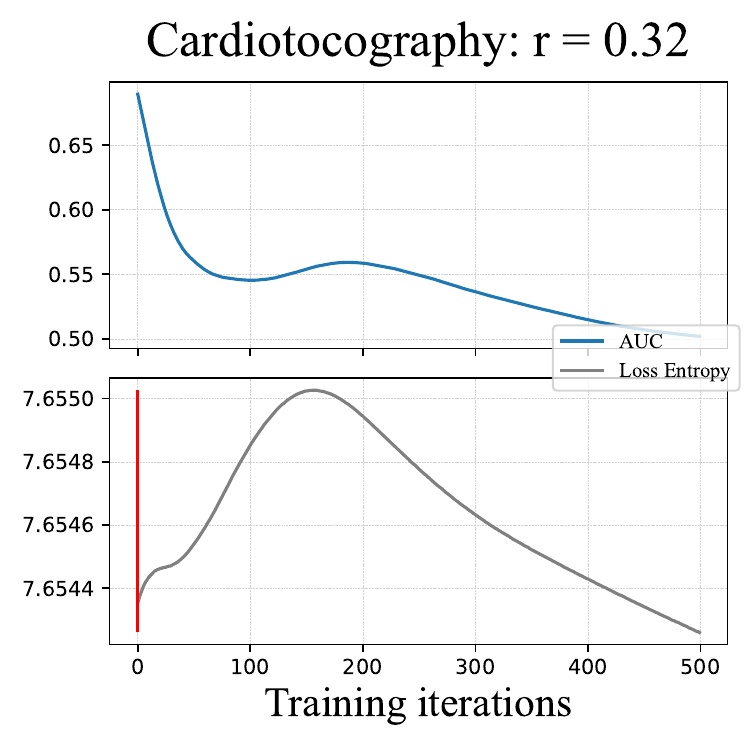}
  \caption{DeepSVDD: AUC curves vs  $H_L$ curves. The red vertical line is the epoch selected by $EntropyStop$. $r$ denotes the Pearson correlation coefficient between AUC and $H_L$.}
  \label{Fig:svdd-all-curve-3}
\end{figure*}

\begin{figure*}
  \centering
\includegraphics[width=0.32\textwidth]{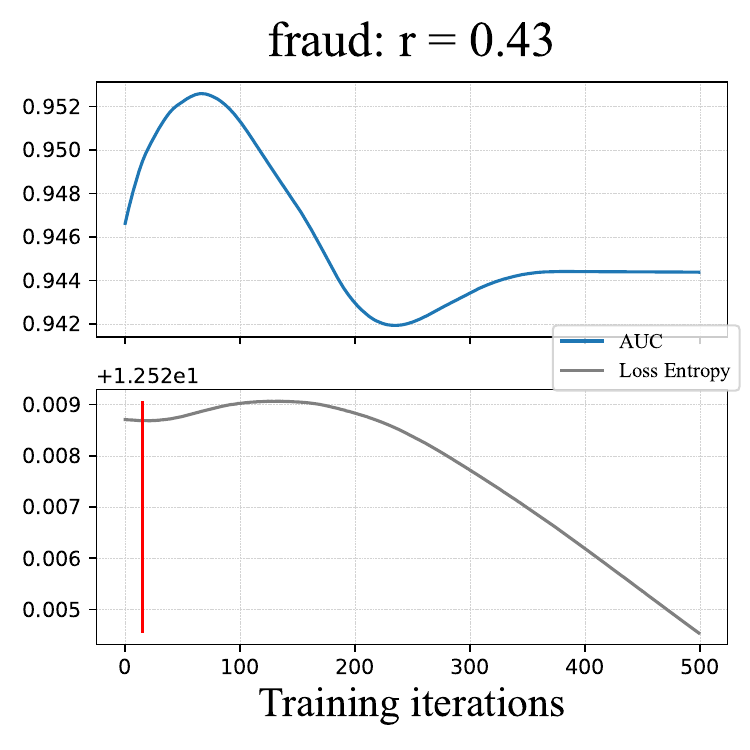}
\includegraphics[width=0.32\textwidth]{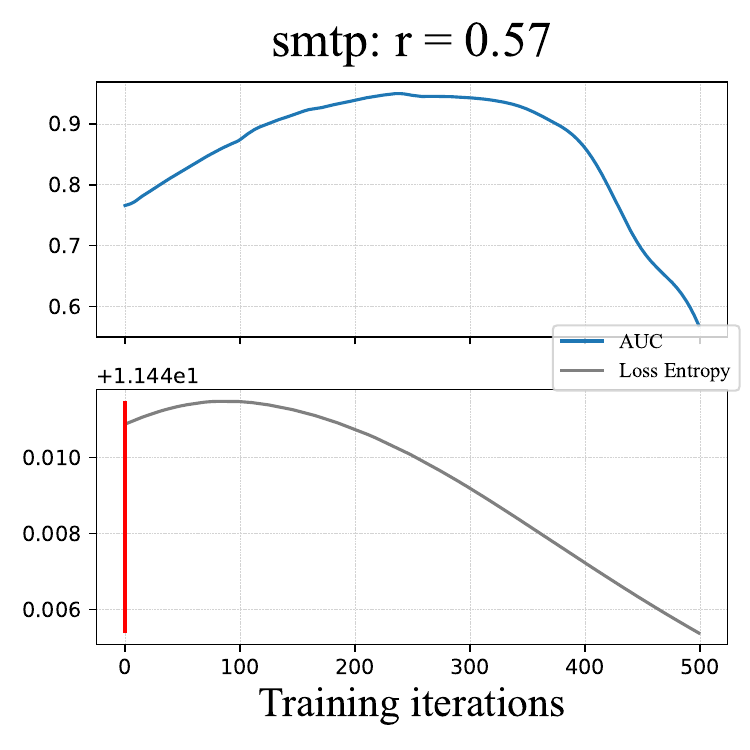}
\includegraphics[width=0.32\textwidth]{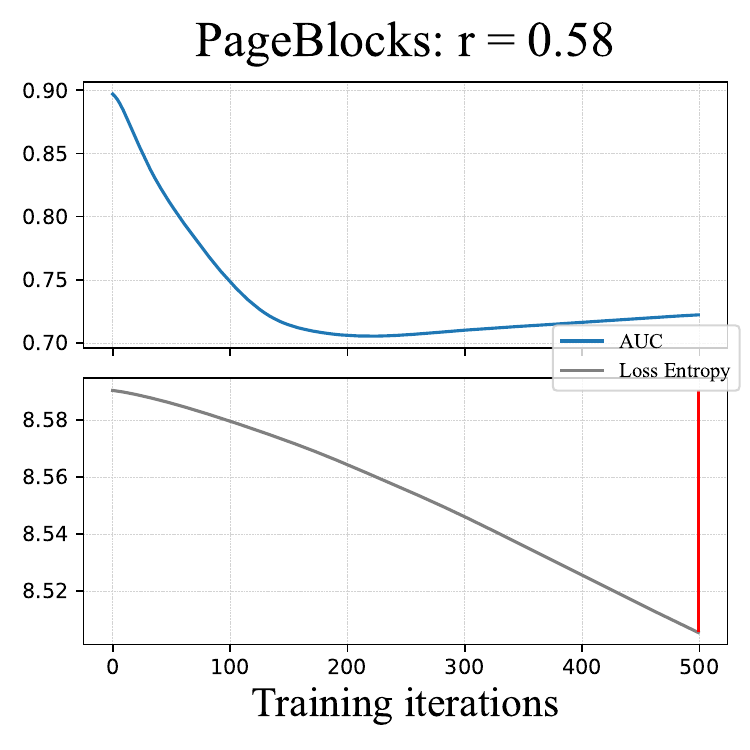}
\includegraphics[width=0.32\textwidth]{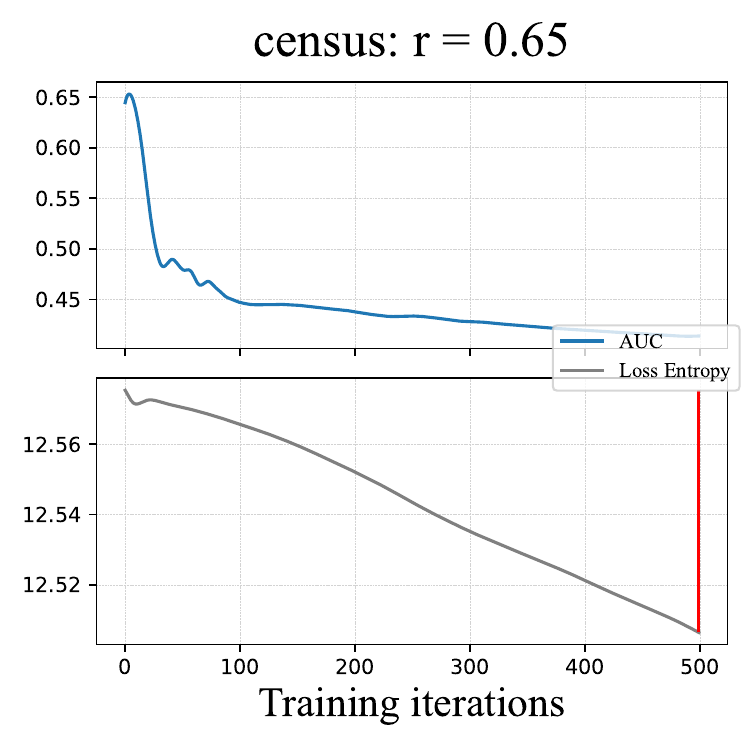}
\includegraphics[width=0.32\textwidth]{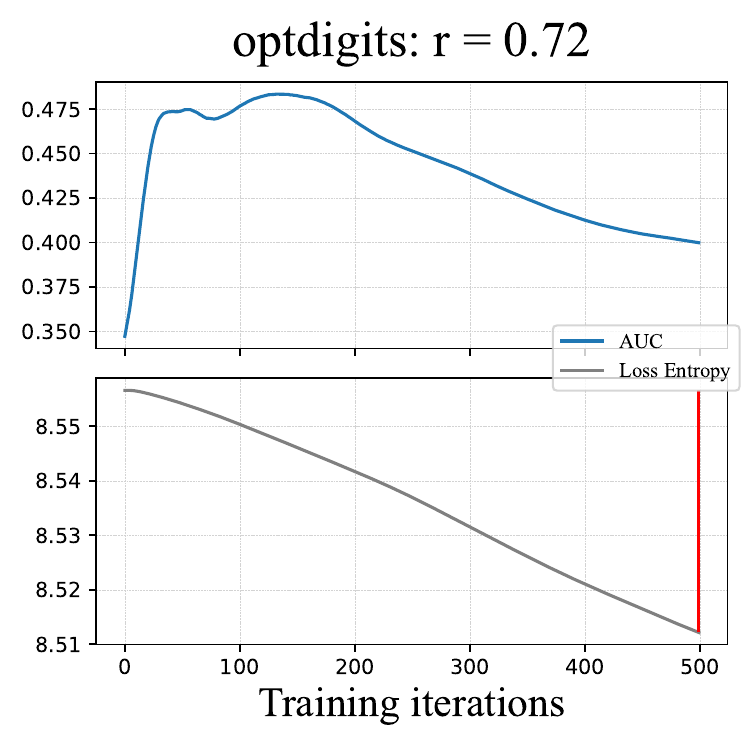}
\includegraphics[width=0.32\textwidth]{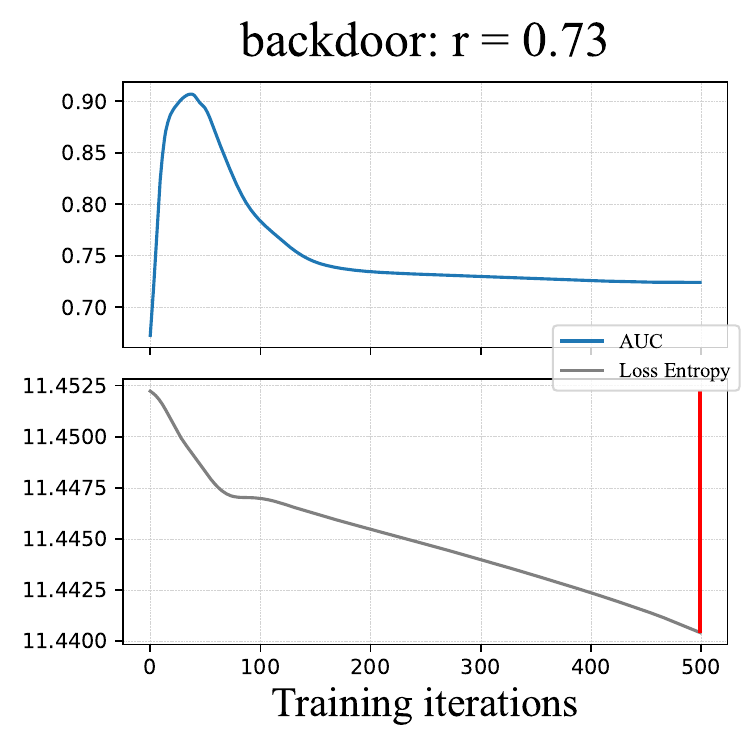}
\includegraphics[width=0.32\textwidth]{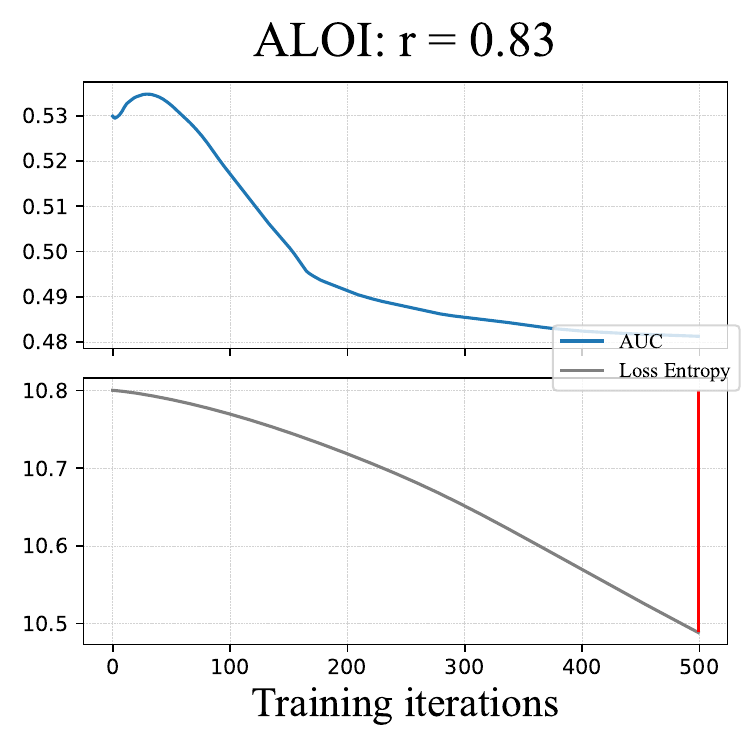}
\includegraphics[width=0.32\textwidth]{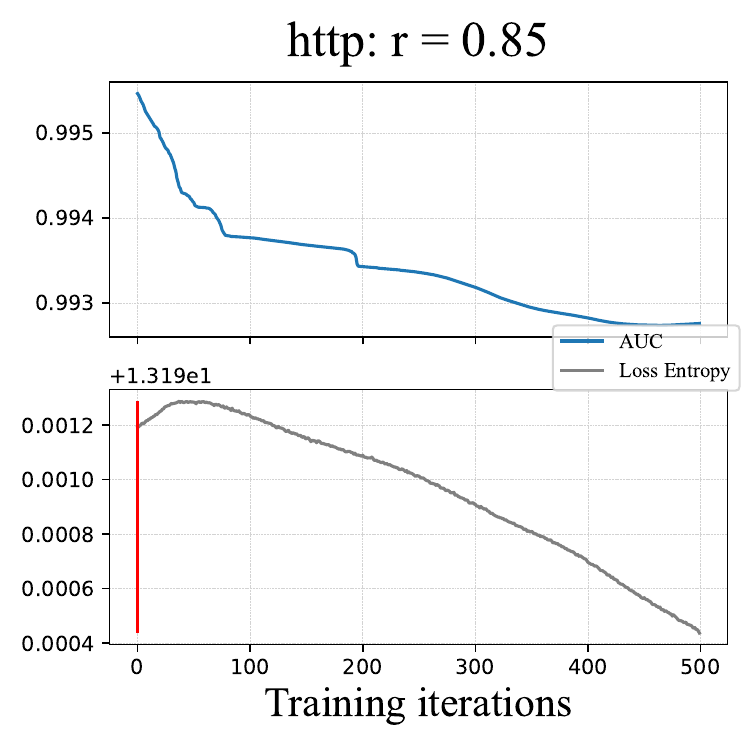}
\includegraphics[width=0.32\textwidth]{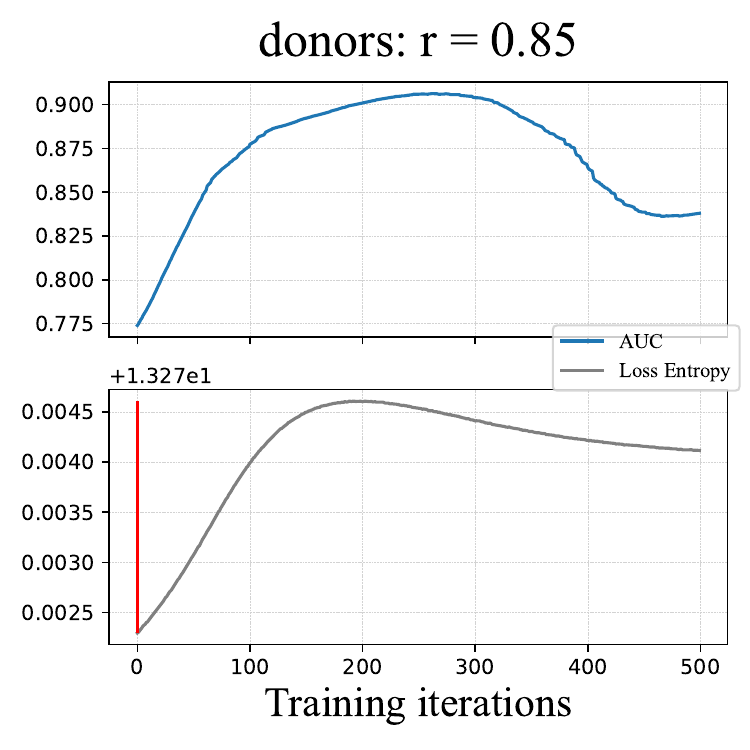}
\includegraphics[width=0.32\textwidth]{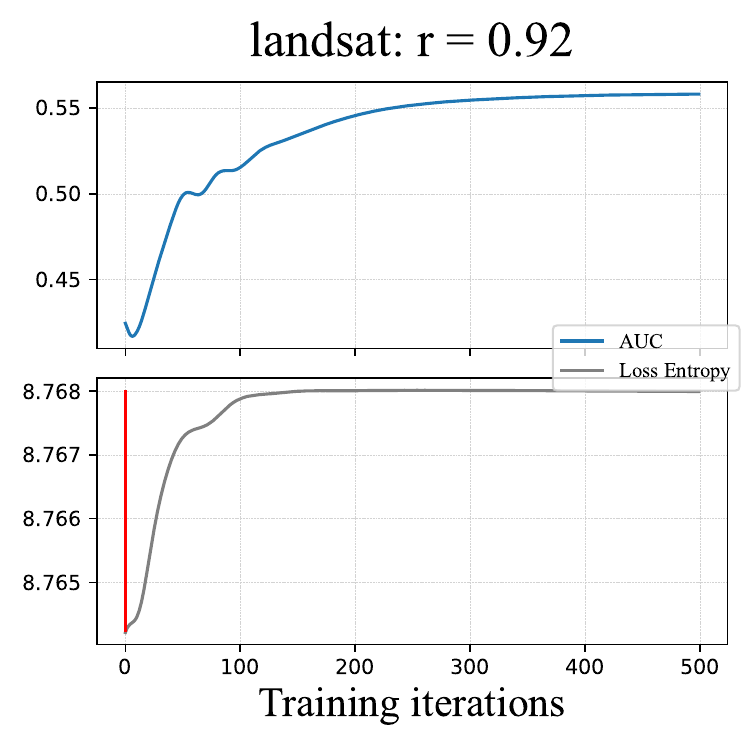}
\includegraphics[width=0.32\textwidth]{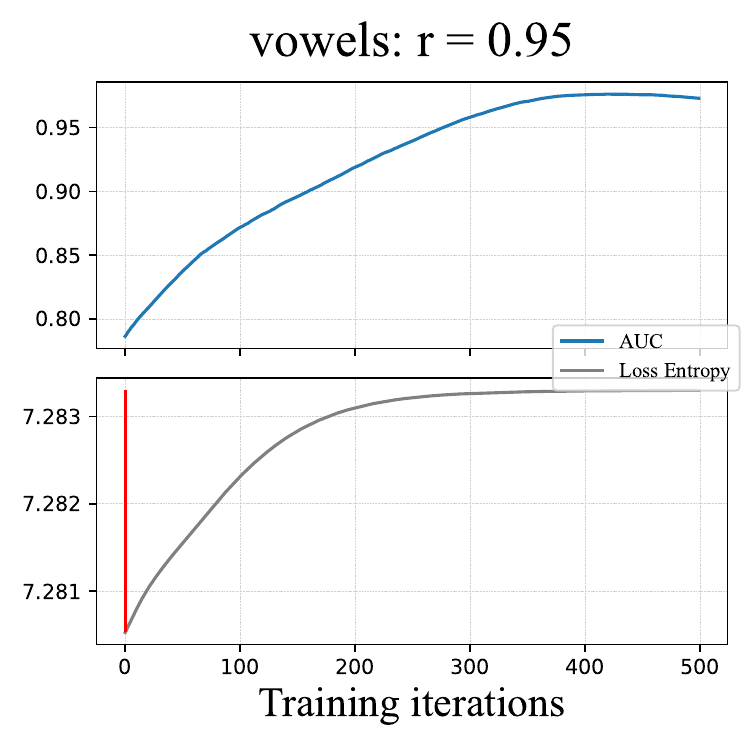}
  \caption{DeepSVDD: AUC curves vs  $H_L$ curves. The red vertical line is the epoch selected by $EntropyStop$. $r$ denotes the Pearson correlation coefficient between AUC and $H_L$.}
  \label{Fig:svdd-all-curve-4}
\end{figure*}

%% file: limit_study_appx.tex
\clearpage
\section{Limitation analysis and case study}

\begin{table}[b]
\small
  \centering
  \caption{The limitation study.}
    \begin{tabular}{l|ccc}
    \toprule
           \textbf{Dataset} & \textbf{\thead{Pearson\\coefficient}} & \textbf{\thead{Label\\Misleading}}& \textbf{\thead{AUC\\Convergence}}\\
    \midrule
 campaign	& 0.92 		& 				&\checkmark	\\
 ALOI		& 0.83 		&\checkmark 	&			\\
 pendigits	& 0.59 		&\checkmark		&			\\
 celeba		& 0.46 		&\checkmark 	&			\\
 yeast		& 0.45 		&\checkmark		&			\\
 fault		& 0.35 		&\checkmark     &			\\
 Pima		& 0.32		&\checkmark 	&			\\
 glass		& 0.28 		&\checkmark 	&			\\
 vertebral	& 0.25 		&\checkmark 	&			\\
 http		& 0.20 		& 				&\checkmark	\\
 PageBlocks	& 0.10 		&\checkmark 	&			\\
 satimage-2	& 0.04 		&\checkmark 	&			\\
 skin		& -0.12 	&\checkmark		&			\\

    \bottomrule
    \end{tabular}%
  \label{Limitation analysis}
\end{table}%

\label{pseudo inlier-study}

In the experiment described in Sec \ref{sec:correlation-exp}, we assess the negative correlation between loss entropy $H_L$ and AUC utilizing the Pearson correlation coefficient, abbreviated as $r$ for clarity.
Fig. \ref{Fig:all-curve-1}, \ref{Fig:all-curve-2}, \ref{Fig:all-curve-3}, and \ref{Fig:all-curve-4} illustrate the evolution of AUC and entropy curves of AE throughout the training period for 47 datasets, ranked by descending order of their negative correlation strength. Notably, while the loss entropy $H_L$ demonstrates a strong negative correlation with AUC across several datasets, there are still some datasets that exhibit weak or even positive correlations, such as ALOI. We attribute this primarily to the following two reasons. We categorized the datasets in Table \ref{Limitation analysis} on which entropy stop does not perform well.
\begin{itemize}
    \item \textbf{Label misleading}: The existence of a large number of pseudo inliers in these datasets. These pseudo inliers exhibit an outlier pattern while being labeled as inliers.
    \item \textbf{The convergence of AUC}: \newtext{The AUC is nearly stationary throughout the entire training process. In such cases, the influence of zigzag fluctuation of AUC and entropy curve outweighs the macroscopical correlation, showing a weak correlation. Then, the entropy could not reflect the changes in AUC. 
     In this case, the ineffectiveness of $H_L$ actually does not influence the final performance, while the training time may still saved by early stopping.}
\end{itemize}
Note that label misleading may occur because the labels only mark one type of anomaly, or due to a mismatch between the model's anomaly assumption and the type of anomalies identified by the labels. These two scenarios are interconnected, and we categorize them collectively under the term "label misleading".

To quantitatively analyze these two factors, we define the following measurement. 

\subsection{\textbf{Measurement for Label Misleading}} Firstly, we define pseudo inliers as those inliers whose loss values are greater than the expected outlier loss, i.e., $\{v_i | v_i > \mathcal{L}_{out}, v_i \in V^-\}$. Here, $V^-$ and $V^+$  are the sets of inlier losses and outlier losses, respectively, while $$\mathcal{L}_{out} = \frac{\sum_{v_i \in V^+} v_i}{|V^+|}$$ is the average loss value of outliers.

\vspace{1mm}
\noindent \textbf{Pseudo Inlier Ratio $R_{pi}$:} To quantify the proportion of pseudo inliers relative to labeled outliers in the dataset, we propose the following metric:
$$ R_{pi} = \frac{|\{ v_i |v_i > \mathcal{L}_{out}, v_i \in V^-\}|}{|V^{+}|}$$
This metric $R_{pi}$ reflects the number of pseudo inliers relative to labeled outliers in the dataset. For example, given $n$ outliers in the dataset, then the $R_{pi}=2$  indicates $2n$ pseudo inliers whose losses are greater than $\mathcal{L}_{out}$.
Essentially, $R_{pi}$ measures the amount of potential anomalies that come from other types and have not been labeled.

The overall outlier ratio is also important. For example, when $R_{pi}$=1 and outlier ratio is 30\%, the proportion of both labeled outliers and pseudo inliers in the dataset could account for 60\%. This leaves inliers unable to provide sufficient learning signals for the model, thus weakening or even breaking inlier priority.

\vspace{1mm}
\noindent \textbf{The Trend of  $R_{pi}$:} The change in $R_{pi}$ during training can also reflect the the existence of label misleading. If $R_{pi}$ is small at the initial training stage and continues to increase with training, it suggests that increasingly more inliers' losses exceed $\mathcal{L}_{out}$, to some extent indicating that the labeled outliers are more like inliers compared to the pseudo inliers. Under this circumstance, inlier priority fails and such UOD model may not be suitable for this dataset. On the other hand, if we observe a significant decrease on $R_{pi}$ during the training, which may suggest the signals that the model learnt from the majority can be generailze to these pseudo inliers. In this case, a large $R_{pi}$ at the initial stage may not cause problem.

In this case, we believe that the phenomenon of label misleading can be identified from two perspectives:
\begin{itemize}
    \item $R_{pi}$ is high and does not decrease.
    \item The overall proportion of pseudo inliers and outlier ratio in the dataset is large, for example, greater than 50\% of the dataset ratio.

\end{itemize}

To sum up, the existence of these pseudo inlier weakens the dependency between AUC and entropy, as AUC is based on labeled outliers, while entropy takes both pseudo inliers and labeled outliers into account.

\subsection{The Measurement of the Converged AUC}
We regard the AUC as converged or having minor changes throughout the entire training process if the changes of AUC is less than 0.05, i.e., $max(AUC) - min(AUC) \le 0.05$. In this case, regardless of whether the strong negative correlation exists, it has minimal impact on the final performance of the model.

\subsection{Case Study}
we explain the reasons for invalidity of $H_L$ on these datasets with the worst negative correlation, including: \textit{campaign, ALOI, pendigits, celeba, yeast}. 
% In Table \ref{Limitation analysis}, we have listed the specific assumptions not met by these datasets.

\noindent \textbf{campaign:} 
\begin{figure}[!htbp]
    \centering
    \includegraphics[width=0.35\textwidth]{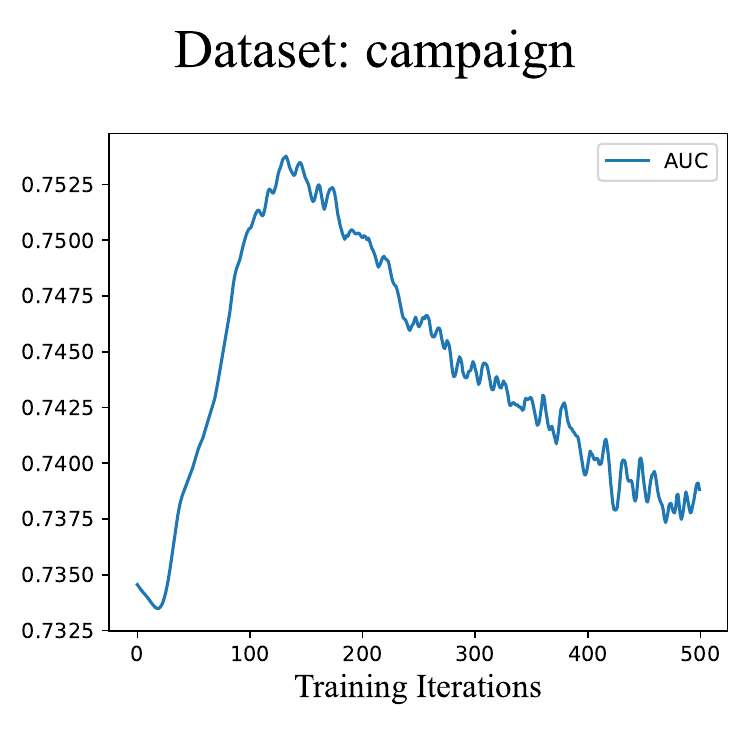}
    \caption{The AUC curve for AE training on campaign.}
    \label{fig:campaign-ana}
\end{figure}
As shown in Fig. \ref{fig:campaign-ana}, \newtext{the maximum AUC is 0.752 while the minimum AUC is 0.732.}
 Therefore, its AUC is nearly stationary  during the training of AE on Dataset, which meets our analysis of \textit{the convergence of AUC}. Thus, the positive relationship between $H_L$ and AUC is not a significant issue.

\noindent \textbf{ALOI:}
\begin{figure}[!htbp]
    \centering
    \includegraphics[width=0.35\textwidth]{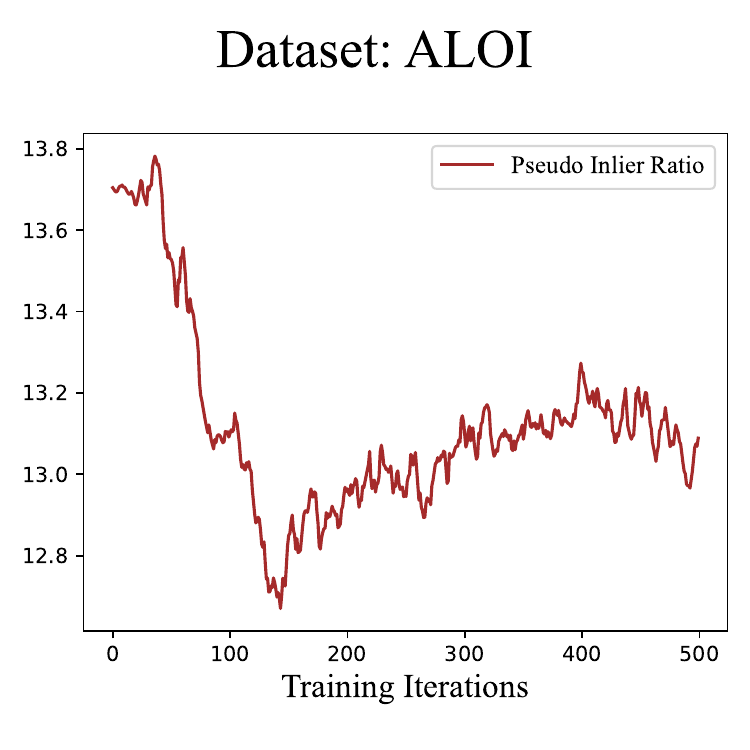}
    \caption{The $R_{pi}$ curve for AE training on ALOI.}
    \label{fig:ALOI-ana}
\end{figure}
As shown in Fig. \ref{fig:ALOI-ana}, we see that the $R_{pi}$ always remains greater than 13. Although it decreases slightly, the number of pseudo outliers always far exceeds the number of labeled outliers, which leads to the failure of $H_L$.

\noindent \textbf{pendigits and celeba:}
As shown in Fig. \ref{fig:two-ana}, we observe a rapid increase in indicators on two datasets, indicating that there are more and more pseudo inliers in the dataset, indicating the existence of label misleading.
\begin{figure}[!htbp]
    \centering
    \includegraphics[width=0.235\textwidth]{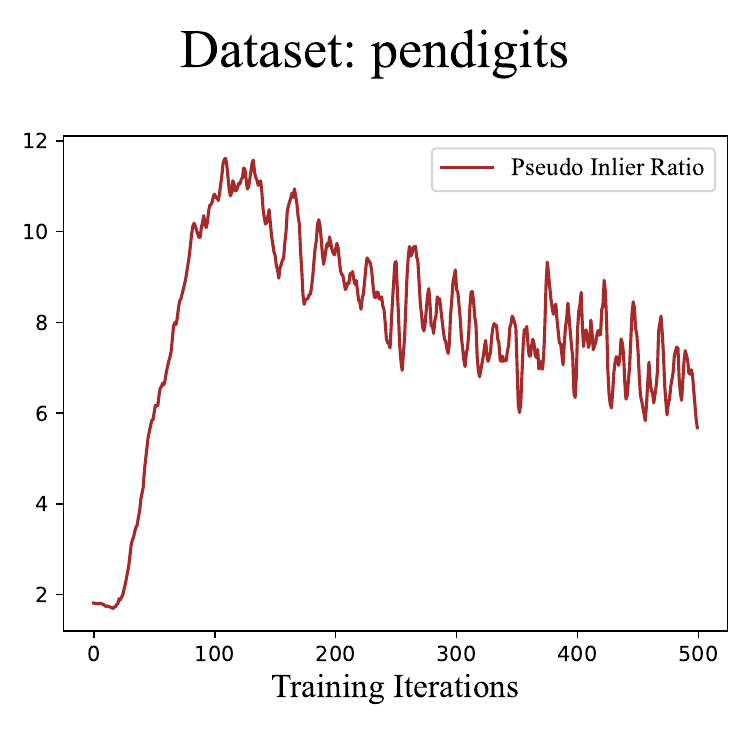}
    \includegraphics[width=0.235\textwidth]{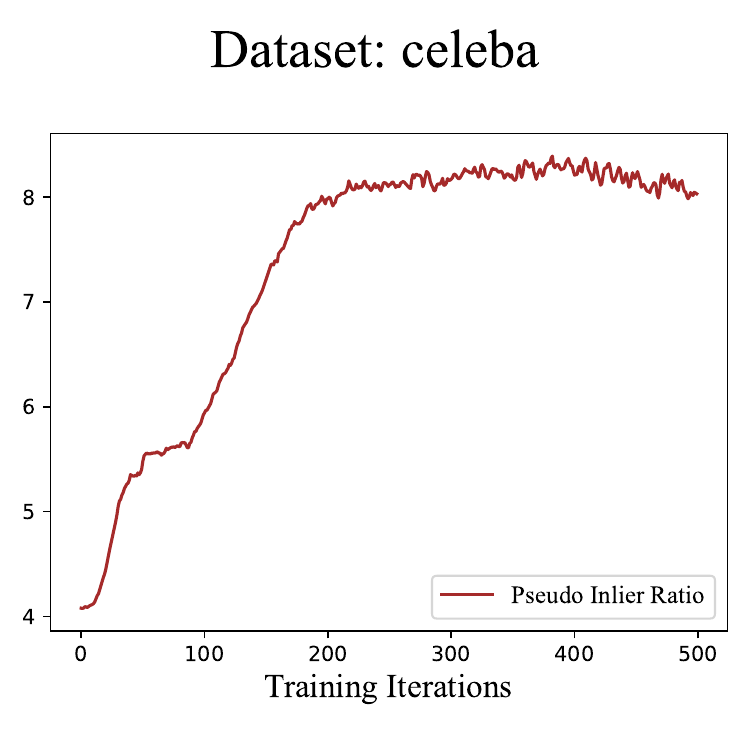}
    \caption{The $R_{pi}$ curves for AE training on pendigits and celeba.}
    \label{fig:two-ana}
\end{figure}

\noindent \textbf{yeast:}
Although the pseudo outlier ratio on yeast is not high in Fig. \ref{fig:yeast-ana}, we found that the labeled-outlier ratio of yeast accounts for approximately 34\%, which means that a coefficient of 1 will cause the total proportion of the pseudo outlier ratio and labeled-outlier ratio to reach 70\% of the data. The remaining 30\% of inliers are not enough to provide enough learning signals for the model to learn.
\begin{figure}[!htbp]
    \centering
    \includegraphics[width=0.35\textwidth]{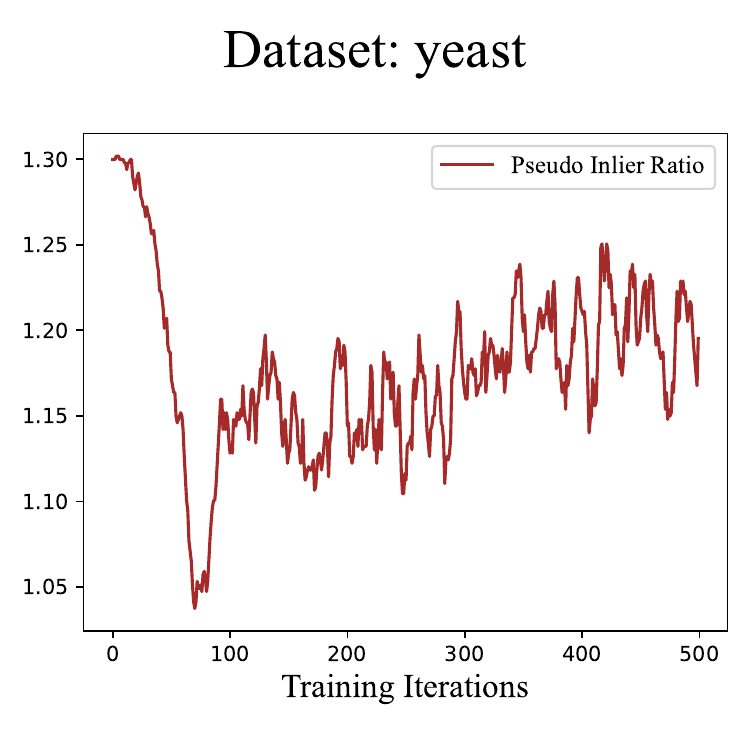}
    \caption{The $R_{pi}$ curve for AE training on yeast.}
    \label{fig:yeast-ana}
\end{figure}

%% file: Appx-EntropyStop-guide.tex
\clearpage
\section{Parameter Study of  \textit{EntropyStop}}
\subsection{The Guidelines for tuning parameteres of \textit{EntropyStop}}
\label{appx:guide-for-tuning}
In this  section, we provide guidelines on how to tune the hyperparameters (HPs) of EntropyStop when working with unlabeled data. The three key parameters are the learning rate, 
$k$, and $R_{down}$. The learning rate is a crucial factor as it significantly impacts the training time. $k$ represents the patience for finding the optimal iteration, with a larger value improving accuracy but also resulting in a longer training time. $R_{down}$ sets the requirement for the significance of the downtrend.

\begin{figure}[!htbp]
  \centering
  \includegraphics[width=0.20\textwidth]{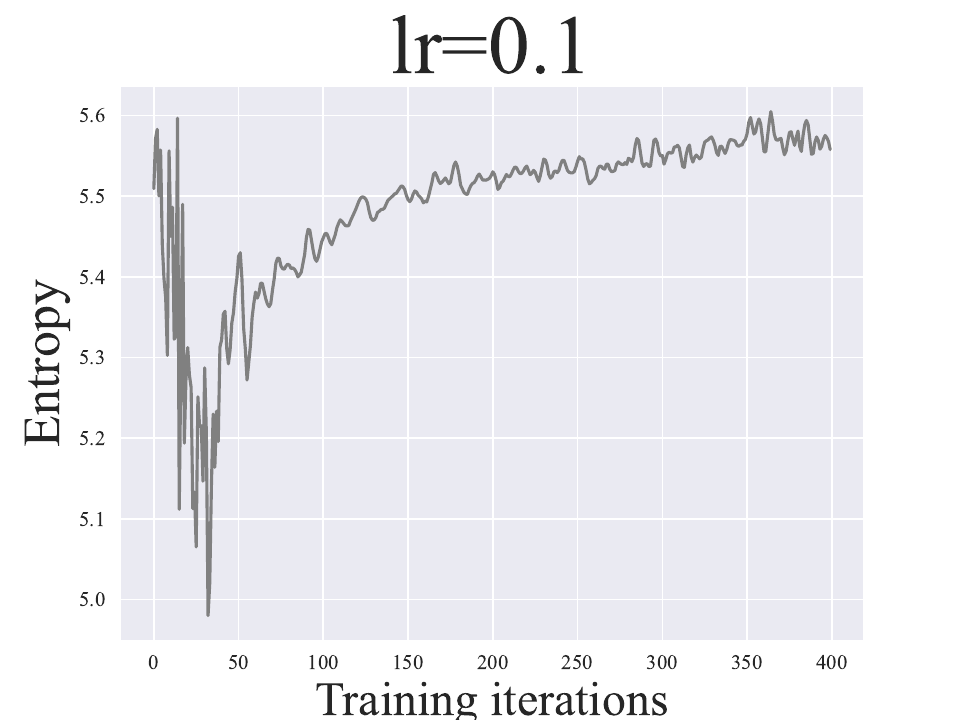}
  \includegraphics[width=0.20\textwidth]{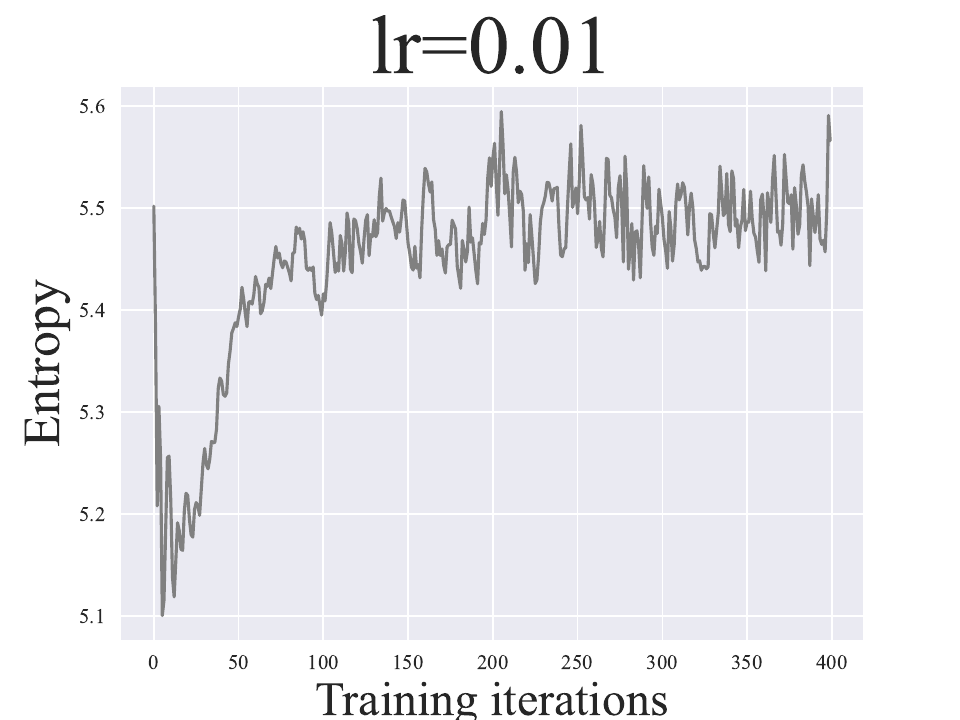}
  \includegraphics[width=0.20\textwidth]{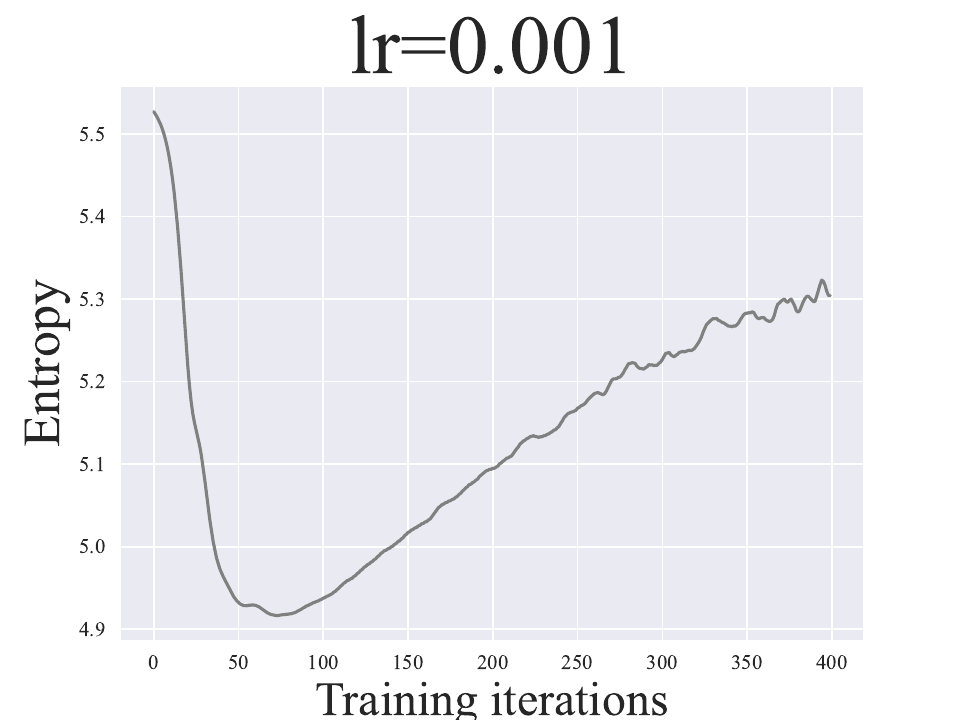}
  \includegraphics[width=0.20\textwidth]{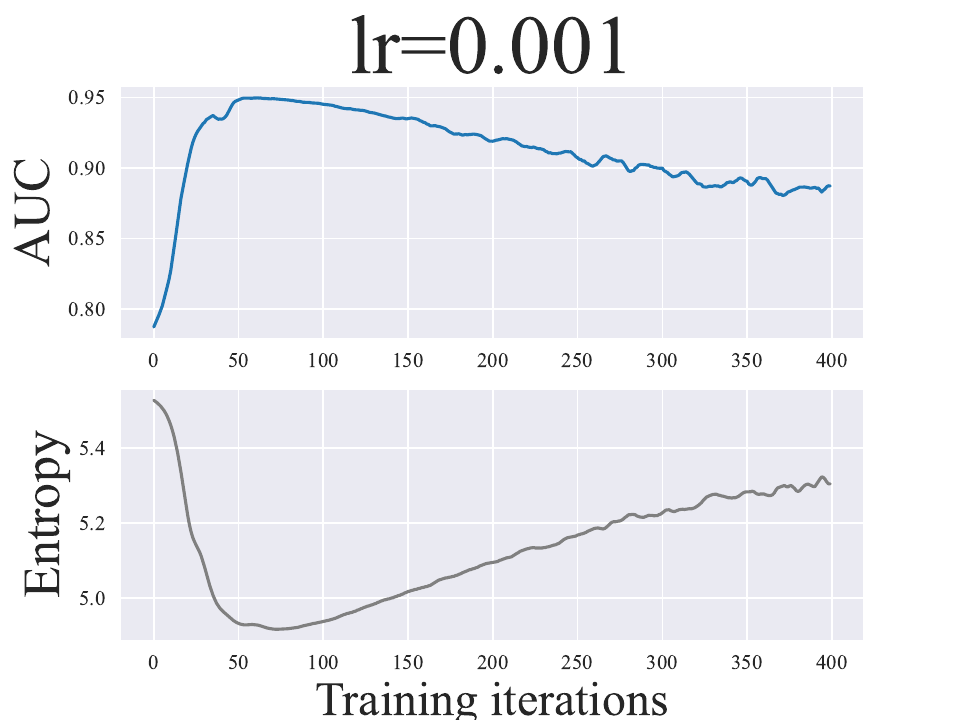}
  \caption{The loss entropy curve of training Autoencoder (AE) on dataset \textit{Ionosphere} with different learning rate.}
  \label{Fig:tune-lr-ae-iono}
\end{figure}

\textbf{Tuning learning rate.} When tuning these parameters, the learning rate should be the first consideration, as its value will determine the shape of the entropy curve, as shown in Fig. \ref{Fig:tune-lr-ae-iono}. For illustration purposes, we first set a large learning rate, such as 0.1, which is too large for  training  autoencoder (AE). This will result in a sharply fluctuating entropy curve, indicating that the learning rate is too large. By reducing the learning rate to 0.01, a less fluctuating curve during the first 50 iterations is obtained, upon which an obvious trend of first falling and then rising can be observed.
Based on the observed entropy curve, we can infer that the training process reaches convergence after approximately 50 iterations. Meanwhile, the optimal iteration for achieving the best performance may occur within the first 25 iterations. However, the overall curve remains somewhat jagged, indicating that the learning rate may need to be further reduced. After reducing the learning rate to 0.001, we observe a significantly smoother curve compared to the previous two, suggesting that the learning rate is now at an appropriate level.

A good practice for tuning learning rate is to begin with a large learning rate to get a overall view of the whole training process while the optimal iteration can be located. For the example in Fig. \ref{Fig:tune-lr-ae-iono}, it is large enough to set learning rate to 0.01 for AE model. Then, zoom out the learning rate to obtain a smoother curve and employ EntropyStop to automatically select the optimal iteration.

\begin{figure}[!htbp]
  \centering
  \includegraphics[width=0.3\textwidth]{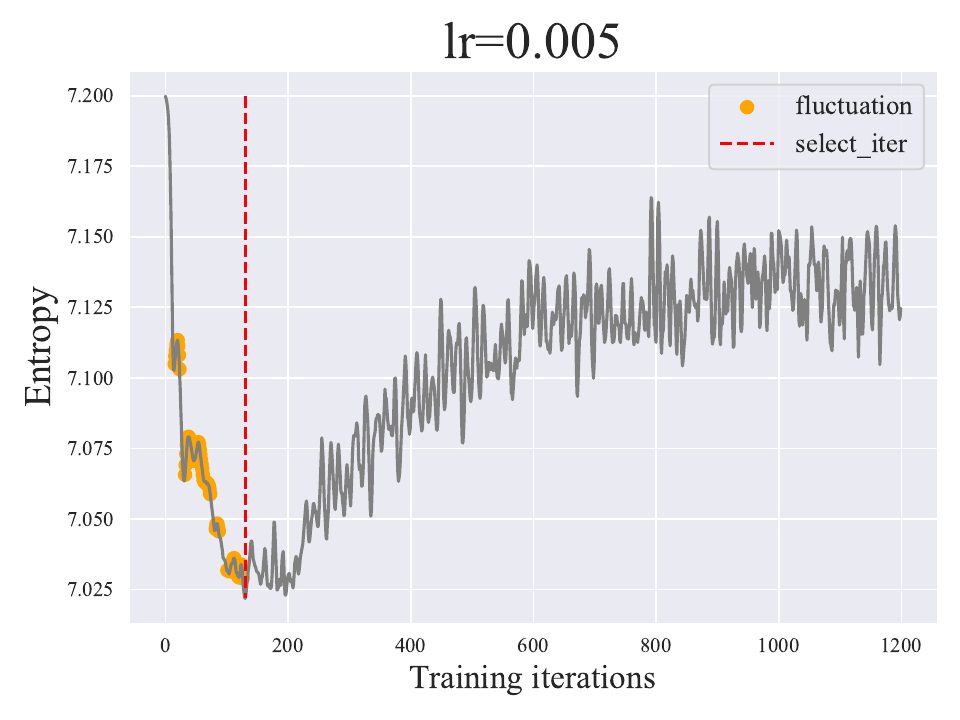}
  \caption{The example of fluctuations (or rises) during the downtrend of entropy curve when training  AE on  dataset \textit{vowels}. The value of $k$ should be set larger than the width of all fluctuations.}
  \label{Fig:fluctuation}
\end{figure}
\textbf{Tuning $k$ and $R_{down}$.} 
After setting the learning rate, the next step is to tune $k$. If the entropy curve is monotonically decreasing throughout the downtrend, then $k=1$ and $R_{down}=1$ will suffice. However, this is impossible for most cases. Thus, an important role of $k$ and $R_{down}$ is to tolerate the existence of small rise or fluctuation during the downtrend of curve. Essentially, the value of $k$ is determined by the maximum width of the fluctuations or small rises before encounting the opitmal iteration. As shown in Fig. \ref{Fig:fluctuation}, the orange color marks the fluctuation area of the curve before our target iteration. The value of $k$ should be set larger than the width of all these fluctuations. For  the example in Fig. \ref{Fig:fluctuation}, as long as $k \geq 50$ , EntropyStop can select the target iteration.

\begin{figure}[!htbp]
  \centering
  \includegraphics[width=0.3\textwidth]{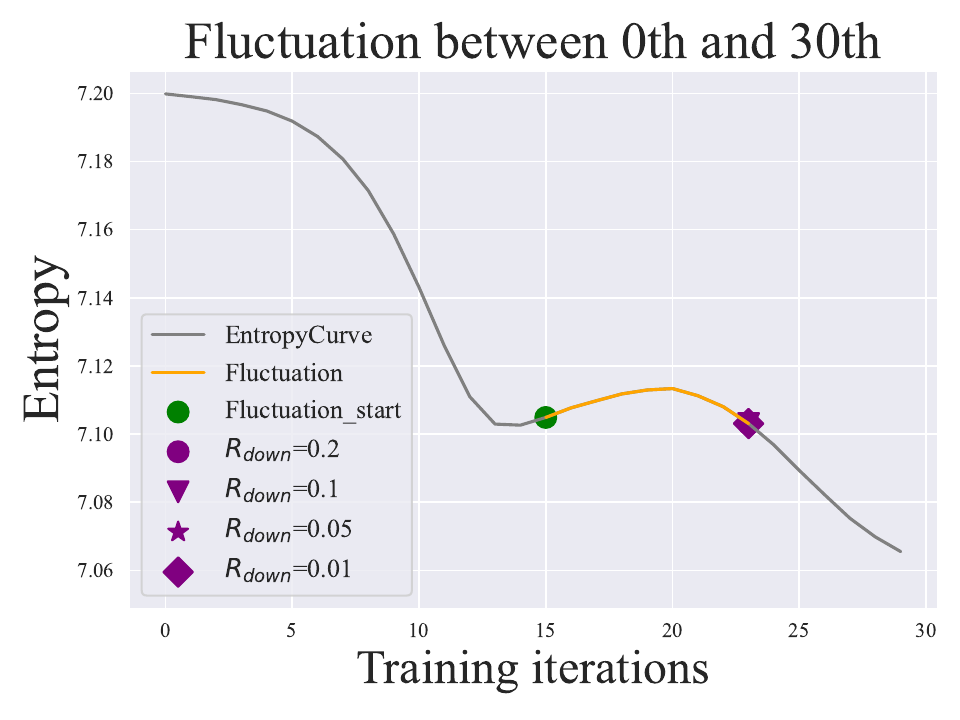}
  \includegraphics[width=0.3\textwidth]{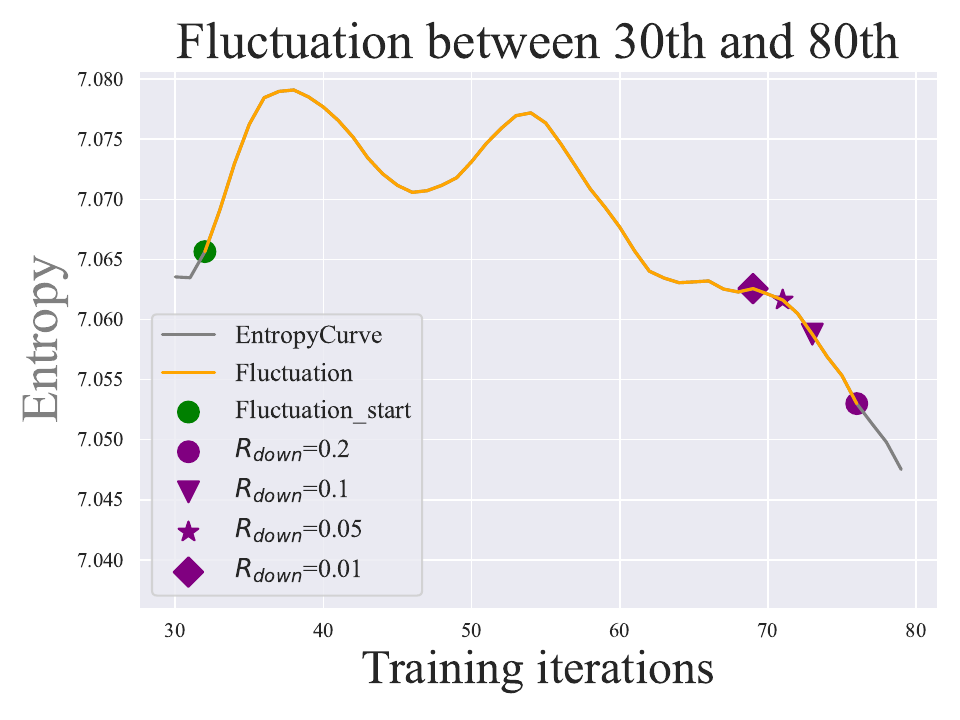}
  \caption{Explanation of the effect of $R_{down}$ in tolerating the existence of fluctuations in the entropy curve shown in Fig. \ref{Fig:fluctuation}.}
  \label{Fig:r-down-explanation}
\end{figure}

Regarding $R_{down}$, a visualization of the effect of $R_{down}$ is depicted in Fig. \ref{Fig:r-down-explanation}. When a small fluctuation (or rise) occurs during the downtrend of the curve, suppose $e_i$ is the start of this fluctuation. Then, the new lowest entropy points $e_q$ that satisfies the downtrend test of varying $R_{down}$ is close to each other. This explains the robustness of EntropyStop to $R_{down}$.

\begin{figure}[!htbp]
  \centering
% \minipage{1.0\textwidth}
    \includegraphics[width=0.25\textwidth]{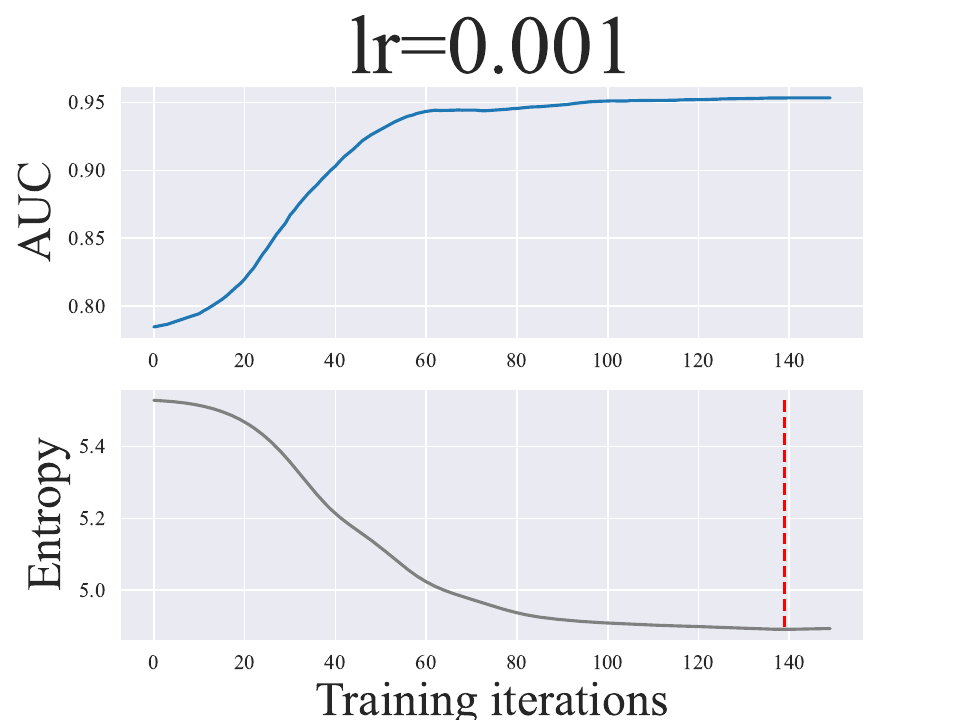}
  \includegraphics[width=0.25\textwidth]{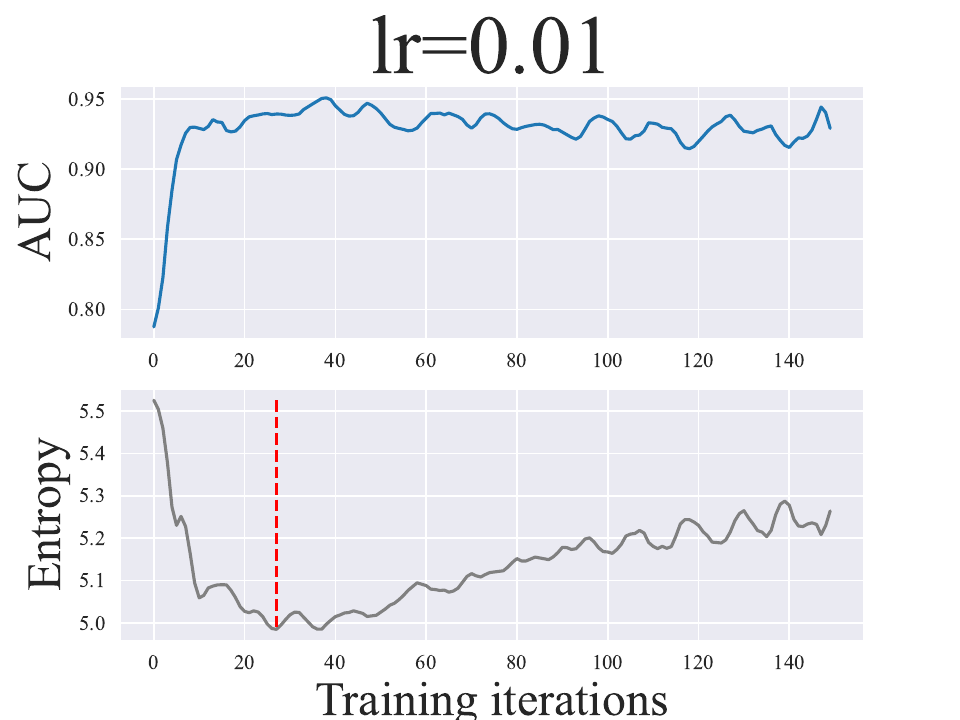}
    \includegraphics[width=0.25\textwidth]{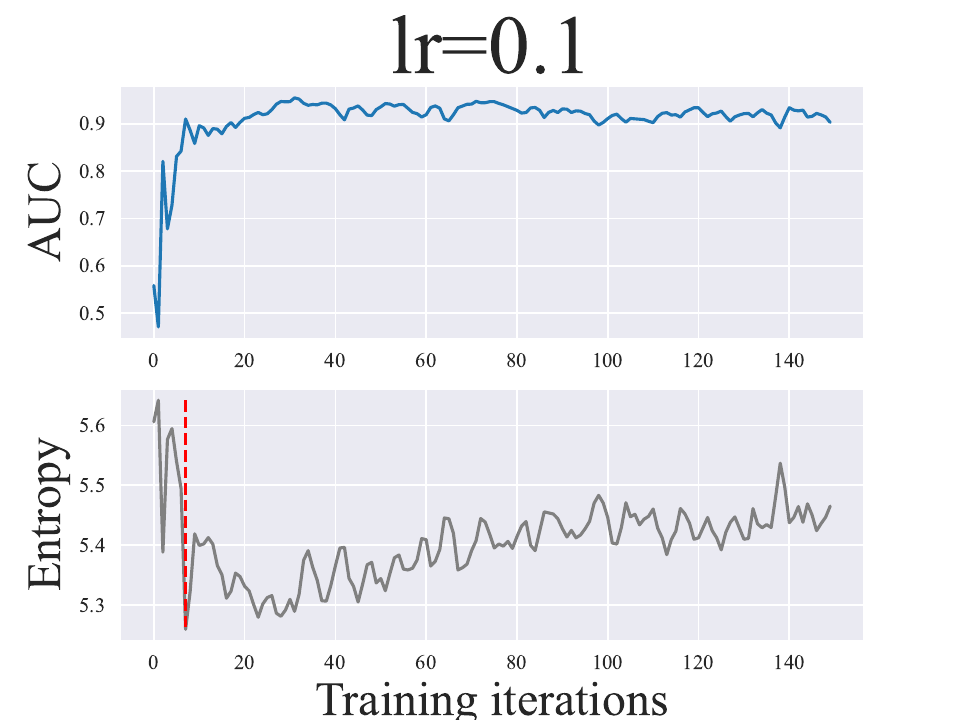}
  \caption{The effect of tolerating the fluctuations when $k$ is set to 50 and $R_{down}$ is set to 0.1 for the training of AE on the dataset \textit{Ionosphere}. The displayed training process only includes the first 150 iterations. The red dashed line marks the iteration selected by \textit{EntropyStop}.}
  \label{Fig:tolerate-to-lr}
\end{figure}

Owning to the effectiveness of $k$ and $R_{down}$ in tolerating the fluctuations, even the entropy curve is not smooth enough due to a large learning rate, the target iteration can still be selected by EntropyStop.(see Fig. \ref{Fig:tolerate-to-lr}). Nevertheless, we still recommend fine-tuning the learning rate to achieve a smooth entropy curve, which will ensure a stable and reliable training process.

% The setting of $k$ in our experiments is available at Table \ref{tab:k-for-entropy}.

% \begin{wrapfigure}{r}{0.5\textwidth}
% % {0.3\textwidth}
% \vspace{-0.4cm}
%     \begin{center}
% \includegraphics[width=0.5\textwidth]{fig/study_r_down.pdf}
% \includegraphics[width=0.45\textwidth]{fig/study_n_eval.pdf}
%     \end{center}
%     \caption{(a) The AUC distributions of EntropyAE with two different $R_{down}$ on 5 datasets. (b) The pearson correlation between entropy curve computed on the whole dataset and the evaluation set with size  $N_{eval}$.  Fig \ref{Fig:appx-ae-n_eval} in Appx. \ref{appx:additonal-fig}  offers a more intuitive display.}
%     \label{fig:para-study}
% % \end{wrapfigure}
% \end{wrapfigure}
\subsection{Parameter Sensitive Study}
\label{entropystop-hp-study}
We study the sensitivity of our approach to  \textit{batch size} and $R_{down}$.
Generally,  the larger \textit{batch size} can result a more stable gradient for optimization. In this case, we set  \textit{batch size} $= 1024$ in our experiments for improvement study in Sec. \ref{sec: ensemble-cmp-exp}. Here, we keep all the hyperparameters (HPs) of the AE exactly the same, except for \textit{batch size} and $R_{down}$, to precisely assess the sensitivity to these two parameters.

\subsubsection{\textbf{batch size}:}
We conduct experiments with two batch\_size, i.e., 1024 and 256. As results shown in Table \ref{tab: bsize-table}, different batch\_size does not bring significant influence.

  \begin{table}[!htbp]
  \centering
  \caption{Impact of \textit{batch size} on EntropyAE Performance}
    \begin{tabular}{r|r|r}
    \toprule
    \multicolumn{1}{l|}{batch\_size} & \multicolumn{2}{c}{EntropyAE} \\
\cmidrule{2-3}          & \multicolumn{1}{c|}{AUC} & \multicolumn{1}{c}{AP} \\
    \midrule
    256   & 0.7687  & 0.3621  \\
    1024  & 0.7689  & 0.3611  \\
    \bottomrule
    \end{tabular}%
  \label{tab: bsize-table}%
\end{table}%

\begin{figure}[!htbp]
    \centering
\includegraphics[width=0.35\textwidth]{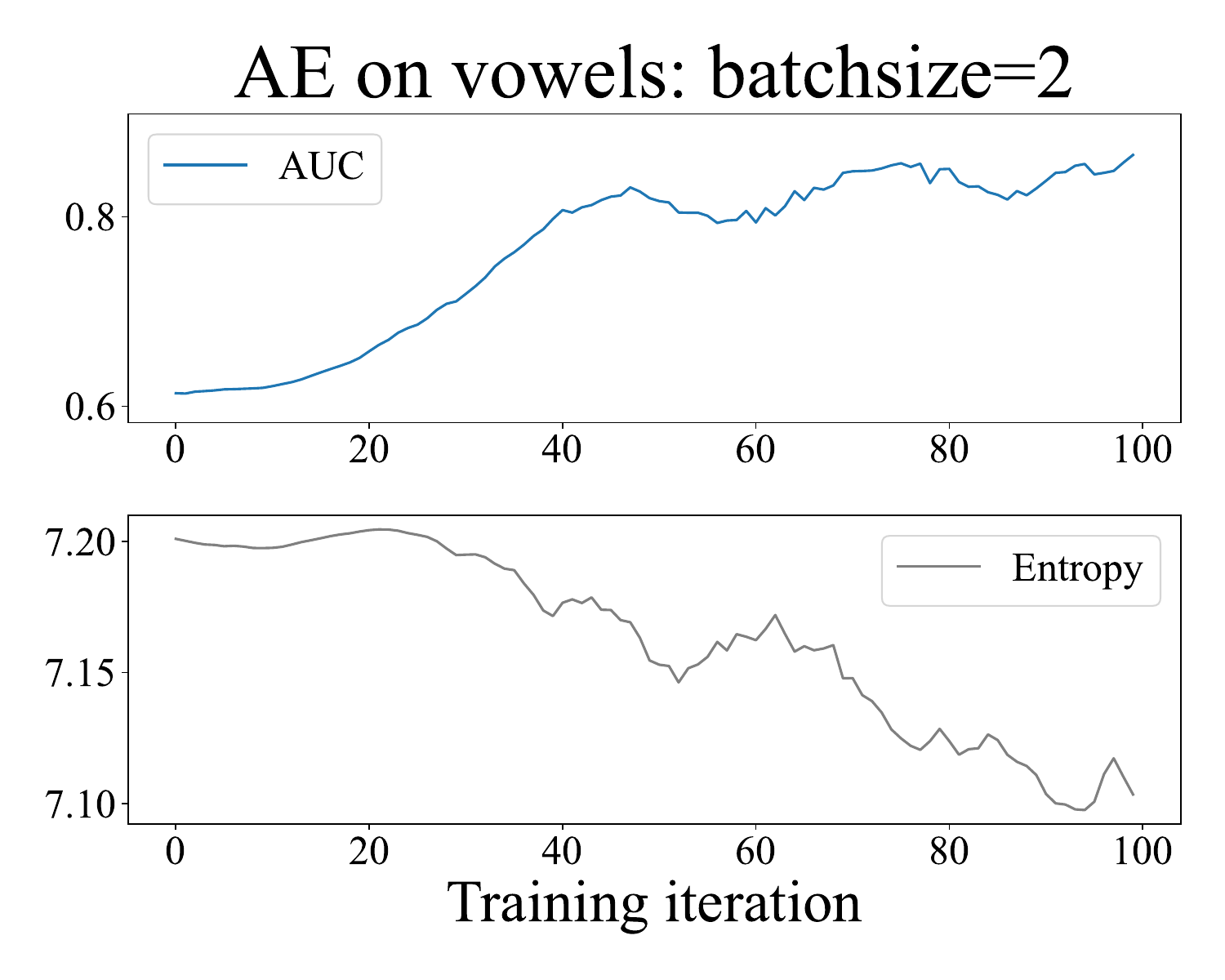}
  \caption{The loss entropy curves of the  training process of AE on the \textit{vowels} dataset with batch size = 2.}
  \label{Fig:appx-ae-batchsize}
\end{figure}

To further investigate, we reduced the \textit{batch size} to 2 to precisely examine the relationship of the AUC and $H_L$ curves. As shown in Fig. \ref{Fig:appx-ae-batchsize}, the result reveals that the AUC and $H_L$ curves still exhibit a strong negative correlation. The primary effect of reducing the\textit{ batch size} is the introduction of additional fluctuations in AUC and $H_L$, attributable to the less stable optimization of the loss.

 \subsubsection{\textbf{$R_{down}$:}}
 Although we have provided a framework for unsupervised adjustment of the $R_{down}$ parameter in Appx. \ref{appx:guide-for-tuning}, it is useful to illustrate that our approach exhibits low sensitivity to variations in $R_{down}$.
We adjusted $R_{down}$ to 0.1 and 0.01 to investigate the performance impact on EntropAE. 

\begin{table}[!htbp]
  \centering
  \caption{Impact of $R_{down}$ on EntropyAE Performance}
    \begin{tabular}{r|r|r}
    \toprule
    \multicolumn{1}{l|}{$R_{down}$} & \multicolumn{2}{c}{EntropyAE} \\
\cmidrule{2-3}          & \multicolumn{1}{c|}{AUC} & \multicolumn{1}{c}{AP} \\
    \midrule
    0.1   & 0.7687  & 0.3621  \\
    0.01  &  0.7735  & 0.3601  \\
    \bottomrule
    \end{tabular}%
    \label{tab: rdown-res}%
\end{table}%

As illustrated in Table \ref{tab: rdown-res}, our findings suggest that a smaller $R_{down}$ tends to yield a marginally higher AUC and a slightly lower AP, although the differences are not statistically significant.

\subsubsection{\textbf{Conclusion}}
Our experiments demonstrate that smaller values of batch size can work effectively. Additionally, \(R_{down}\) has a certain impact on AUC and AP, but the effect is not significantly pronounced.

%% file: Appx-table.tex
\clearpage
% Table generated by Excel2LaTeX from sheet 'AUC'
% Table generated by Excel2LaTeX from sheet 'AUC'
% Table generated by Excel2LaTeX from sheet 'AUC'
\begin{table*}[htbp]
  \centering
  \caption{AUC of four AE models on 47 datasets}
    \begin{tabular}{ccccc}
    \toprule
    Dataset & VanillaAE & EntropyAE & ROBOD & RandNet \\
    \midrule
    4\_breastw & 0.891  & \textbf{0.928 } & 0.897  & 0.695  \\
    37\_Stamps & 0.775  & 0.905  & \textbf{0.910 } & 0.898  \\
    22\_magic.gamma & \textbf{0.754 } & 0.667  & 0.618  & 0.598  \\
    44\_Wilt & 0.659  & \textbf{0.769 } & 0.395  & 0.471  \\
    10\_cover & 0.896  & 0.898  & \textbf{0.971 } & \textbf{0.971 } \\
    14\_glass & \textbf{0.788 } & 0.707  & 0.667  & 0.721  \\
    16\_http & 0.991  & 0.995  & \textbf{0.996 } & \textbf{0.996 } \\
    38\_thyroid & 0.913  & 0.934  & 0.967  & \textbf{0.969 } \\
    12\_fault & \textbf{0.654 } & 0.616  & 0.504  & 0.484  \\
    2\_annthyroid & \textbf{0.725 } & 0.691  & 0.707  & 0.705  \\
    36\_speech & \textbf{0.497 } & 0.477  & 0.472  & 0.474  \\
    21\_Lymphography & 0.972  & 0.996  & 0.996  & \textbf{0.998 } \\
    42\_WBC & 0.948  & \textbf{0.993 } & 0.989  & 0.989  \\
    29\_Pima & 0.594  & \textbf{0.640 } & 0.581  & 0.481  \\
    47\_yeast & \textbf{0.431 } & 0.401  & 0.429  & 0.429  \\
    40\_vowels & 0.872  & \textbf{0.878 } & 0.688  & 0.552  \\
    28\_pendigits & 0.801  & 0.819  & \textbf{0.933 } & 0.932  \\
    6\_cardio & 0.802  & 0.949  & 0.956  & \textbf{0.957 } \\
    23\_mammography & 0.775  & \textbf{0.866 } & 0.752  & 0.733  \\
    45\_wine & 0.608  & \textbf{0.807 } & 0.560  & 0.646  \\
    13\_fraud & 0.949  & 0.951  & 0.951  & 0.951  \\
    25\_musk & 0.994  & 0.998  & \textbf{1.000 } & \textbf{1.000 } \\
    27\_PageBlocks & 0.893  & 0.915  & \textbf{0.920 } & 0.900  \\
    9\_census & \textbf{0.682 } & 0.677  & 0.661  & 0.659  \\
    30\_satellite & \textbf{0.638 } & 0.624  & 0.743  & 0.740  \\
    18\_Ionosphere & 0.918  & \textbf{0.927 } & 0.861  & 0.863  \\
    24\_mnist & 0.819  & 0.842  & 0.903  & \textbf{0.904 } \\
    20\_letter & \textbf{0.871 } & 0.846  & 0.595  & 0.524  \\
    46\_WPBC & \textbf{0.494 } & 0.481  & 0.452  & 0.447  \\
    35\_SpamBase & 0.528  & \textbf{0.550 } & 0.508  & 0.499  \\
    8\_celeba & \textbf{0.792 } & 0.784  & 0.756  & 0.756  \\
    15\_Hepatitis & 0.651  & \textbf{0.747 } & 0.727  & 0.750  \\
    41\_Waveform & 0.622  & 0.638  & \textbf{0.682 } & 0.648  \\
    1\_ALOI & 0.552  & \textbf{0.567 } & 0.545  & 0.544  \\
    33\_skin & 0.503  & \textbf{0.691 } & 0.486  & 0.545  \\
    5\_campaign & 0.747  & \textbf{0.738 } & 0.733  & 0.735  \\
    7\_Cardiotocography & 0.542  & 0.683  & 0.704  & \textbf{0.713 } \\
    19\_landsat & 0.484  & 0.543  & \textbf{0.549 } & 0.545  \\
    34\_smtp & \textbf{0.905 } & 0.887  & 0.829  & 0.773  \\
    3\_backdoor & \textbf{0.910 } & \textbf{0.910 } & 0.893  & 0.892  \\
    43\_WDBC & 0.930  & \textbf{0.986 } & 0.973  & 0.978  \\
    11\_donors & \textbf{0.801 } & 0.726  & 0.608  & 0.596  \\
    26\_optdigits & 0.445  & \textbf{0.531 } & 0.476  & 0.487  \\
    39\_vertebral & 0.461  & 0.385  & \textbf{0.494 } & 0.486  \\
    31\_satimage-2 & 0.952  & 0.971  & \textbf{0.982 } & 0.979  \\
    32\_shuttle & 0.935  & 0.987  & \textbf{0.993 } & 0.992  \\
    17\_InternetAds & 0.564  & \textbf{0.615 } & 0.614  & 0.611  \\
    \bottomrule
    \end{tabular}%
  \label{tab:ae-auc}%
\end{table*}%

\newpage

% Table generated by Excel2LaTeX from sheet 'AP'
\begin{table*}[htbp]
  \centering
  \caption{AP of four AE models on 47 datasets}
    \begin{tabular}{ccccc}
    \toprule
    Dataset & VanillaAE & EntropyAE & ROBOD & RandNet \\
    \midrule
    4\_breastw & 0.761  & 0.842  & \textbf{0.874 } & 0.698  \\
    37\_Stamps & 0.221  & 0.344  & \textbf{0.355 } & 0.339  \\
    22\_magic.gamma & \textbf{0.677 } & 0.591  & 0.578  & 0.562  \\
    44\_Wilt & 0.077  & \textbf{0.181 } & 0.041  & 0.048  \\
    10\_cover & 0.074  & 0.084  & 0.145  & \textbf{0.147 } \\
    14\_glass & \textbf{0.130 } & 0.113  & 0.103  & 0.109  \\
    16\_http & 0.325  & 0.463  & 0.355  & \textbf{0.473 } \\
    38\_thyroid & 0.199  & 0.276  & 0.426  & \textbf{0.455 } \\
    12\_fault & \textbf{0.469 } & 0.443  & 0.372  & 0.356  \\
    2\_annthyroid & 0.192  & 0.179  & 0.224  & \textbf{0.227 } \\
    36\_speech & \textbf{0.024 } & 0.019  & 0.019  & 0.018  \\
    21\_Lymphography & 0.545  & 0.931  & 0.931  & \textbf{0.948 } \\
    42\_WBC & 0.543  & \textbf{0.924 } & 0.853  & 0.845  \\
    29\_Pima & 0.423  & \textbf{0.465 } & 0.421  & 0.360  \\
    47\_yeast & \textbf{0.305 } & 0.295  & 0.303  & 0.302  \\
    40\_vowels & \textbf{0.279 } & 0.272  & 0.096  & 0.053  \\
    28\_pendigits & 0.083  & 0.081  & 0.205  & \textbf{0.216 } \\
    6\_cardio & 0.369  & 0.607  & \textbf{0.661 } & 0.659  \\
    23\_mammography & 0.091  & \textbf{0.182 } & 0.152  & 0.157  \\
    45\_wine & 0.104  & \textbf{0.238 } & 0.102  & 0.140  \\
    13\_fraud & 0.106  & 0.131  & \textbf{0.156 } & \textbf{0.156 } \\
    25\_musk & 0.883  & 0.954  & \textbf{1.000 } & \textbf{1.000 } \\
    27\_PageBlocks & 0.478  & 0.522  & \textbf{0.565 } & 0.546  \\
    9\_census & \textbf{0.095 } & 0.092  & 0.086  & 0.086  \\
    30\_satellite & 0.497  & 0.565  & \textbf{0.695 } & 0.693  \\
    18\_Ionosphere & 0.906  & \textbf{0.924 } & 0.803  & 0.798  \\
    24\_mnist & 0.369  & 0.377  & 0.442  & \textbf{0.445 } \\
    20\_letter & \textbf{0.361 } & 0.273  & 0.108  & 0.089  \\
    46\_WPBC & \textbf{0.231 } & 0.227  & 0.213  & 0.211  \\
    35\_SpamBase & 0.399  & \textbf{0.410 } & 0.391  & 0.389  \\
    8\_celeba & 0.076  & \textbf{0.112 } & 0.107  & 0.107  \\
    15\_Hepatitis & 0.289  & \textbf{0.343 } & 0.329  & 0.341  \\
    41\_Waveform & 0.047  & 0.045  & \textbf{0.054 } & 0.048  \\
    1\_ALOI & 0.038  & \textbf{0.039 } & 0.037  & 0.037  \\
    33\_skin & 0.199  & \textbf{0.284 } & 0.184  & 0.203  \\
    5\_campaign & \textbf{0.292 } & 0.279  & 0.283  & 0.288  \\
    7\_Cardiotocography & 0.334  & 0.417  & 0.454  & \textbf{0.461 } \\
    19\_landsat & 0.195  & 0.215  & \textbf{0.222 } & \textbf{0.222 } \\
    34\_smtp & 0.165  & 0.348  & 0.366  & \textbf{0.368 } \\
    3\_backdoor & \textbf{0.547 } & 0.543  & 0.520  & 0.515  \\
    43\_WDBC & 0.204  & \textbf{0.556 } & 0.469  & 0.497  \\
    11\_donors & \textbf{0.132 } & 0.105  & 0.087  & 0.086  \\
    26\_optdigits & 0.024  & \textbf{0.029 } & 0.025  & 0.025  \\
    39\_vertebral & 0.120  & 0.099  & \textbf{0.124 } & 0.118  \\
    31\_satimage-2 & 0.375  & 0.572  & \textbf{0.778 } & 0.776  \\
    32\_shuttle & 0.620  & 0.853  & \textbf{0.918 } & 0.917  \\
    17\_InternetAds & 0.225  & \textbf{0.295 } & 0.293  & 0.288  \\
    \bottomrule
    \end{tabular}%
  \label{tab:ae-ap}%
\end{table*}%

\clearpage
% Table generated by Excel2LaTeX from sheet 'Time'

% Table generated by Excel2LaTeX from sheet 'Time'
\begin{table*}[htbp]
  \centering
  \caption{Average Training Time (Compared to VanillaAE) of four AE models on 47 datasets}
    \begin{tabular}{ccccc}
    \toprule
    Dataset & VanillaAE & EntropyAE & ROBOD & RandNet \\
    \midrule
    4\_breastw & 1.00  & 0.08  & 1.18  & 7.61  \\
    37\_Stamps & 1.00  & 0.22  & 3.69  & 21.58  \\
    22\_magic.gamma & 1.00  & 0.00  & 3.66  & 21.63  \\
    44\_Wilt & 1.00  & 0.02  & 3.77  & 21.44  \\
    10\_cover & 1.00  & 0.00  & 3.71  & 21.72  \\
    14\_glass & 1.00  & 0.15  & 3.07  & 18.53  \\
    16\_http & 1.00  & 0.00  & 3.68  & 21.29  \\
    38\_thyroid & 1.00  & 0.02  & 3.31  & 19.01  \\
    12\_fault & 1.00  & 0.07  & 3.50  & 19.86  \\
    2\_annthyroid & 1.00  & 0.01  & 3.27  & 18.96  \\
    36\_speech & 1.00  & 0.02  & 3.33  & 33.63  \\
    21\_Lymphography & 1.00  & 0.17  & 3.04  & 19.96  \\
    42\_WBC & 1.00  & 0.13  & 3.07  & 19.35  \\
    29\_Pima & 1.00  & 0.16  & 3.12  & 19.13  \\
    47\_yeast & 1.00  & 0.04  & 3.15  & 19.19  \\
    40\_vowels & 1.00  & 0.11  & 3.17  & 19.22  \\
    28\_pendigits & 1.00  & 0.02  & 3.22  & 19.18  \\
    6\_cardio & 1.00  & 0.04  & 3.32  & 19.53  \\
    23\_mammography & 1.00  & 0.01  & 3.16  & 18.99  \\
    45\_wine & 1.00  & 0.33  & 4.37  & 28.90  \\
    13\_fraud & 1.00  & 0.00  & 3.67  & 20.79  \\
    25\_musk & 1.00  & 0.02  & 3.84  & 28.92  \\
    27\_PageBlocks & 1.00  & 0.02  & 3.69  & 21.35  \\
    9\_census & 1.00  & 0.00  & 3.53  & 40.85  \\
    30\_satellite & 1.00  & 0.02  & 3.25  & 21.64  \\
    18\_Ionosphere & 1.00  & 0.20  & 3.31  & 23.02  \\
    24\_mnist & 1.00  & 0.05  & 3.91  & 26.89  \\
    20\_letter & 1.00  & 0.08  & 3.91  & 22.05  \\
    46\_WPBC & 1.00  & 0.13  & 3.71  & 22.00  \\
    35\_SpamBase & 1.00  & 0.05  & 3.98  & 25.08  \\
    8\_celeba & 1.00  & 0.00  & 3.95  & 22.99  \\
    15\_Hepatitis & 1.00  & 0.86  & 2.66  & 18.01  \\
    41\_Waveform & 1.00  & 0.01  & 3.68  & 21.85  \\
    1\_ALOI & 1.00  & 0.00  & 3.89  & 22.02  \\
    33\_skin & 1.00  & 0.00  & 3.65  & 21.28  \\
    5\_campaign & 1.00  & 0.00  & 3.75  & 25.48  \\
    7\_Cardiotocography & 1.00  & 0.04  & 3.64  & 22.28  \\
    19\_landsat & 1.00  & 0.01  & 3.87  & 23.79  \\
    34\_smtp & 1.00  & 0.00  & 3.60  & 21.48  \\
    3\_backdoor & 1.00  & 0.00  & 3.90  & 30.01  \\
    43\_WDBC & 1.00  & 0.09  & 3.68  & 21.54  \\
    11\_donors & 1.00  & 0.00  & 3.51  & 21.84  \\
    26\_optdigits & 1.00  & 0.02  & 3.90  & 25.68  \\
    39\_vertebral & 1.00  & 0.38  & 3.46  & 21.39  \\
    31\_satimage-2 & 1.00  & 0.01  & 3.79  & 22.23  \\
    32\_shuttle & 1.00  & 0.00  & 3.52  & 21.35  \\
    17\_InternetAds & 1.00  & 0.03  & 4.02  & 58.91  \\
    \bottomrule
    \end{tabular}%
  \label{tab:ae-time}%
\end{table*}%